\documentclass{article}
\usepackage{caption}
\usepackage{rotating}
\usepackage[maxfloats=100]{morefloats}
\usepackage{listings}
\usepackage{booktabs}
\usepackage{xcolor}
\usepackage[title]{appendix}
\usepackage{makecell}
\usepackage{multirow}
\usepackage[symbol]{footmisc}
\usepackage{longtable}
\usepackage{subcaption}
\usepackage[numbers]{natbib}
\usepackage{geometry}
\usepackage{verbatim}
\usepackage{graphicx}
\usepackage{enumitem}
\usepackage{comment}
\usepackage{subfiles}
\usepackage{todonotes}
\usepackage{tikz-qtree,tikz-qtree-compat}
\usepackage{scrextend}
\usepackage{framed}
\usepackage{placeins}
\usepackage{setspace}
\usepackage{mathtools,collcell,eqparbox}
\usepackage{multicol}
\usepackage{pgfplots}
\usepackage{tikz}
\usepackage{courier}
\usepackage{url}
\usepackage{color}
\usepackage{authblk}
\usepackage{amssymb}
\usepackage{amsthm}
\usepackage{amsmath}
\usepackage[pagebackref = true]{hyperref}

\hypersetup{
  colorlinks = true,
  linkcolor = teal,
  anchorcolor = teal,
  citecolor = teal,
  filecolor = teal,
  urlcolor = teal
}

\pgfplotsset{compat=1.18}
\numberwithin{equation}{section}

\makeatletter
\newcounter{BMatrix}

\newcommand{\setmaxwd}[1]{%
  \eqmakebox[BM-\theBMatrix][\BMalign]{$#1$}%
}
\MHInternalSyntaxOn

\MHInternalSyntaxOff
\makeatother

\newtheorem{theorem}{Theorem}
\newtheorem{assumption}{Assumption}

\newtheorem{proposition}{Proposition}[section]

\newtheorem{corollary}{Corollary}[theorem]
\newtheorem{lemma}{Lemma}
\theoremstyle{definition}
\newtheorem{definition}{Definition}

\theoremstyle{remark}
\newtheorem{remark}{Remark}
\newtheorem*{remark*}{Remark}

\newcommand{\inv}{^{-1}}

\newcommand{\W}{\mathbb{W}}
\newcommand{\R}{\mathbb{R}}

\newcommand{\X}{\mathbb{X}}

\newcommand{\Ss}{\mathbb{S}}
\newcommand{\Z}{\mathbb{Z}}

\newcommand{\E}{\mathbb{E}}

\newcommand{\veps}{\varepsilon}

\newcommand{\eqd}{\overset{d}{=}}

\newcommand{\ra}{\rightarrow}

\newcommand{\xratb}[2]{\xrightarrow[#2]{#1}}

\newcommand{\cd}{\cdot}
\newcommand{\ds}{\dots}

\newcommand{\mrm}[1]{\mathrm{#1}}
\newcommand{\mbb}[1]{\mathbb{#1}}
\newcommand{\mfk}[1]{\mathfrak{#1}}

\newcommand{\tmix}{t_{\mathrm{mix}}}

\newcommand{\diam}{\mathrm{diam}}

\newcommand{\PiH}{\Pi_{\mathrm{H}}}

\newcommand{\cA}{\mathcal{A}}

\newcommand{\cD}{\mathcal{D}}

\newcommand{\cF}{\mathcal{F}}
\newcommand{\cG}{\mathcal{G}}
\newcommand{\cH}{\mathcal{H}}

\newcommand{\cL}{\mathcal{L}}

\newcommand{\cN}{\mathcal{N}}

\newcommand{\cP}{\mathcal{P}}

\newcommand{\cR}{\mathcal{R}}
\newcommand{\cS}{\mathcal{S}}
\newcommand{\cT}{\mathcal{T}}

\newcommand{\cW}{\mathcal{W}}
\newcommand{\cX}{\mathcal{X}}

\newcommand{\cZ}{\mathcal{Z}}

\newcommand{\spnorm}[1]{\left|#1\right|_\mathrm{span}}
\newcommand{\set}[1]{\left\{{#1}\right\}}
\newcommand{\setcond}[2]{\left\{ #1 : #2 \right\}}

\newcommand{\ceil}[1]{\left\lceil{#1}\right\rceil}

\newcommand{\norm}[1]{\left\|#1\right\|}

\newcommand{\normTV}[1]{\left\|#1\right\|_{\mathrm{TV}}}

\newcommand{\abs}[1]{\left|#1\right|}
\newcommand{\sqbk}[1]{\left[ #1 \right]}
\newcommand{\sqbkcond}[2]{\left[ #1 \middle| #2 \right]}

\newcommand{\crbk}[1]{\left( #1 \right)}

\title{Q-Learning with Stable Infinite-Dimensional \\Linear Function Approximation}
\author{Shengbo Wang}
\affil{Daniel J. Epstein Department of Industrial and Systems Engineering, University of Southern California}
\date{August 2026}

\begin{document}
\maketitle
\begin{abstract}
Q-learning with linear function approximation can be unstable because an arbitrary approximation architecture need not preserve the Bellman contraction. We develop a stable infinite-dimensional linear function approximation framework for Q-learning from a single Markovian behavior-policy trajectory. The learning variable is a coefficient field $\theta\in C(\mbb L)$ on a compact latent metric space $(\mbb L,\rho)$. The framework uses a reconstruction operator that maps $\theta$ to a continuous Q-function and a compression operator that maps Bellman updates back to latent coordinates. Nonexpansiveness of both operators induces a contractive latent Bellman map on $C(\mbb L)$, with a unique fixed point $\theta^*$ whose reconstruction approximates the optimal Q-function up to representation error. We propose two stochastic approximation (SA) algorithms and establish their sup-norm convergence bounds with a leading term of order $\widetilde O(n^{-1/2})$. The infinite-dimensional formulation provides a powerful abstraction for identifying the structures that govern statistical difficulty. Smoothness of the compression map in $\rho$ is inherited by $\theta^*$ and the SA iterates, allowing uniform estimation errors to be controlled through covering numbers of $(\mbb L,\rho)$ rather than the dimension of $C(\mbb L)$. Remarkably, the SA algorithms we propose are agnostic to the choice of $\rho$, and thus can automatically adapt to both the smoothness and the geometry. We further illustrate the framework through Q-measure-learning with linear density approximation and output-layer neural weight training under a frozen pretrained network.

\end{abstract}

\section{Introduction}
\label{sec:introduction}

Value-based reinforcement learning (RL) is a principled route from data to the optimal control of a Markov decision process (MDP) with unknown rewards and dynamics. In this paper, we study variants of the celebrated Q-learning algorithm \citep{watkinsdayan1992qlearning} for value-based RL in large and potentially infinite state-action spaces. When the underlying MDP model is known, the Bellman optimality equation characterizes the optimal action-value function $Q^*$ and an optimal greedy policy. Q-learning replaces exact Bellman updates with one-sample Bellman targets, thereby providing a model-free route to approximating $Q^*$. Extending this principle beyond finite state-action spaces requires function approximation: in large or continuous spaces, $Q^*$ is defined on a complex domain rather than a finite table, and a learned Q-function must generalize beyond the observed state-action pairs. Consequently, the effectiveness of value-based RL in such settings depends critically on how function approximation is incorporated into the Bellman learning dynamics.

Linear function approximation is a simple but powerful abstraction for scalable function learning: a complex function is represented as a linear combination of a collection of features. In approximate dynamic programming and value-based RL, linear architectures underpin early projected methods \citep{tsitsiklisvanroy1996feature,tsitsiklisvanroy1997td}, least-squares approaches \citep{bradtkebarto1996lstd,lagoudakisparr2003lspi}, kernel smoothing \citep{ormoneitsen2002kernel}, and quantization \citep{karaetal2023general}. Directly composing Bellman updates with a linear approximator built from an arbitrary basis, however, need not yield stable learning dynamics. The approximation step may fail to preserve the Bellman operator's sup-norm contraction, and classical counterexamples show that linear value-learning recursions can diverge \citep{baird1995residual}. This interaction among function approximation, bootstrapping, and off-policy sampling is commonly known as the ``deadly triad.''

Therefore, a starting point for designing a Q-learning-type algorithm with linear function approximation is to require not merely that the chosen function class approximate $Q^*$ well, but also that the Bellman update and approximation mechanism together define a provably stable and convergent algorithmic procedure. To achieve this, we leverage a classical insight from approximate dynamic programming: the stability induced by Bellman contraction can be preserved by using nonexpansive averaging operators for linear function approximation \citep{gordon1995stable,stachurski2008continuous}. In this paper, we develop a function-space extension of this principle and apply it to a single-trajectory Q-learning setting.

Specifically, we consider an infinite state-action space setting in which the learner observes a single state-action-reward trajectory generated by a Markovian behavior policy. We assume that the behavior chain admits a stationary distribution $\mu_b$, which describes the long-run sampling law of the observed state-action pairs. The objective is to optimize a coefficient field $\theta\in C(\mbb L)$ on a latent metric space $\mbb L$, which may be finite or infinite, so that the reconstructed function approximates $Q^*$. At the core of the algorithm design are two nonexpansive operators. A linear reconstruction map $\Phi$ decodes $\theta$ into a continuous Q-function, while a compression map $\mrm K$ maps a Q-function back to its latent coordinates. Given a coefficient field $\theta\in C(\mbb L)$, these reconstruction and compression maps induce the natural population value-iteration procedure
\[
\theta
\xratb{\Phi}{\text{reconstruction}}
q_\theta:=\Phi\theta
\xratb{\cT}{\text{Bellman update}}
\cT q_\theta
\xratb{\mrm K}{\text{compression}}
\mrm K \cT\Phi\theta.
\]
The map $\Phi$ plays the role of the feature map in standard linear function approximation, specifying how features and latent coefficients are combined to represent a Q-function. The compression map $\mrm K$ is chosen as a stationary-normalized kernel from the latent space to the state-action space, absolutely continuous with respect to $\mu_b$, so that the population update is aligned with the limiting distribution of the observed trajectory. If both $\Phi$ and $\mrm K$ are nonexpansive in the sup norm, then the operator $\cG:=\mrm K\cT\Phi$ inherits the contraction property of $\cT$. Consequently, $\cG$ has a unique fixed point $\theta^*$, and this contraction structure enables the design of stable stochastic approximation (SA) methods for learning $\theta^*$ from the observed trajectory.

The stationary distribution $\mu_b$ is generally unknown, so the stationary-normalized population update cannot be implemented directly. We develop two SA algorithms that estimate the same latent fixed point $\theta^*$ using only the observed trajectory: an unnormalized recursion that avoids the stationary normalization and an empirical-normalized recursion that estimates the normalization from the same data stream. For both algorithms, we establish high-probability, last-iterate guarantees in the sup norm. The leading stochastic error has the canonical $\widetilde O(n^{-1/2})$ order, with latent-space complexity entering through metric-entropy factors.

Beyond developing the stable approximation framework and its associated algorithms, a core contribution of this paper is the deliberate use of an infinite latent space $\mbb L$. This abstraction provides a way to assess the intrinsic statistical difficulty of large representations. Working directly in function space identifies the structures that support efficient learning: positivity and nonexpansiveness of the approximation maps, regularity of the coefficient field, and the geometry of the latent representation. In particular, smoothness of the compression kernel with respect to the latent metric $\rho$ is inherited by both the population fixed point and the SA iterates, yielding deterministic Lipschitz bounds throughout the learning dynamics. These bounds allow us to control the stochastic error on a finite $\epsilon$-net of $\mbb L$ and then extend the result to the full latent space. Consequently, the statistical complexity of estimating $\theta^*$ is governed by the covering number $\cN_\rho(\epsilon)$ of the latent space rather than by the formal dimension of $C(\mbb L)$.

Remarkably, $\rho$ only manifest in the regularity and covering arguments used in the analysis, while the SA algorithms are agnostic to the choice of $\rho$. Therefore, the same algorithmic trajectory can be analyzed under any admissible metric, and the resulting error bound reflects the optimal tradeoff between smoothness of the compression kernel and the covering complexity of $(\mbb L,\rho)$. In this sense, the algorithms admit automatic adaptivity to the latent geometry without requiring that geometry to be specified or optimized during learning. This perspective is particularly useful for large representations, where different metrics can reveal substantially different effective complexity despite describing the same latent space.

We illustrate the framework through two applications. First, we develop a linear density approximation version of the Q-measure-learning proposed by \citet{wang2026qml}, showing how its stationary-normalized kernel construction fits within our framework, with the density serving as the coefficient field. Second, we apply the framework to a frozen pretrained neural network, where the pre-output neurons form the latent space and the frozen network provides the reconstruction architecture. The activation patterns of the pre-output neurons induce a metric under which two neurons are close when they respond similarly across the state-action space. The resulting covering number measures the number of functionally distinct activation patterns rather than the raw network width. Thus, a very wide pretrained representation can still have moderate statistical complexity when its neurons aggregate into a much smaller number of distinct response modes.

\subsection{Literature Review}

\paragraph{Approximate dynamic programming and stability.}
Approximate dynamic programming combines Bellman updates with projection, interpolation, aggregation, or regression in a restricted function class, so the approximation operator affects both representation error and stability. General approximation operators may fail even to admit a fixed point \citep{defariasvanroy2000fixedpoints}. \citet{tsitsiklisvanroy1996feature} study feature-based methods for large-scale dynamic programming. \citet{gordon1995stable} identifies nonexpansive averages that preserve Bellman contraction and \citet{stachurski2008continuous} extends the nonexpansive approximation principle to continuous-state dynamic programming. We follow this stability-by-design perspective, but propose a reconstruction and behavior-distribution-adapted compression framework to enable RL from data generated online by a single Markovian trajectory.

\paragraph{Temporal-difference and Q-learning with linear function approximation.}
\citet{sutton1988td} introduced temporal-difference (TD) learning for policy evaluation. For on-policy linear approximation, \citet{tsitsiklisvanroy1997td} establish convergence and characterize the projected Bellman fixed point, which LSTD estimates through least squares \citep{bradtkebarto1996lstd}. Off-policy sampling can destabilize linear TD \citep{baird1995residual}, motivating gradient and emphatic variants \citep{suttonetal2009gtd,suttonmahmoodwhite2016emphatic}. Finite-sample analyses are established for linear TD under Markovian sampling and broad classes of off-policy methods \citep{bhandarietal2021td,chenetal2021offpolicytd}; recent work also proves convergence with arbitrary features \citep{wangzhang2026arbitrarytd}. These results concern policy evaluation and therefore do not directly resolve the stability of control updates involving the Bellman maximum.

Approximate control with linear architectures includes LSPI, linear SARSA, and gradient-based off-policy methods such as Greedy-GQ \citep{lagoudakisparr2003lspi,zouxuliang2019sarsa,maeietal2010greedygq,wangzou2020greedygq}. Direct Q-learning is more delicate because the target contains a maximization and the greedy target policy changes with the current approximation. \citet{melomeynribeiro2008analysis} give sufficient convergence conditions for Q-learning-type methods. Later work modifies the dynamics through two-estimator or target-network constructions, switching-system analysis, and regularization \citep{carvalhoetal2020convergent,zhangyaowhiteson2021target,leehe2020switching,limlee2024regularized,yangetal2026periodic}. Most closely, \citet{chenclarkemaguluri2023target} combine a target network and truncation to obtain finite-sample guarantees for online linear Q-learning from one Markovian trajectory. Recent analysis also shows that standard linear Q-learning may remain bounded in $L^2$ without converging to a Bellman fixed point \citep{liuetal2025linearq}. Our approach instead builds contraction into positive compression and reconstruction on $C(\mbb L)$ and explicitly separates convergence to the compressed fixed point from approximation of $Q^*$.

\paragraph{Reinforcement learning in continuous and general spaces.}
Continuous and general state-action spaces require approximation beyond sampled locations. Existing approaches include kernel-based empirical value iteration \citep{ormoneitsen2002kernel}, fitted Q iteration \citep{ernstetal2005tree,munosszepesvari2008fitted}, Q-learning with interpolation, nearest-neighbor, and kernel methods \citep{szepesvarismart2004interpolation,shahxie2018nearest,yehetal2023kernelq}, quantized Q-learning \citep{karaetal2023general}, and Q-measure-learning  \citep{wang2026qml}. Their data-access models differ: fitted and kernel analyses often use independent batches or a generative model, whereas \citet{shahxie2018nearest} and \citet{wang2026qml} study a single sample path. 

\paragraph{Stochastic approximation with Markov noise.}
\citet{borkar2000ode} develops the ordinary differential equation (ODE) method for SA. More recent work treats asymptotic stability under Markovian noise, finite-sample nonlinear SA, and single-trajectory contractive recursions \citep{liuetal2025ode,chenetal2022nonlinear,quwierman2020asynchronous}. Contractive SA has also been analyzed in separable Banach spaces under stochastic query access \citep{mouetal2022banach}. Our analysis combines dependent behavior-policy data, Bellman optimality, learned stationary normalization, and $C(\mbb L)$-valued iterates. Metric-entropy arguments convert latent regularity into uniform concentration \citep{vandervaartwellner1996}, yielding high-probability last-iterate control without reducing the problem to a fixed finite-dimensional parameterization.

\section{Markov Decision Processes and Behavior Policy}

We consider a compact state space $\mbb X\subset\R^{d_X}$, equipped with the Borel $\sigma$-algebra $\cX$.  The action space $\mbb A$ is either a compact subset of $\R^{d_A}$ or a finite set equipped with the discrete metric, and is equipped with its Borel $\sigma$-algebra $\cA$. Let
$\mbb Z = \mbb X\times\mbb A$ and $\cZ$ be the product $\sigma$-algebra. We denote by $C(\mbb Z)$ the Banach space of continuous real-valued functions on $\mbb Z$,
equipped with the sup norm $\norm{\cd}$.  We denote the set of probability measures on a Borel space $(\Ss,\cS)$ by $\cP(\cS)$. 

The state-action process is denoted by $\set{Z_n = (X_n,A_n)\in\Z:n\ge 0}$. Fix a discount factor $\gamma\in(0,1)$. Formally, we consider the control objective to maximize the $\gamma$-discounted value
\begin{equation}
        \sup_{\pi\in\PiH}
        \E^\pi\sqbkcond{\sum_{k=0}^{\infty}\gamma^kR_{k+1}}{X_0=x},\label{eqn:value_objective}
\end{equation}
where $\PiH$ is the class of history-dependent randomized policies, $R_{k+1}$ is the reward (to be specified later) that is realized at time $k+1$, and $\E^\pi$ is induced by the policy
$\pi$, the transition kernel $P$, and the reward law. Under standard regularity conditions \citep{hernandezlerma1996discrete}, stationary deterministic policies achieve the best possible value. Moreover, the optimal value as well as optimal stationary deterministic policies are given by the Bellman equation, which we formulate later. 

We are interested in the RL problem arising from this optimal control formulation, where the transition and reward distributions are unknown and must be learned from data in order to infer an optimal policy. In particular, we consider a setting where the optimizer observes only a single state-action-reward trajectory generated by a behavior policy.

Concretely, let $\set{P(\cdot| z)\in\cP(\cX):z\in\Z}$ be a Borel controlled transition kernel and $\set{\pi_b(\cdot|x)\in\cP(\cA):x\in \X}$ be a Markovian behavior policy. We consider an RL problem where the data are generated by the Markov chain under the behavior policy. Concretely, we assume a probability space $(\Omega,\cF,P)$ supporting a Markov chain $\set{Z_n = (X_n,A_n)\in\Z:n\ge 0}$ such that
$A_n\sim\pi_b(\cdot|X_n),$ and $X_{n+1}\sim P(\cdot|X_n,A_n).$ The process $\set{Z_n:n\ge 0}$ is referred to as the behavior chain, and its transition kernel, denoted by $P_b$, is given by
$$P_b(dy|z)=P(dx'|z)\pi_b(da'| x'), \quad y = (x',a')\in\Z,\; z\in\Z. 
$$

We consider a bounded and randomized reward structure where the reward is realized after the next state is revealed. Specifically, we take a sequence $\{F_n:n\ge1\}$ of i.i.d. measurable random reward functions, independent of the controlled state-action process, such that $F_n(\omega,\cdot)\in C(\mbb Z\times\mbb X)$ and
$\norm{F_n(\omega,\cd)}\le1$ everywhere. 
Then, the realized reward and the conditional mean reward are
$$
        R_{n+1}=F_{n+1}(Z_n,X_{n+1})
        \quad \text{and}\quad
        r(z)=E\sqbkcond{F_1(z,X_1)}{Z_0 = z}
$$
where the expectation is over both the fresh reward function $F$ and the next state.  Thus
$\abs{R_{n+1}}\le1$ and $\norm{r}\le1$.

Throughout the paper we use the natural
filtration $\cF_n=\sigma(Z_0,F_1,Z_1,\ldots,F_n,Z_n)$, for all $n\ge0$. Moreover, to facilitate readability, random variables indexed by time $n$ are generally $\cF_n$-measurable.

\begin{assumption}
\label{assump:mdp}
Assume that the mean reward $r\in C(\mbb Z)$.  Moreover, the controlled transition kernel is weakly continuous: whenever $z_n\to z$ in $\mbb Z$, $P(\cdot| z_n)\Rightarrow P(\cdot| z)$ weakly. 
\end{assumption}

Since $\Z$ is compact, under Assumption \ref{assump:mdp}, for every $h\in C(\mbb X)$, the mapping $z\ra\int_{\mbb X}h(x)P(dx| z)$ belongs to $C(\mbb Z)$. This follows from the Portmanteau Theorem. The weak-continuity condition is standard in stochastic control; see \citet[Chapter 3.3]{hernandezlerma1996discrete}.

Define the Bellman operator
$\mathcal T:C(\mbb Z)\to C(\mbb Z)$ as
\begin{equation}
        \mathcal{T}q(x,a)
        =
        r(x,a)+\gamma
        \int_{\mbb X}
        \max_{b\in\mbb A}q(x',b)P(dx'|x,a),
        \label{eqn:bellman_operator}
\end{equation}
We note that under Assumption \ref{assump:mdp}, the maximum theorem and the continuity property imply that $\mathcal T$ maps $C(\mbb Z)$ into itself.  Moreover, the
maximum and integration operators are nonexpansive in sup norm, so
\begin{equation}
        \norm{\mathcal{T}q-\mathcal{T}q'}
        \le
        \gamma\norm{q-q'},
        \qquad q,q'\in C(\mbb Z);
        \label{eqn:T_contraction}
\end{equation}
i.e. $\cT$ is a $\gamma$-contraction on $C(\Z)$. The Banach fixed-point theorem then implies the existence and uniqueness of a fixed-point $Q^*$. Since $\norm{r}\le1$, $Q^*(z)$ lies in
$[-1/(1-\gamma),1/(1-\gamma)]$ for all $z\in\Z$.

Under the standard dynamic-programming hypotheses; see, for example, \citep{hernandezlerma1996discrete}, an optimal deterministic Markovian policy for the control objective \eqref{eqn:value_objective} can be obtained by choosing greedy actions with respect to $Q^*$. Therefore, value-based RL methods approximate $Q^*$ and then construct a policy by taking greedy actions with respect to this estimator. Accordingly, in RL settings, learning $Q^*$ from data is a central problem of interest. 

Since we consider data generated by a single behavior trajectory, learning the global transition dynamics and approximating $Q^*$ require the behavior policy to explore the entire state space sufficiently well. For this purpose, we impose the standard uniform ergodicity assumption.

\begin{assumption}[Uniform ergodicity]
\label{assump:mix}
The Markov chain under the behavior policy $\pi_b$ is uniformly ergodic with a unique invariant probability measure $\mu_b$. 
\end{assumption}

For the later convergence-rate analysis, we need to quantify how quickly the behavior chain converges in distribution to $\mu_b$. Many metrics have been proposed in the literature to measure this convergence time \citep{aldous1997,meyn2012markov}. Perhaps the most commonly used one in the machine learning literature is the total-variation mixing time, introduced below. Throughout the paper, we use the convention
$\normTV{\mu-\nu}=2\sup_{B\in\cZ}\abs{\mu(B)-\nu(B)}$. 

\begin{definition}\label{def:mixing_time} 
Define the mixing time
$$\tmix :=
        \inf\setcond{ t\ge1}{
        \sup_{z\in\mathbb{Z}}
        \normTV{P_b^t(z,\cdot)-\mu_b}
        \le \frac12}.$$

\end{definition}
We note that under Assumption \ref{assump:mix}, $\tmix$ is well-defined and we always have $\tmix < \infty$.

\section{Stable Linear Function Approximation}
\label{sec:latent_averaging_formulation}

Learning the optimal Q-function is of central interest, but for computational and algorithmic reasons, modern machine learning methods typically approximate $Q^*$ within a prescribed function class. In this paper, we consider linear function approximation with a possibly infinite-dimensional coefficient field.

This section aims to transfer the Bellman equation into a coordinate system and induce a stable linear function approximation framework for Q-learning. To enable stable and sample-efficient learning, we first construct and analyze a stable compression-and-reconstruction procedure for the control problem in which the transition kernel $P$ and reward $r$ are known.

The learning algorithm will maintain a coefficient field representation $\theta\in \Theta\subset C(\mbb L)$ of the Q-function, which may be infinite-dimensional when $\mbb L$ is infinite. Given $\theta\in \Theta$, we will specify a reconstruction operator $\Phi$ to reconstruct the Q-function from $\theta$ by $q_\theta=\Phi\theta$.  The reverse direction is handled by the compression kernel $\mrm{K}$, which compresses a function on the (typically larger) state-action space $\mbb Z$ into a coefficient field on $\mbb L$.  

Concretely, let $(\mbb L,\rho)$ be a compact Polish space, and let $\cL$ be its Borel $\sigma$-algebra.  For $f:\mbb L\to\R$, define the Lipschitz seminorm as
$$
        [f]_{\rho}
        :=
        \sup_{\ell\ne\ell'}
        \frac{\abs{f(\ell)-f(\ell')}}{\rho(\ell,\ell')}.
$$
All coefficient fields are elements of $C(\mbb L)$.  Since
$(\mbb L,\rho)$ is a compact metric space, $C(\mbb L)$ is separable. 

The reconstruction kernel $\Phi$ is linear in the input coefficient field.  Let $\lambda\in \cP(\cL)$ be a probability reference measure, and let $\phi:\mbb Z\times\mbb L\ra\R_+$ be the feature map that satisfies certain continuity assumptions to be specified in Assumption \ref{assump:stable_kernels} and Remark \ref{rmk:Phi_C_map}.  Given a coefficient field $\theta$, $\Phi:C(\mbb L)\to C(\mbb Z)$ reconstructs the
Q-function by
\begin{equation}
        q_\theta(z):=\Phi\theta(z)
        =
        \int_{\mbb L}\phi(z,\ell)\theta(\ell)\lambda(d\ell),
        \qquad z\in\mbb Z.
        \label{eqn:Phi_reconstruction}
\end{equation}

In applications of interest, the latent space $\mbb L$ is typically less complex than the state-action space $\Z$. Thus, we consider a stable way to compress each candidate $q$ into a $\theta \in C(\mbb L)$. Specifically, let $\kappa:\mbb L\times\mbb Z\to\R_+$ and define the compression kernel $\mrm{K}$ by
\begin{equation}
        \mrm{K}q(\ell)
        =
        \frac{\int_{\mbb Z}\kappa(\ell,z)q(z)\mu_b(dz)}
             {c(\ell)}, \quad\text{where}\quad 
        c(\ell)=\int_{\mbb Z}\kappa(\ell,z)\mu_b(dz).
        \label{eqn:K_latent_average}
\end{equation}
Since $\mrm K$ is normalized by its marginal integral, we further assume without loss of generality that $\kappa\in[0,1]$ to simplify the discussion. Importantly, the definition of $\mrm K$ depends on the stationary distribution $\mu_b$ of the behavior chain. This choice is designed to induce stability in the SA algorithm running on the single trajectory generated by $\pi_b$.

Since our objective is to approximate the solution $Q^* = \cT Q^*$, given this framework, the natural approximator should be $q^* = \Phi\theta^*$ for some $\theta^*\in C(\mbb L)$. This motivates us to consider the Bellman-equation-compatible coefficient $\theta^*$ as a fixed point of 
\begin{equation}
        \theta^* = \cG\theta^*:=\mrm{K}\mathcal T\Phi \theta^*.
        \label{eqn:G_latent_map}
\end{equation}
Thus, $\cG$ first reconstructs a value function in $C(\Z)$ from a coefficient field in $C(\mbb L)$, then performs a Bellman operator in $C(\mbb Z)$, and finally compresses the result back into $C(\mbb L)$. In particular, $\cG$ maps $C(\mbb L)$ to $C(\mbb L)$. 

If, in addition, we impose stability conditions on $\Phi$ and $\mrm K$ so that the operator $\cG$ is a contraction on $C(\mbb L)$, then running Picard iteration gives $\theta_k:=\cG^k \theta_0\ra \theta^*$. The dynamics work as follows.
\[
\theta_k\in C(\mbb L)\xratb{\Phi}{\text{Reconstruction} } q_{\theta_k}\in C(\Z)\xratb{\cT}{\text{Bellman update}} \cT q_{\theta_k}\in C(\Z) \xratb{\mrm K}{\text{Compression }} \theta_{k+1}\in C(\mbb L). 
\]
In particular, if $\cG$ is a contraction, we can leverage stochastic approximation theory to design algorithms that approximate $\theta^*$. Moreover, if the information loss in the compression and reconstruction steps is moderate, then an estimator $\hat\theta$ that approximates $\theta^*$ should yield a good approximation of $Q^*$ through $\Phi\hat\theta$. 

To induce contraction of $\cG$, we note that $\cG$ is the composition $\mrm K\circ\cT\circ\Phi$, where $\cT$ is already a contraction. Therefore, if both $\Phi$ and $\mrm K$ are nonexpansive, then $\cG$ inherits the contraction property of $\cT$. This motivates the following assumption.

\begin{assumption}[Stable reconstruction and compression kernels]
\label{assump:stable_kernels}
The scalar kernels
$\phi:\Z\times \mbb L\ra \R_+$ and $\kappa:\mbb L\times \Z\ra [0,1]$ are measurable and satisfy:
\begin{enumerate}[label=(\roman*)]
\item for every $z\in\mbb Z$,
\begin{equation}
        \int_{\mbb L}\phi(z,\ell)\lambda(d\ell)\le1;
        \label{eqn:Phi_normalized}
\end{equation}

\item there exists $c_\wedge>0$ such that
\begin{equation}
        c(\ell)=\int_{\mbb Z}\kappa(\ell,z)\mu_b(dz)\ge c_\wedge,
        \qquad \ell\in\mbb L;
        \label{eqn:c_lower}
\end{equation}

\item there exists $L_{\kappa,\rho}<\infty$ such that
\begin{equation}
        \abs{\kappa(\ell,z)-\kappa(\ell',z)}
        \le
        L_{\kappa,\rho}\rho(\ell,\ell'),
        \qquad
        \ell,\ell'\in\mbb L,
        \ z\in\mbb Z;
        \label{eqn:kappa_lip}
\end{equation}

\item $\Phi\theta\in C(\mbb Z)$ whenever $\theta\in C(\mbb L)$ and
$\norm{\theta}\le 1/(1-\gamma)$.
\end{enumerate}
\end{assumption}

The first condition makes $\Phi$ a positive sub-Markov reconstruction kernel, so sup-norm errors in coefficient space do not grow when reconstructed as $q$-functions. The second is a non-singularity condition ensuring that the normalized compression is well defined. The third condition is not needed for $\cG$ to be a contraction; instead, it induces Lipschitz regularity of the fixed point $\theta^*$, an important property for the finite-sample analysis of the algorithms introduced later. The fourth condition ensures that Bellman targets reconstructed from bounded continuous coefficient fields remain in the space of continuous $q$-functions.

\begin{remark}\label{rmk:Phi_C_map}
We note that (iv) has simple sufficient conditions. Indeed, it holds when $\mbb Z$ and $\mbb L$ are compact and the feature $\phi$ is continuous on $\mbb Z\times\mbb L$.  More generally, it is enough to assume the weaker condition: if $z_m\to z$ then $\int_{\mbb L}\abs{\phi(z_m,\ell)-\phi(z,\ell)}\lambda(d\ell)\to0.$ This directly implies $\Phi\theta\in C(\mbb Z)$ for every bounded coefficient field
$\theta\in C(\mbb L)$.
\end{remark}

Since $0\le\kappa\le1$ and $\mu_b$ is a probability measure, $c_\wedge\le c(\ell)\le1$ for all
$\ell\in\mbb L$. Moreover, Assumption \ref{assump:stable_kernels} also implies that 
\begin{equation}
        \norm{\Phi\theta-\Phi\theta'}
        \le
        \norm{\theta-\theta'}
        \quad\text{and}\quad
        \norm{\mrm{K}q-\mrm{K}q'}
        \le
        \norm{q-q'} .
        \label{eqn:stable_nonexpansion}
\end{equation}
for all $\theta,\theta'\in C(\mbb L)$ and $q,q'\in C(\mbb Z)$. Indeed, $\Phi$ is linear and nonexpansive by \eqref{eqn:Phi_normalized}, and $\mrm{K}q(\ell)$ is an average of $q$ under the probability kernel proportional to $\kappa(\ell,z)\mu_b(dz)$. Since natural estimators of $Q^*$ are bounded by $1/(1-\gamma)$ and both $\mrm K$ and $\Phi$ are nonexpansive, it suffices to consider the following bounded subset of coefficient fields:
$$
\Theta:=\left\{\theta\in C(\mbb L):\norm{\theta}\le\frac{1}{1-\gamma}\right\}.
$$
We show in Proposition~\ref{prop:latent_fixed_point} that $\cG=\mrm K\cT\Phi$ is a $\gamma$-contraction on $\Theta$, inheriting the contraction factor of $\cT$.

Under the Lipschitz assumption $(iii)$, the compression kernel $\mrm{K}$ also induces Lipschitz continuity. Indeed, for any bounded measurable $q$,
\begin{equation}
\begin{aligned}
        \abs{(\mrm{K}q)(\ell)-(\mrm{K}q)(\ell')}
        &\le
        \frac{\left|\int_{\mbb Z}(\kappa(\ell,z)-\kappa(\ell',z))q(z)\mu_b(dz)\right|}{c(\ell)}     +
        \left|\int_{\mbb Z}\kappa(\ell',z)q(z)\mu_b(dz)\right|
        \abs{\frac1{c(\ell)}-\frac1{c(\ell')}}                                      \\
        &\le
        \frac{\norm{q}L_{\kappa,\rho}\rho(\ell,\ell')}{c(\ell)}
        +
        \norm{q}c(\ell')
        \frac{\abs{c(\ell)-c(\ell')}}{c(\ell)c(\ell')}                                      \\
        &\le
        2\norm{q}L_{\kappa,\rho} c_\wedge^{-1}\rho(\ell,\ell').
\end{aligned}
        \label{eqn:K_continuity_bound}
\end{equation}
Consequently, $\mrm{K}q\in C(\mbb L)$ whenever $q$ is bounded. 

These observations lead to the following proposition. 
\begin{proposition}
\label{prop:latent_fixed_point}
Under Assumptions~\ref{assump:mdp} and \ref{assump:stable_kernels}, the operator
$\cG=\mrm{K}\mathcal T\Phi$ maps $\Theta$ into itself and is a $\gamma$-contraction on $\Theta$.  Hence $\cG$
has a unique fixed point $\theta^*\in\Theta$.  Moreover, $\theta^*$ is $\rho$-Lipschitz and
\begin{equation}
        [\theta^*]_{\rho}
        \le
        \frac{2L_{\kappa,\rho}}{c_\wedge(1-\gamma)}.
        \label{eqn:theta_star_lipschitz}
\end{equation}
With $q^*:=\Phi\theta^*$, we have
\begin{equation}
        \theta^*=\mrm{K}\mathcal T\Phi\theta^* \quad\text{and}\quad 
        q^*=\Phi\mrm{K}\mathcal T q^* .
        \label{eqn:latent_fixed_identities}
\end{equation}
Moreover, if $Q^*\in C(\mbb Z)$ satisfies $Q^*=\mathcal TQ^*$, then
\begin{equation}
        \norm{Q^*-q^*}
        \le
        \frac{\norm{Q^*-\Phi\mrm{K}Q^*}}{1-\gamma}.
        \label{eqn:latent_bias_bound}
\end{equation}
\end{proposition}

Note that the error $\norm{Q^*-\Phi\mrm{K}Q^*}$ is the representation error induced by compressing
$Q^*$ through $\mrm{K}$ and reconstructing it through $\Phi$.  If the architecture is rich enough that $\Phi\mrm{K}Q^*\approx Q^*$, then the population latent fixed point approximates the true Bellman fixed point, with the representation error amplified by at most a factor of $(1-\gamma)^{-1}$.

\begin{proof}
The continuity estimate \eqref{eqn:K_continuity_bound} shows that
$\mrm{K}\mathcal T\Phi\theta\in C(\mbb L)$ for every $\theta\in\Theta$.  Moreover, if
$\theta\in\Theta$, then the reconstruction $\Phi \theta$ is bounded by $1/(1-\gamma)$, and hence
$$
        \norm{\mathcal T\Phi\theta}
        \le
        1+\frac{\gamma}{1-\gamma}
        =
        \frac{1}{1-\gamma}.
$$
The nonexpansiveness of $\mrm{K}$ then gives $\norm{\cG\theta}\le1/(1-\gamma)$, so $\cG$ maps $\Theta$ into
itself.  

For $\theta,\theta'\in\Theta$, the reconstruction, Bellman update, and compression steps are
respectively nonexpansive, $\gamma$-contractive, and nonexpansive.  Thus
\eqref{eqn:stable_nonexpansion} and \eqref{eqn:T_contraction} give $\norm{\cG\theta-\cG\theta'}
        \le
        \gamma\norm{\theta-\theta'}.$
The set $\Theta$ is closed in the Banach space $C(\mbb L)$, hence complete, and Banach's fixed-point
theorem gives a unique fixed point $\theta^*\in\Theta$.

Next, we show the regularity generated by the compression map.  Since
$\theta^*=\mrm{K}\mathcal T\Phi\theta^*$ and
$\norm{\mathcal T\Phi\theta^*}\le1/(1-\gamma)$, the continuity estimate \eqref{eqn:K_continuity_bound}
applied with $q=\mathcal T\Phi\theta^*$ gives
$$
        \abs{\theta^*(\ell)-\theta^*(\ell')}
        \le
        \frac{2L_{\kappa,\rho}}{c_\wedge(1-\gamma)}\rho(\ell,\ell'),
$$
which proves \eqref{eqn:theta_star_lipschitz}.  Then, applying $\Phi$ to the fixed-point identity
$\theta^*=\mrm{K}\mathcal T\Phi\theta^*$ gives \eqref{eqn:latent_fixed_identities}.  

It remains to identify the approximation bias. Since $Q^*=\mathcal TQ^*$ and $q^*=\Phi\mrm{K}\mathcal Tq^*$, we use linearity of $\mrm{K}$ and $\Phi$ to write
$$
        Q^*-q^*
        =
        (Q^*-\Phi\mrm{K}Q^*)
        +
        \Phi\mrm{K}(\mathcal TQ^*-\mathcal Tq^*).
$$
Taking sup norms and applying \eqref{eqn:stable_nonexpansion} and
\eqref{eqn:T_contraction} gives
$$
        \norm{Q^*-q^*}
        \le
        \norm{Q^*-\Phi\mrm{K}Q^*}
        +
        \gamma\norm{Q^*-q^*}.
$$
Rearranging proves \eqref{eqn:latent_bias_bound}. 
\end{proof}

\section{Two Algorithms and Their Regularities}
\label{sec:two_algorithms}

In this section, we motivate and develop stochastic approximation (SA) algorithms to estimate the fixed point $\theta^*$ in Proposition~\ref{prop:latent_fixed_point} using data generated online by the behavior trajectory. The main design issue is that the natural latent update rule associated with the contraction operator $\cG$ contains the denominator $c$. Since $c$ is defined by integrating $\kappa$ against the stationary distribution of the behavior chain, it is generally unknown. We therefore consider two implementable alternatives. The unnormalized recursion avoids the denominator entirely, while the empirical-normalized recursion estimates the denominator from the same data trajectory.

Before introducing these two implementable versions, we first consider an idealized algorithm that assumes knowledge of $c$. This serves as the basis for the alternative designs.

\subsection{An Idealized Algorithm}
The population fixed-point equation suggests the following idealized SA update for estimating $\theta^*$. Starting from $\theta_0\in\Theta$, 
\begin{equation}
        \theta'_{n+1}(\ell)
        =
        \theta'_n(\ell)
        +
        \alpha_{n+1}
        \frac{\kappa(\ell,Z_n)}{c(\ell)}
        \crbk{Y'_{n+1}-\theta'_n(\ell)},\quad \ell\in\mbb L
        \label{eqn:ideal_alg}
\end{equation}
where $Y'_{n+1}$ is the one-sample Bellman target defined by
$$   Y'_{n+1}:=R_{n+1}+\gamma\max_{a\in\mathbb{A}}\Phi\theta'_n(X_{n+1},a).
$$
Note that $Y'_{n+1}$ and $\theta'_{n+1}$ are $\cF_{n+1}$-measurable.

This recursion can be understood as follows. First, under Assumption~\ref{assump:mix}, the law of $Z_n$ converges uniformly to $\mu_b$ over all initial states. Then, if $Z_n\sim\mu_b$ and the iterate is frozen at $\theta\in\Theta$, the stationary mean drift of
\eqref{eqn:ideal_alg} is $\cG\theta-\theta$.  Indeed, with
$Y_\theta = F(Z,X') + \gamma\max_{a\in\mathbb{A}}\Phi\theta(X',a)$, where $F\eqd F_1$,
$Z\sim \mu_b$, and $X'|Z\sim P(\cd|Z)$,
\begin{equation}\label{eqn:ideal_expected_update}
        E_{F, Z, X'}\left[
        \frac{\kappa(\ell,Z)}{c(\ell)}
        (Y_\theta-\theta(\ell))
        \right]
        =\int_\Z\frac{\kappa(\ell,z)}{c(\ell)}
        \mathcal T\Phi\theta(z)\mu_b(dz)-\theta(\ell)
        =
        \cG\theta(\ell)-\theta(\ell).
\end{equation}

Thus, in stationarity, \eqref{eqn:ideal_alg} can be viewed as the $C(\mbb L)$-valued stochastic approximation to the fixed point of $\cG$:
$$
\theta_{n+1}' = \theta_n' + \alpha_{n+1}\{(\cG-I)\theta_n'+\xi _{n+1}\},
$$
where $\set{\xi_n:n\geq 1}$ is a sequence of centered perturbations away from the population dynamics. The corresponding mean-field $C(\mbb L)$-valued ODE for the SA update is
\begin{equation}\label{eqn:ideal_ode}
   \dot\vartheta(t) = \cG\vartheta(t) - \vartheta(t), \quad \vartheta(0) = \theta_0.
\end{equation}
Due to the stability properties of $\cG$ in Proposition \ref{prop:latent_fixed_point}, the unique fixed point $\theta^*$ of this ODE is exponentially stable, with a sup-norm convergence time constant $1-\gamma$; i.e. $\|\vartheta(t) - \theta^*\|\leq e^{-(1-\gamma)t}\norm{\theta_0 - \theta^*}$.

As noted earlier, this recursion is generally not implementable because $c(\ell)=\int\kappa(\ell,z)\mu_b(dz)$ depends on the unknown stationary law $\mu_b$. However, the alternative algorithms proposed in this paper follow the same basic design principle. In the subsequent sections, we introduce two implementable approaches to address the unknown denominator, leading to the unnormalized (labeled by $\mrm U$) and normalized (labeled by $\mrm N$) algorithms.

Before proceeding, we define some notation following the convention used for the idealized algorithm. For $\mrm M\in\{\mrm U,\mrm N\}$, the unnormalized and normalized algorithms still update the $\cF_n$-measurable iterate $\theta_n^{\mrm M}\in C(\mbb L)$ using the reconstructed Q-function $q_n^{\mrm M}:=\Phi\theta_n^{\mrm M}$ and the one-sample Bellman target
\begin{equation}
        Y_{n+1}^{\mrm M}
        =R_{n+1}+\gamma\max_{a\in\mathbb{A}}q_n^{\mrm M}(X_{n+1},a).
        \label{eqn:target_def}
\end{equation}
Taking conditional expectation, we have \begin{equation}
        E[Y_{n+1}^{\mrm M}\mid\cF_n]
        =
        (\mathcal{T}q_n^{\mrm M})(Z_n).
        \label{eqn:conditional_target_identity}
\end{equation}
Therefore, the conditional mean of the one-sample Bellman target is exactly the Bellman operator applied to the current reconstruction.

\subsection{Unnormalized Algorithm}

The first algorithm we consider simply removes the normalization $c(\ell)$ from the idealized algorithm, leading to the following simple recursion
\begin{equation}
        \theta_{n+1}^{\mrm U}(\ell)
        =
        \theta_{n}^{\mrm U}(\ell)
        +
        \alpha_{n+1}^{\mrm U}
        \kappa(\ell,Z_n)
        \crbk{Y_{n+1}^{\mrm U}-\theta_{n}^{\mrm U}(\ell)}
        \label{eqn:unnormalized_update}
\end{equation}
where $Y_{n+1}^{\mrm U}$ is defined in \eqref{eqn:target_def}.  This SA update completely removes the normalization factor $c$, hence the name unnormalized algorithm.

Observe that removing the $c(\ell)$ in the denominator changes the mean-field ODE without changing the fixed point. Indeed, instead of \eqref{eqn:ideal_expected_update}, the expected stationary update becomes
$$
E_{F, Z, X'}\left[
        \kappa(\ell,Z)
        (Y_\theta-\theta(\ell))
        \right] 
        =
        c(\ell)[\cG\theta(\ell)-\theta(\ell)].
$$
Hence, the mean-field ODE becomes
\begin{equation}\label{eqn:U_ODE_pre}
\dot\vartheta(t) = c[\cG\vartheta(t) - \vartheta(t)], \quad \vartheta(0) = \theta_0. 
\end{equation}
Because $c\geq c_\wedge>0$, the multiplier $c$ can be viewed as a preconditioner, and this ODE has the same fixed point $\theta^*$ as the idealized flow, characterized by $\dot\vartheta(t)=0$.

However, the contraction rate of this ODE flow in the sup norm is limited by the slowest-converging latent coordinate, which in the worst case corresponds to the smallest multiplier value $c_\wedge$. As shown in Lemma~\ref{lemma:U_ODE_exponential_stability}, the ODE \eqref{eqn:U_ODE_pre} is exponentially stable with a sup norm convergence time constant of $c_\wedge(1-\gamma)$. Since $c_\wedge\leq 1$, the mean-field ODE \eqref{eqn:U_ODE_pre} generally converges more slowly to the fixed point $\theta^*$ than the idealized ODE \eqref{eqn:ideal_ode}.

\subsection{Normalized Algorithm}
The second algorithm we propose maintains an empirical estimate of the stationary compression denominator. For $n\ge0$, define
\begin{equation}
        \hat c_n(\ell)
        =
        \frac1{n+1}\sum_{t=0}^{n}\kappa(\ell,Z_t),
        \qquad
        \tilde c_n(\ell)= \hat c_n(\ell)\vee c_\wedge.
        \label{eqn:empirical_c}
\end{equation}
Equivalently, $\hat c_n$ can be recursively updated as
\begin{equation}
        \hat c_{n+1}(\ell)
        =
        \frac{n+1}{n+2}\hat c_n(\ell)
        +
        \frac1{n+2}\kappa(\ell,Z_{n+1})
 \label{eqn:empirical_c_recursion}
\end{equation}
for $n\ge0$. Then, the empirical-normalized recursion is
\begin{equation}
        \theta_{n+1}^{\mrm N}(\ell)
        =
        \theta_{n}^{\mrm N}(\ell)
        +
        \alpha_{n+1}^{\mrm N}
        \frac{\kappa(\ell,Z_n)}{\tilde c_n(\ell)}
        \crbk{Y_{n+1}^{\mrm N}-\theta_{n}^{\mrm N}(\ell)}.
        \label{eqn:empirical_normalized_update}
\end{equation}

Here, the clipping in $\tilde c_n$ is only a safety device that enforces the natural bound $\tilde c_n\ge c_\wedge$ to avoid numerical instabilities. Note that $c_\wedge$ only needs to be a lower bound of $\inf_{\ell\in\mbb L} c(\ell)$, so a practitioner can set it to a small safety parameter without knowing $\inf_{\ell\in\mbb L} c(\ell)$.

When $n$ is large, the empirical denominator converges almost surely to the population denominator $c$. Thus, one should expect the stochastic recursion to track the same idealized mean-field ODE
\begin{equation}\label{eqn:N_ODE_pre}
\dot\vartheta(t) = \cG\vartheta(t) - \vartheta(t), \quad \vartheta(0) = \theta_0. 
\end{equation}
whose sup norm contraction time constant is $1-\gamma$ in the worst case, not modulated by the minimum $c_\wedge$.

\subsection{Regularities of the Iterates}
We impose the initialization and step size conditions used for both recursions.
\begin{assumption}[Initialization and step sizes]
\label{assump:stepsizes}
Suppose that for $\mrm M \in\set{\mrm U, \mrm N}$, the initial coefficient fields satisfy
$\theta_0^{\mrm M} \in \Theta$ and $[\theta_0^{\mrm M}]_{\rho}\le L_{0,\rho}$.  Moreover, we use the
algorithm-specific step sizes
$$
        \alpha_n^{\mrm M}=\frac{\alpha_{\mrm M}}{n+n_0^{\mrm M}},
        \qquad n\ge1,
        \quad \mrm M\in\{\mrm U,\mrm N\}.
$$
Moreover, assume
$n_0^{\mrm U}\ge \alpha_{\mrm U}$ and $n_0^{\mrm N}\ge \alpha_{\mrm N}/c_\wedge$ for the unnormalized and normalized algorithms, respectively.
\end{assumption}

The requirements on the offsets $n_0^{\mrm U}$ and $n_0^{\mrm N}$ in the step-size conditions ensure that, for both algorithms, each update is a convex combination of the current coefficient and a bounded Bellman target. This yields the boundedness of the coefficient fields as stated in the next proposition. Moreover, the standard rescaled linear step size is used to obtain the usual $n^{-1/2}$ last-iterate error convergence rate for stochastic approximation. Slower decaying polynomial step sizes generally lead to worse last-iterate convergence rates, while Polyak--Ruppert averaging can help recover the parametric rate. In this paper, we focus on last-iterate convergence and therefore use the rescaled linear step size.

\begin{proposition}[Regularity of the coefficient fields]
\label{prop:lipschitz_iterates}
Suppose Assumptions~\ref{assump:mdp}, \ref{assump:stable_kernels}, and \ref{assump:stepsizes} hold. Then, for every $n\ge0$, $\theta_n^\mrm U,\theta_n^\mrm N\in\Theta$. Moreover, the coefficient fields are also Lipschitz with
\begin{equation}
        [\theta_n^{\mrm U}]_{\rho}
        \le L_{0,\rho} + \frac{2\alpha_{\mrm U} L_{\kappa,\rho}}{1-\gamma}
        \log\crbk{\frac{n+n_0^{\mrm U}}{n_0^{\mrm U}}},
        \label{eqn:LU_def}
\end{equation}
and
\begin{equation}
        [\theta_n^{\mrm N}]_{\rho}
        \le
        L_{0,\rho}+
        \frac{2\alpha_{\mrm N} L_{\kappa,\rho}}{1-\gamma} \crbk{\frac{1}{c_\wedge}+\frac{1}{c_\wedge^2}}
        \log\crbk{\frac{n+n_0^{\mrm N}}{n_0^{\mrm N}}}.
        \label{eqn:LN_def}
\end{equation}
\end{proposition}

Proposition~\ref{prop:lipschitz_iterates} provides the regularity needed for uniform concentration.  Its
bounds are deterministic and do not require smoothness of the unknown fixed point $\theta^*$.  On an
$\epsilon$-net of $\mbb L$, the stochastic terms reduce to finitely many scalar martingales, while the
Lipschitz estimates control the interpolation error away from the net.  Consequently, the geometry of
$\mbb L$ will enter the finite-time bounds through its covering number.

\begin{proof}[Proof of Proposition~\ref{prop:lipschitz_iterates}]
First, we establish the boundedness of the Bellman targets assuming bounded iterates. For $\mrm M\in\{\mrm U,\mrm N\}$, if
$\norm{\theta_n^{\mrm M}}\le 1/(1-\gamma)$, then by \eqref{eqn:stable_nonexpansion},
$\norm{q_n^{\mrm M}} = \norm{\Phi\theta_n^\mrm M}\le 1/(1-\gamma)$, and therefore
\begin{equation}
        \abs{Y_{n+1}^{\mrm M}}
        \le
        1+\frac{\gamma}{1-\gamma}
        =
        \frac{1}{1-\gamma}.
        \label{eqn:target_bound}
\end{equation}

Next, we prove the statement for both algorithms separately.
\paragraph{Unnormalized algorithm.}
We first show the boundedness using induction. By Assumption \ref{assump:stepsizes}, $\norm{\theta_0^{\mrm U}}\leq \frac{1}{1-\gamma}$. For the induction step, assume $\norm{\theta_k^{\mrm U}}\le \frac{1}{1-\gamma}$.  Then,
\eqref{eqn:target_bound} implies that $\abs{Y_{k+1}^{\mrm U}}\le \frac{1}{1-\gamma}$.  Therefore, for all $\ell\in\mbb L$, 
$$
        \theta_{k+1}^{\mrm U}(\ell)
        =
        \crbk{1-\alpha_{k+1}^{\mrm U}\kappa(\ell,Z_k)}
        \theta_k^{\mrm U}(\ell)
        +
        \alpha_{k+1}^{\mrm U}\kappa(\ell,Z_k)Y_{k+1}^{\mrm U}.
$$
Since $0\le \kappa\le1$ and $\alpha_{n+1}^{\mrm U}\le \alpha_{\mrm U}/n_0^{\mrm U}\le1$, the right-hand side is a convex combination of
two points in $[-\frac{1}{1-\gamma},\frac{1}{1-\gamma}]$. Therefore, $\theta_{k+1}^{\mrm U}(\ell)\in [-\frac{1}{1-\gamma},\frac{1}{1-\gamma}]$ and, by induction, $\norm{\theta_n^{\mrm U}}\leq\frac{1}{1-\gamma}$ for all $n\ge 0$.

For the Lipschitz estimates, we subtract the unnormalized updates at $\ell,\ell'\in\mbb L$.
\begin{equation}
\begin{aligned}
&\theta_{k+1}^{\mrm U}(\ell)-\theta_{k+1}^{\mrm U}(\ell')\\
&= \crbk{1-\alpha_{k+1}^{\mrm U}\kappa(\ell,Z_k)}\theta_k^{\mrm U}(\ell) + \alpha_{k+1}^{\mrm U}\kappa(\ell,Z_k)Y_{k+1}^{\mrm U} \\
&\quad 
-\crbk{1-\alpha_{k+1}^{\mrm U}\kappa(\ell',Z_k)}\theta_k^{\mrm U}(\ell') -  \alpha_{k+1}^{\mrm U}\kappa(\ell',Z_k)Y_{k+1}^{\mrm U}\\
&\stackrel{(i)}{=}
\crbk{1-\alpha_{k+1}^{\mrm U}\kappa(\ell,Z_k)} 
\crbk{\theta_k^{\mrm U}(\ell)-\theta_k^{\mrm U}(\ell')} + \alpha_{k+1}^{\mrm U}\crbk{\kappa(\ell',Z_k)-\kappa(\ell,Z_k)} \theta_k^{\mrm U}(\ell') \\
&\quad +\alpha_{k+1}^{\mrm U}
\crbk{\kappa(\ell,Z_k)-\kappa(\ell',Z_k)}
Y_{k+1}^{\mrm U}\\
&= \crbk{1-\alpha_{k+1}^{\mrm U}\kappa(\ell,Z_k)} 
\crbk{\theta_k^{\mrm U}(\ell)-\theta_k^{\mrm U}(\ell')} + \alpha_{k+1}^{\mrm U}\crbk{\kappa(\ell,Z_k)-\kappa(\ell',Z_k)} \crbk{Y_{k+1}^{\mrm U} -\theta_k^{\mrm U}(\ell')} 
\end{aligned}
\label{eqn:tu_U_l_diff_bd}    
\end{equation}
where $(i)$ adds and subtracts $\crbk{1-\alpha_{k+1}^{\mrm U}\kappa(\ell,Z_k)}\theta_k^{\mrm U}(\ell')$. 

Using the Lipschitz assumption \eqref{eqn:kappa_lip} on $\kappa$, $\abs{Y_{k+1}^{\mrm U}}\le \frac{1}{1-\gamma}$, and
$\norm{\theta_k^{\mrm U}}\le \frac{1}{1-\gamma}$, we have
$$
\begin{aligned}
\abs{\theta_{k+1}^{\mrm U}(\ell)-\theta_{k+1}^{\mrm U}(\ell')}
&\leq \abs{1-\alpha_{k+1}^{\mrm U}\kappa(\ell,Z_k)} 
\abs{\theta_k^{\mrm U}(\ell)-\theta_k^{\mrm U}(\ell')} + \alpha_{k+1}^{\mrm U}L_{\kappa,\rho} \rho(\ell,\ell') \frac{2}{1-\gamma}\\
&\leq  
\abs{\theta_k^{\mrm U}(\ell)-\theta_k^{\mrm U}(\ell')} + \alpha_{k+1}^{\mrm U}L_{\kappa,\rho} \rho(\ell,\ell') \frac{2}{1-\gamma}
\end{aligned}
$$
where the last inequality uses $\alpha_{k+1}^{\mrm U}\kappa\in[0,1]$. Thus, for all $k\ge 0$
$$
        [\theta_{k+1}^{\mrm U}]_{\rho} - [\theta_k^{\mrm U}]_{\rho}
        \le
        \frac{2L_{\kappa,\rho}}{1-\gamma}\alpha_{k+1}^{\mrm U}.
$$
Summing over $k = 0,\ds ,n-1$, we have that for all $n\ge 0$
$$
[\theta_n^{\mrm U}]_{\rho} - [\theta_0^{\mrm U}]_{\rho} = \sum_{k=0}^{n-1}[\theta_{k+1}^{\mrm U}]_{\rho} - [\theta_k^{\mrm U}]_{\rho} \le \frac{2L_{\kappa,\rho}}{1-\gamma}\sum_{k=0}^{n-1}\alpha_{k+1}^{\mrm U} \leq \frac{2\alpha_{\mrm U} L_{\kappa,\rho}}{1-\gamma}\log \crbk{\frac{n+n_0^{\mrm U}}{n_0^{\mrm U}}}
$$
where the last inequality follows from the harmonic-series bound
$\sum_{k=0}^{n-1}(k+1+n_0^{\mrm U})^{-1}
\le \log\{(n+n_0^{\mrm U})/n_0^{\mrm U}\}$.
Recalling the Lipschitz bound assumption $[\theta_0^{\mrm U}]_{\rho}\leq L_{0,\rho}$ proves \eqref{eqn:LU_def}.

\paragraph{Normalized algorithm.} For the empirical-normalized recursion, by \eqref{eqn:empirical_c} and the choice $n_0^{\mrm N}\ge \alpha_{\mrm N}/c_\wedge$ in Assumption \ref{assump:stepsizes}, we have
$$
        \tilde c_n(\ell)\ge c_\wedge,
        \quad \text{and}\quad
        0\le
        \alpha_{n+1}^{\mrm N}\frac{\kappa(\ell,Z_n)}{\tilde c_n(\ell)}
        \le
        \frac{\alpha_{n+1}^{\mrm N}}{c_\wedge}
        \le1.
$$
Therefore,
$$
        \theta_{n+1}^{\mrm N}(\ell)
        =
        \crbk{
        1-\alpha_{n+1}^{\mrm N}\frac{\kappa(\ell,Z_n)}{\tilde c_n(\ell)}}
        \theta_n^{\mrm N}(\ell)
        +
        \alpha_{n+1}^{\mrm N}
        \frac{\kappa(\ell,Z_n)}{\tilde c_n(\ell)}
        Y_{n+1}^{\mrm N}
$$
is again a convex combination of two points in $[-\frac{1}{1-\gamma},\frac{1}{1-\gamma}]$. So, the same induction argument as in the unnormalized case implies 
$\norm{\theta_n^{\mrm N}}\le \frac{1}{1-\gamma}$ for all $n\ge 0$.

For the Lipschitz constant bound, first note that $$\abs{\int_\Z \kappa(\ell,z) \nu(dz) - \int_\Z \kappa(\ell',z) \nu(dz) }\leq L_{\kappa,\rho} \rho (\ell,\ell')$$ for any probability measure $\nu$ on $\cZ$ and $\ell,\ell'\in\mbb L$. Moreover, $u\ra\max\{u,c_\wedge\}$ is one-Lipschitz.
Hence $\hat c_n$ and $\tilde c_n$ are $L_{\kappa,\rho}$-Lipschitz with respect to $\rho$. 

On the other hand, since $\tilde c_n\ge c_\wedge$ and $0\le\kappa\le1$,
$$
\begin{aligned}
        \abs{
        \frac{\kappa(\ell,z)}{\tilde c_n(\ell)}
        -
        \frac{\kappa(\ell',z)}{\tilde c_n(\ell')}}
        &\le
        \frac{\abs{\kappa(\ell,z)-\kappa(\ell',z)}}{\tilde c_n(\ell)}
        +
        \kappa(\ell',z)
        \abs{
        \frac{1}{\tilde c_n(\ell)}
        -
        \frac{1}{\tilde c_n(\ell')}}                                                   \\
        &\le
        L_{\kappa,\rho}
        \crbk{\frac{1}{c_\wedge}+\frac{1}{c_\wedge^2}}
        \rho(\ell,\ell').
\end{aligned}
$$
Repeating the same argument in \eqref{eqn:tu_U_l_diff_bd} with
$\kappa(s,Z_k)/\tilde c_k(s)$ in place of $\kappa(s,Z_k)$ where $s = \ell,\ell'$, we see that 
$$
        [\theta_{k+1}^{\mrm N}]_{\rho} - [\theta_k^{\mrm N}]_{\rho}
        \le
        \crbk{\frac{1}{c_\wedge}+\frac{1}{c_\wedge^2}}\frac{2L_{\kappa,\rho}}{1-\gamma} \alpha_{k+1}^{\mrm N}.
$$
Again, summing over $k=0,\ds,n-1$ gives \eqref{eqn:LN_def}.
\end{proof}

The regularity of the iterates for both algorithms motivates the use of metric-entropy-type arguments to control their sup-norm errors. Specifically, we control the coefficient-field iterates on deterministic nets of the latent space and extend the bounds to all of $\mbb L$ using the Lipschitz regularity. To support this strategy, we introduce the following notation. 
\begin{definition}[Covering number and $\epsilon$-net]\label{def:eps_net}
An $\epsilon$-net of $(\mbb L,\rho)$ is a subset $\mfk N\subseteq\mbb L$ such that, for every $\ell\in\mbb L$, there exists $\ell'\in\mfk N$ with $\rho(\ell,\ell')\leq\epsilon$. For $\epsilon>0$, let $\cN_\rho(\epsilon)$ denote the covering number of $\mbb L$ by $\rho$-balls of radius $\epsilon$; that is, $\cN_\rho(\epsilon)$ is the smallest cardinality of any $\epsilon$-net. Moreover, let $\mfk N_\rho(\epsilon)$ denote an $\epsilon$-net satisfying $|\mfk N_\rho(\epsilon)|=\cN_\rho(\epsilon)$.
\end{definition}

Since $(\mbb L,\rho)$ is compact, $\cN_\rho(\veps)<\infty$ for every $\veps>0$.  For a finite latent
space with the discrete metric, $\cN_\rho(\veps)\le|\mbb L|$ when $\veps<1$. Compact Euclidean latent
spaces of diameter $\cD$ and dimension $d_\mbb L$ admit the usual covering number bound $\cN_\rho(\veps) \lesssim\crbk{1+ \cD/\epsilon}^{d_{\mbb L}}$.

\section{Convergence Rate of the Unnormalized Algorithm}
\label{sec:unnormalized_rate}

In this section, we present the convergence rate results for the unnormalized algorithm. To simplify notation, define the log-order term
\begin{equation}
 \label{eqn:U_log_term}\cL_n^{\mrm U}(\delta)
        :=
        \log\!\crbk{
        \frac{64\pi^2\alpha_{\mrm U}^2
        \log^2\!\crbk{e(n+n_0^{\mrm U})}}{3\delta}}.   
\end{equation}

\begin{theorem}[Finite-time convergence of the unnormalized latent recursion]
\label{thm:unnormalized_rate}
Suppose Assumptions~\ref{assump:mdp}, \ref{assump:mix},
\ref{assump:stable_kernels}, and \ref{assump:stepsizes} hold. Moreover, choose $\alpha_{\mrm U}c_\wedge(1-\gamma)=:\lambda_\mrm U>1$ and set 
\begin{equation}
        \epsilon_n^{\mrm U}
        := \sqbk{(n+n_0^\mrm U)\crbk{ 2L_{0,\rho} + 5L_{\kappa,\rho}} \log\!\crbk{
        \frac{5 (n+n_0^{\mrm U})}{n_0^{\mrm U} }}}\inv.
        \label{eqn:U_mesh_def}
\end{equation}
Then, for every $\delta\in(0,1)$,  with probability at least $1-\delta$, 
\begin{equation}
\begin{aligned}
        \norm{\theta_n^{\mrm U}-\theta^*}&\le
        \frac{3}{1-\gamma}
        \sqbk{\frac{17 n_0^{\mrm U}}{n+n_0^{\mrm U}}}^{\lambda_{\mrm U}} +
        \frac{328\tmix}{c_\wedge (1-\gamma)^2}
        \sqbk{
        \sqrt{\frac{\alpha_\mrm U(\cL_n^\mrm U(\delta) + \log\cN_\rho(\epsilon_n^\mrm U))}{n+n_0^{\mrm U}}}
        +
        \frac{\alpha_{\mrm U}^{2}}{(\lambda_{\mrm U}-1)(n+n_0^{\mrm U})}}
\end{aligned}
        \label{eqn:unnormalized_main_rate}
\end{equation}
simultaneously for all $n\ge 16\alpha_\mrm U$.
\end{theorem}

\begin{remark}
The condition $\lambda_{\mrm U}>1$ is a simplification used in the theorem. Preserving the leading $n^{-1/2}$ term only requires $\lambda_{\mrm U}>1/2$. The stronger condition $\lambda_{\mrm U}>1$ ensures that the error dynamics does not enlarge the order of the higher-order error term, which is $n^{-1}$. For $\lambda_{\mrm U}\in(1/2,1)$, this higher-order term becomes $n^{-\lambda_{\mrm U}}$.
\end{remark}

The three terms in \eqref{eqn:unnormalized_main_rate} have different origins. The first is the initial transient term. Since the accumulated step size scales as $\alpha_{\mrm U}\log(n)$, evaluating the exponentially stable mean-field ODE with decay time constant $c_\wedge(1-\gamma)$ (cf. Lemma \ref{lemma:U_ODE_exponential_stability}) at this clock yields the initial error decay rate $e^{-c_\wedge(1-\gamma)\alpha_{\mrm U}\log(n)}= n^{-\lambda_{\mrm U}}.$ The second term is generated by the Bellman and Poisson martingales and has the parametric $n^{-1/2}$ rate up to logarithmic factors. The final $n^{-1}$ term collects the transience error from initializing the Markov chain away from stationarity, together with higher-order discretization errors from tracking the mean ODE. We also observe that the regularity parameters $L_{0,\rho}$ and $L_{\kappa,\rho}$ affect this rate only through the covering scale $\epsilon_n^{\mrm U}$.

Importantly, we observe that the slower convergence rate of the ODE \eqref{eqn:U_ODE_pre} compared to the idealized version \eqref{eqn:ideal_ode} affects the overall convergence rate by requiring $\alpha_\mrm U > 1/(c_\wedge (1-\gamma))$. This large step size requirement leads to a factor of at least $c_\wedge^{-3/2}$ in front of the $\widetilde O(n^{-1/2})$ error term.

We highlight that, for a fixed latent space $\mbb L$, the choice of metric $\rho$ can significantly shape the complexity landscape, since the geometry and covering numbers can vary substantially across metrics. At the same time, $\rho$ also enters through the smoothness of $\mrm K$ in Assumption~\ref{assump:stable_kernels}, which affects $\epsilon_n^\mrm U$. Importantly, the algorithm itself is completely agnostic to the choice of $\rho$. Thus, for any fixed $n$ satisfying the assumptions of Theorem~\ref{thm:unnormalized_rate}, the last-iterate error satisfies
\[
\norm{\theta_n^{\mrm U}-\theta^*}\lesssim
        \frac{\tmix}{c_\wedge (1-\gamma)^2}
        \sqrt{\frac{\alpha_\mrm U}{n+n_0^{\mrm U}}}\crbk{\sqrt{\cL_n^\mrm U(\delta)} + \inf_{\rho}\sqrt{ \log\cN_\rho(\epsilon_n^\mrm U)}},
\]
where the infimum is taken over all admissible metrics $\rho$ for which $(\mbb L,\rho)$ is a compact metric space, and $\kappa$ is Lipschitz. This highlights that the statistical complexity of the SA algorithm automatically adapts to the structure of $\mbb L$ through the best admissible choice of metric $\rho$. The possibility of a complexity reduction by this automatic adaptation becomes clear when we consider the frozen neural network application in Section \ref{subsec:pretrained_decoder}.

For generic latent spaces $\mbb L$,  direct application of Theorem~\ref{thm:unnormalized_rate} yields the following corollary, which specializes the error bound to two important cases: when $\mbb L$ is finite and when $\mbb L$ is a bounded subset of a Euclidean space.

\begin{corollary}[Finite and Euclidean latent spaces for the unnormalized recursion]
\label{cor:unnormalized_covering_specializations}
Suppose the hypotheses of Theorem~\ref{thm:unnormalized_rate} hold.  In each of the following cases, the
displayed bound holds simultaneously for all $n\ge 16\alpha_{\mrm U}$ on an event of probability at least
$1-\delta$.
\begin{enumerate}[label=(\roman*)]
\item If $\mbb L$ is finite with cardinality $J$ and is equipped with the discrete metric, then
\[\norm{\theta_n^{\mrm U}-\theta^*}
        \le
        \frac{3}{1-\gamma}
        \sqbk{\frac{17n_0^{\mrm U}}{n+n_0^{\mrm U}}}^{\lambda_{\mrm U}} 
        +\frac{328\tmix}{c_\wedge(1-\gamma)^2}
        \sqbk{
        \sqrt{\frac{\alpha_{\mrm U}(\cL_n^{\mrm U}(\delta)+\log J)}{n+n_0^{\mrm U}}}
        +
        \frac{\alpha_{\mrm U}^{2}}{(\lambda_{\mrm U}-1)(n+n_0^{\mrm U})}
        }.\]

\item Suppose there are constants $\cD\ge0$ and $d_{\mbb L}\ge1$ such that for all $\epsilon>0$
\begin{equation}
        \cN_\rho(\veps)
        \le
        \crbk{1+\frac{\cD}{\epsilon}}^{d_{\mbb L}}. \label{eqn:polynomial_covering_condition_unnormalized}
\end{equation}
Then\[
        \norm{\theta_n^{\mrm U}-\theta^*}\le
        \frac{3}{1-\gamma}
        \sqbk{\frac{17n_0^{\mrm U}}{n+n_0^{\mrm U}}}^{\lambda_{\mrm U}}+
        \frac{328\tmix}{c_\wedge(1-\gamma)^2}
        \sqbk{
        \sqrt{
        \frac{
        \alpha_{\mrm U}(\cL_n^{\mrm U}(\delta)
        +d_{\mbb L}\log\!\crbk{1+\cD/\epsilon_n^{\mrm U}})}
        {n+n_0^{\mrm U}}}
        +
        \frac{\alpha_{\mrm U}^{2}}{(\lambda_{\mrm U}-1)(n+n_0^{\mrm U})}
        }.\]
\end{enumerate}
\end{corollary}

We note that Condition \eqref{eqn:polynomial_covering_condition_unnormalized} is the usual Euclidean covering bound. For example, if $\mbb L\subset\R^{d_{\mbb L}}$ is compact and $\rho$ is the Euclidean metric, then the elementary grid bound for Euclidean metric entropy implies \eqref{eqn:polynomial_covering_condition_unnormalized} with $\cD=2\diam(\mbb L)$; see, for instance, \citet[Chapter~2]{vandervaartwellner1996}.

Thus, Corollary~\ref{cor:unnormalized_covering_specializations} makes the effective dimension dependence explicit. For a finite latent space, the leading stochastic term scales as $\sqrt{\log J/n}$, up to the common
logarithmic factors. Under Euclidean-type metric entropy, it scales as $\sqrt{d_{\mbb L}\log \cD /n}$.  Thus, even when $\mbb L$ is infinite and hence $C(\mbb L)$ is infinite dimensional, the Lipschitz regularity of the algorithm iterates still enables a canonical $n^{-1/2}$ rate, while the latent-space diameter enters logarithmically.

\section{Proof of Theorem \ref{thm:unnormalized_rate}}
\label{sec:proof_unnormalized_rate}

This section presents the proof of Theorem~\ref{thm:unnormalized_rate} for the unnormalized algorithm. Owing to its simpler design, the analysis is cleaner and serves as a useful preparation for the normalized case.

We prove Theorem~\ref{thm:unnormalized_rate} by jointly analyzing the convergence rate of the mean-field ODE and the ODE-tracking error of the iterates. The SA recursion evolves in discrete time with iteration index $n$, whereas the mean-field ODE evolves on a continuous clock. We therefore align these two clocks and compare the SA recursion with the mean-field ODE over fixed intervals of ODE time. Specifically, the proof proceeds in four steps, each corresponding to a subsection below:
\begin{enumerate}
    \item Establish the stability of the mean-field ODE; see Lemma~\ref{lemma:U_ODE_exponential_stability}.
    \item Prove a deterministic ODE tracking error bound for the SA recursion over any fixed window of the ODE clock.
    \item Decompose the within-window ODE tracking error into deterministic bias and stochastic error terms, and analyze them separately.
    \item Partition the full algorithm horizon into epochs of constant ODE-clock length. Apply the within-window error bound and solve the resulting inter-epoch recursion to obtain the final error bound.
\end{enumerate}

\subsection{Existence, Uniqueness, and Stability of the Mean-Field ODE}

Define the vector field for the mean-field ODE as $\cH_{\mrm U}\theta:=c(\cG\theta-\theta).$
Note that for all $\theta,\theta'\in C(\mbb L)$,
\begin{equation}\label{eqn:HU_Lip}
    \norm {\cH_\mrm U\theta- \cH_\mrm{U}\theta'}\leq \norm{c}(\norm{\cG\theta - \cG\theta'} + \norm{\theta - \theta'} )\leq (1+\gamma)\norm{\theta-\theta'}
\end{equation}
In particular, $\cH_\mrm U$ is $1+\gamma$-Lipschitz on the Banach space $(C(\mbb L), \norm{\cd})$. Therefore, the classical existence and uniqueness result for ODE solutions applies. For completeness, we restate the version from \citet{brezis2011functional} in Lemma~\ref{lemma:ODE_exist_uniques_sol} in the appendix.

\begin{lemma}[Exponential stability]
\label{lemma:U_ODE_exponential_stability}
For $\theta\in\Theta$, there exists a unique solution $\vartheta_{\mrm U}(t;\theta)\in\Theta$ for all $t\ge 0$ to 
\begin{equation}
        \vartheta_{\mrm U}(t;\theta)
        =
        \theta+
        \int_0^t\cH_{\mrm U}\vartheta_{\mrm U}(s;\theta)ds        \label{eqn:unnormalized_ODE_integral_form}
\end{equation}
such that
\begin{equation}
        \norm{\vartheta_{\mrm U}(t;\theta)-\theta^*}
        \le
        e^{-c_\wedge(1-\gamma)t}\norm{\theta-\theta^*}.
        \label{eqn:U_ODE_exponential_stability}
\end{equation}
\end{lemma}

\begin{proof}[Proof of Lemma \ref{lemma:U_ODE_exponential_stability}]
First, we note that $\cH_\mrm{U}:C(\mbb L)\ra C(\mbb L)$ is Lipschitz. Hence, Lemma \ref{lemma:ODE_exist_uniques_sol} implies the existence and uniqueness of the solution to \eqref{eqn:unnormalized_ODE_integral_form}. Next, we show that this solution remains in $\Theta$. To this end, we consider a clipped version of the operator $\cH_\mrm U$ defined by
$$
        \overline \cH_{\mrm U}\theta:
        = \cH_\mrm U \Pi \theta + \Pi\theta - \theta = 
        (1-c)\Pi\theta +c\cG \Pi\theta -\theta.
$$
where $\Pi$ is the clipping operator
$$
        \Pi \theta(\ell):=
        \max\!\left\{\min\!\left\{\theta(\ell),\frac{1}{1-\gamma}\right\},-
        \frac{1}{1-\gamma}\right\}.
$$
Since $\Pi$ is 1-Lipschitz, $\overline \cH_{\mrm U}$ is also globally Lipschitz on $C(\mbb L)$, and it agrees with $\cH_{\mrm U}$ on $\Theta$.  By
Lemma~\ref{lemma:ODE_exist_uniques_sol}, the ODE with vector field $\overline\cH_{\mrm U}$ and initial condition $\theta\in\Theta$ has a unique solution, denoted by $\bar \vartheta_\mrm U(t;\theta)$.

Moreover, we use the integration factor $e^{t}$ (i.e., in differential form $\frac{d}{dt}{[e^tf(t)]} = e^tf(t) + e^t\frac{d}{dt}f(t)$) to see that
$$
\begin{aligned}
    e^t\bar\vartheta_\mrm U(t;\theta) &=\theta + \int_0^te^{s} \sqbk{\overline \cH_\mrm U\bar \vartheta_\mrm U(s;\theta) + \bar \vartheta_\mrm U(s;\theta)}ds \\
    &=\theta + \int_0^te^{s} \sqbk{(1-c)\Pi\bar \vartheta_\mrm U(s;\theta) +c\cG \Pi\bar \vartheta_\mrm U(s;\theta)}ds.
\end{aligned}
$$
Since $\theta\in\Theta$, $c(\ell)\in[0,1]$ for all $\ell\in\mbb L$, and $\cG$ is a $\gamma$-contraction in the sup norm by Proposition \ref{prop:latent_fixed_point}, we have that for all $t\ge 0$ 
$$
\norm {\bar \vartheta_\mrm U(t;\theta)}\leq e^{-t}\norm\theta + \int_0^t  e^{-(t-s)}\norm{\Pi\bar \vartheta_\mrm U(s;\theta)}ds \leq \frac{1}{1-\gamma};
$$
i.e. $\bar  \vartheta_\mrm U(t;\theta)\in\Theta$. Notice that $\overline \cH_\mrm U = \cH_\mrm U $ on $\Theta$. So, the unique solutions $\vartheta_\mrm U(t;\theta) = \bar\vartheta_\mrm U(t;\theta)$ agree and hence $\vartheta_\mrm U(t;\theta)\in\Theta$ for all $t\ge 0$.

We now prove the exponential stability estimate \eqref{eqn:U_ODE_exponential_stability}.  Write $\epsilon(t):=\vartheta_{\mrm U}(t;\theta)-\theta^*$. As $\cG\theta^*=\theta^*$ and hence $\cH_\mrm U \theta^*= 0 $,  we have
$$
\begin{aligned}
        \epsilon(t)
        &=
        \epsilon(0)+
        \int_0^t
        \sqbk{\cH_{\mrm U}\vartheta_{\mrm U}(s;\theta)-\cH_{\mrm U}\theta^*}ds \\
        &= \epsilon(0)+
        \int_0^t
        \sqbk{c(\cG\vartheta_{\mrm U}(s;\theta)-\cG\theta^*) + (1-c)\epsilon(s) - \epsilon(s)}ds.
\end{aligned}
$$
Applying the integration factor $e^t$ as before, 
\begin{equation}\label{eqn:tu_U_eps_t}
        e^t\epsilon(t)
        =
        \epsilon(0)+
        \int_0^t e^s
        \sqbk{c(\cG\vartheta_{\mrm U}(s;\theta)-\cG\theta^*) + (1-c)\epsilon(s) }ds.
\end{equation}

Since $\cG$ is a $\gamma$-contraction in $\norm{\cd}$ and $1\ge c\geq c_\wedge\ge 0$,  $$\norm{c(\cG\vartheta_{\mrm U}(s;\theta)-\cG\theta^*) + (1-c)\epsilon(s)}\leq  \norm{c\gamma + (1-c)}\norm{\epsilon(s)}\leq [1-(1-\gamma)c_\wedge]\norm{\epsilon(s)} $$
Thus, defining $x(t)=e^t\norm{\epsilon(t)}$ and taking the sup norm on both sides of \eqref{eqn:tu_U_eps_t}, we get
$$
        x(t)
        \le
        x(0)+
        (1-(1-\gamma)c_\wedge)\int_0^t x(s)ds.
$$
Then, Gr\"onwall's inequality implies $x(t)\le x(0)e^{[1-(1-\gamma)c_\wedge]t}$, and hence $\norm{\epsilon(t)}\leq e^{-(1-\gamma)c_\wedge t}\norm{\epsilon(0)}.$
This proves \eqref{eqn:U_ODE_exponential_stability}.
\end{proof}

\subsection{Time-Scale Conversion and ODE Tracking Error}

Our proof follows from the classical idea in the SA literature that views the stochastic recursion as an Euler discretization of the mean-field ODE, plus a stochastic perturbation, see \citet{borkar2000ode}. Specifically, define the unnormalized perturbation by
\begin{equation}
        \xi_{k+1}^{\mrm U}(\ell)
        =
        \kappa(\ell,Z_k)
        \crbk{Y_{k+1}^{\mrm U}-\theta_k^{\mrm U}(\ell)}
        -
        c(\ell)
        \crbk{\cG\theta_k^{\mrm U}(\ell)-\theta_k^{\mrm U}(\ell)} .
        \label{eqn:U_xi_def}
\end{equation}
Then the update rule \eqref{eqn:unnormalized_update} of the unnormalized algorithm can be written as
\begin{equation}
        \theta_{k+1}^{\mrm U}
        =
        \theta_k^{\mrm U}
        +
        \alpha_{k+1}^{\mrm U}\cH_{\mrm U}\theta_k^{\mrm U}
        +
        \alpha_{k+1}^{\mrm U}\xi_{k+1}^{\mrm U}.
        \label{eqn:unnormalized_SA}
\end{equation}

Here, we note that the perturbation term $\xi_{k+1}^{\mrm U}$ is not a martingale difference with respect to the natural filtration, because $Z_k$ is Markov and the chain need not be stationary. However, its stationary mean is zero, and the initial bias away from stationarity will be handled in Section~\ref{subsec:U_stochastic_error} by a standard Poisson's equation technique. In this section, however, we mainly focus on establishing the error bound between the mean-field ODE and the SA recursion in terms of the perturbation sequence. 

We now introduce the ODE clock corresponding to the SA iterates.  For any $m\ge n$, set
\begin{equation}
        \tau_\mrm U(n,m)
        =
        \sum_{k=n}^{m-1}\alpha_{k+1}^{\mrm U}.
        \label{eqn:window_notation}
\end{equation}
We note that $\tau_{\mrm U}(n,m)$ is interpreted as the ODE time elapsed while the stochastic approximation runs from iteration
$n$ up to iteration $m$. Indeed, an Euler discretization of the ODE
\begin{equation}\label{eqn:U_window_ODE}
    \vartheta_{\mrm U}(\tau_{\mrm U}(n,m);\theta_n^{\mrm U})
    =
    \theta_n^{\mrm U}
    +
    \int_0^{\tau_{\mrm U}(n,m)}
    \cH_{\mrm U}\vartheta_{\mrm U}(s;\theta_n^{\mrm U})\,ds,
\end{equation}
using the time grid $\set{\alpha_k^{\mrm U}:k=n+1,\ds,m}$ leads to the recursion
$$\vartheta^{\mrm U}_{k+1} = \vartheta_k^{\mrm U} + \alpha_{k+1}^{\mrm U}\cH_{\mrm U}\vartheta_k^{\mrm U}, \qquad \vartheta_n^{\mrm U}=\theta_n^{\mrm U},
$$
which is exactly the noise-free version of \eqref{eqn:unnormalized_SA}. Thus, if the current estimator of $\theta^*$ from the SA algorithm is $\theta_n^{\mrm U}$, then after running the recursion to $m\ge n$, the ODE object with which it should be compared is $\vartheta_{\mrm U}(\tau_{\mrm U}(n,m);\theta_n^{\mrm U})$. 

On the other hand, to analyze the convergence rate guarantee in Theorem \ref{thm:unnormalized_rate}, we look at the ODE tracking error within an ODE clock interval $T$. In particular, for a fixed ODE clock period $T>0$, define the corresponding SA iteration window by
\begin{equation}
        \cW_{\mrm U}(n,T)
        =
        \setcond{m\ge n}{\tau_{\mrm U}(n,m)\le T}.
        \label{eqn:generic_window_notation}
\end{equation}
We note for later use that we will choose $T=1/4$, and some useful properties of this step size and window size are established in Lemma~\ref{lemma:step size_window}. Nevertheless, the following Lemma holds for all fixed $T > 0$. 

\begin{lemma}[ODE tracking error bound]
\label{lemma:U_ode_tracking} The tracking error between the SA iterates and the mean-field ODE at the corresponding $\tau_\mrm U$ clock satisfies the following deterministic bound: 
\begin{equation}
        \max_{m\in\cW_{\mrm U}(n,T)}
        \norm{
        \theta_m^{\mrm U}-\vartheta_{\mrm U}(\tau_{\mrm U}(n,m);\theta_n^{\mrm U})}
        \le
        e^{2T}\sqbk{\max_{m\in\cW_{\mrm U}(n,T)}\norm{\sum_{k=n}^{m-1}\alpha_{k+1}^\mrm U\xi_{k+1}^\mrm U} + \frac{2 T\alpha_{n+1}^\mrm U}{1-\gamma}} .
        \label{eqn:U_tracking}
\end{equation}
\end{lemma}

\begin{proof}[Proof of Lemma~\ref{lemma:U_ode_tracking}]

Fix $n\ge 0$. For simplicity, we write $\tau_k=\tau_{\mrm U}(n,k)$ for all $k\ge n$.  

First, we note that from the SA recursion \eqref{eqn:unnormalized_SA} 
$$\theta_m^\mrm U = \theta_n^\mrm U + \sum_{k=n}^{m-1}\alpha_{k+1}^\mrm U \cH_\mrm U \theta_k^\mrm U+ \sum_{k=n}^{m-1} \alpha_{k+1}\xi_{k+1}^\mrm U$$
for all $m\ge n$. We then subtract the ODE \eqref{eqn:U_window_ODE} to get that for $m\ge n$,
\begin{equation}\label{eqn:tu_U_euler_error_bd}
\begin{aligned}
        \theta_m^{\mrm U}- \vartheta_{\mrm U}(\tau_m;\theta_n^{\mrm U})&= 
          \sum_{k=n}^{m-1}\alpha_{k+1}^\mrm U \cH_\mrm U \theta_k^\mrm U+ \sum_{k=n}^{m-1} \alpha_{k+1}\xi_{k+1}^\mrm U - \int_0^{\tau_\mrm U(n,m)}\cH_\mrm U\theta_\mrm U(s;\theta_n^\mrm U)ds\\
        &=\underbrace{\sum_{k=n}^{m-1}\alpha_{k+1}^{\mrm U}
        \sqbk{\cH_{\mrm U}\theta_k^{\mrm U}-\cH_{\mrm U}\vartheta_\mrm U(\tau_k;\theta_n^\mrm U)}}_{D_1}  \\
        &\quad+\underbrace{\sum_{k=n}^{m-1}
        \sqbk{
        \alpha_{k+1}^{\mrm U}\cH_{\mrm U}\vartheta_\mrm U(\tau_k;\theta_n^\mrm U)
        -\int_{\tau_k}^{\tau_{k+1}}\cH_{\mrm U}\vartheta_{\mrm U}(s;\theta_n^{\mrm U})ds}
        }_{D_2}
        +\sum_{k=n}^{m-1}\alpha_{k+1}^{\mrm U}\xi_{k+1}^{\mrm U}.
\end{aligned}
\end{equation}

We then bound the two terms $D_1$ and $D_2$. For $D_1$, recall from \eqref{eqn:HU_Lip} that $\cH_\mrm U$ is $1+\gamma\leq 2$-Lipschitz. So,
$$\norm {D_1}\le 2\sum_{k=n}^{m-1}\alpha_{k+1}^\mrm U \norm{\theta_k - \vartheta_\mrm U(\tau_k;\theta_n)}.$$
On the other hand, since $\tau_{k+1} -\tau_k = \alpha_{k+1}^\mrm U$, we have
$$\begin{aligned}
    \norm {D_2} &\leq \sum_{k=n}^{m-1}\int_{\tau_k}^{\tau_{k+1}}\norm{\cH_\mrm U\vartheta_\mrm U(\tau_k;\theta_n^\mrm{U}) - \cH_\mrm U\vartheta_\mrm U(s,\theta_n^\mrm U)}ds\\
    &\leq 2\sum_{k=n}^{m-1}\int_{\tau_k}^{\tau_{k+1}}\norm{\vartheta_\mrm U(\tau_k;\theta_n^\mrm{U}) -\vartheta_\mrm U(s,\theta_n^\mrm U)}ds\\
    &\stackrel{(i)}{=}2\sum_{k=n}^{m-1}\int_{\tau_k}^{\tau_{k+1}}\norm{\int_{\tau_k}^s \cH_\mrm U \vartheta(r;\theta_n^\mrm U)dr}ds\\
    &\leq \frac{2}{1-\gamma}\sum_{k=n}^{m-1} \crbk{\alpha_{k+1}^\mrm U}^2 
\end{aligned}$$
where $(i)$ follows from the ODE \eqref{eqn:U_window_ODE}. The last inequality follows from $\theta_n^\mrm U\in \Theta$ in Proposition \ref{prop:lipschitz_iterates}, the stability $\vartheta(r;\theta_n^\mrm U)\in\Theta$ in Lemma \ref{lemma:U_ODE_exponential_stability}, and that $\norm{\cH_{\mrm U}(\theta)}\le\frac{2}{1-\gamma}$ for $\theta\in \Theta$.

Denote $R_m:=\norm{\theta_m^{\mrm U}-\vartheta_{\mrm U}(\tau_m;\theta_n^{\mrm U})}.$
From \eqref{eqn:tu_U_euler_error_bd}, we have that
$$
R_m\leq 2\sum_{k=n}^{m-1} \alpha_{k+1}^\mrm U R_k+ \sum_{k=n}^{m-1} \frac{2\crbk{\alpha_{k+1}^\mrm U}^2}{1-\gamma}  + \norm{\sum_{k=n}^{m-1} \alpha_{k+1}^{\mrm U}\xi_{k+1}^{\mrm U}}.
$$
Thus, we apply the discrete Gr\"onwall's inequality in Lemma \ref{lemma:discrete_gronwall} to get that for all $j\in \cW_\mrm U(n,T)$
\begin{align*}
R_j&\leq \sqbk{\max_{n\leq m\leq j}\norm{\sum_{k=n}^{m-1}\alpha_{k+1}^\mrm U \xi_{k+1}^\mrm U}+ \frac{2}{1-\gamma}\sum_{k=n}^{j-1}\crbk{\alpha_{k+1}^\mrm U}^2}\exp\crbk{2\sum_{k=n}^{j-1}\alpha_{k+1}^\mrm U}\\
&\stackrel{(i)}{\leq} e^{2T}\sqbk{\max_{m\in\cW_{\mrm U}(n,T)}\norm{\sum_{k=n}^{m-1}\alpha_{k+1}^\mrm U\xi_{k+1}^\mrm U} + \frac{2}{1-\gamma}\sum_{k=n}^{j-1}\crbk{\alpha_{k+1}^\mrm U}^2}\\
&\stackrel{(ii)}{\leq} e^{2T}\sqbk{\max_{m\in\cW_{\mrm U}(n,T)}\norm{\sum_{k=n}^{m-1}\alpha_{k+1}^\mrm U\xi_{k+1}^\mrm U} + \frac{2 T\alpha_{n+1}^\mrm U}{1-\gamma}}
\end{align*}
where both $(i)$ and $(ii)$ use the definition of $\cW_\mrm U(n,T)$ so that $\sum_{k=n}^{j-1}\alpha_{k+1}^\mrm U = \tau_\mrm U(n,j) \leq T$; in addition, $(ii)$ uses the step size being non-increasing so that 
\begin{equation}\sum_{k=n}^{j-1} \crbk{\alpha_{k+1}^\mrm U}^2 \leq \alpha_{n+1}^\mrm U\sum_{k=n}^{j-1} \alpha_{k+1}^\mrm U \leq T\alpha_{n+1}^\mrm U\label{eqn:U_sum_alpha_square_bd}
\end{equation} 

Finally, taking maximum over $j\in\cW_\mrm U(n,T)$, we conclude Lemma \ref{lemma:U_ode_tracking}. 
\end{proof}

\subsection{Within-Window Error Decomposition and Bounds}
\label{subsec:U_stochastic_error}

Lemma~\ref{lemma:U_ode_tracking} shows that, within a $T$-window of the ODE clock, the error between the SA recursion and the restarted mean-field ODE is controlled by a running maximum of step size-weighted cumulative stochastic perturbations, plus a higher-order term of order $\alpha_{n+1}^{\mrm U}=\Theta(n^{-1})$. In this section, we show that the leading term
$$
\max_{m\in \cW_{\mrm U}(n,T)}
        \norm{
        \sum_{k=n}^{m-1}\alpha_{k+1}^{\mrm U}\xi_{k+1}^{\mrm U}} = \widetilde  O(n^{-1/2})
$$
with high probability. 

\paragraph{Error decomposition.} We first decompose the perturbation sequence
$\set{\xi_k^{\mrm U}:k\ge1}$ into two sources.  The first is the Bellman target noise, which is a
martingale difference under the filtration $\{\cF_k\}_{k\ge0}$.  The second comes from Markov sampling:
the data-generating process differs from stationary sampling under $\mu_b$, which defines
$\cG=\mrm K\cT\Phi$.

Specifically, recall that $q_k^\mrm U=\Phi\theta_k^{\mrm U}$ and the definition of $\xi_{k}^\mrm U$ in \eqref{eqn:U_xi_def}. Then, we write
\begin{equation}\label{eqn:U_noise_two_error}
\begin{aligned}
    \xi_{k+1}^\mrm U(\ell)&= \kappa(\ell,Z_k)(Y_{k+1}^\mrm U - \theta_k^\mrm U(\ell)) - 
        c(\ell)
        \crbk{\cG\theta_k^{\mrm U}(\ell)-\theta_k^{\mrm U}(\ell)}\\
        &= \underbrace{\kappa(\ell,Z_k)(Y_{k+1}^\mrm U - \mathcal T \Phi \theta_k^\mrm U(Z_k) )}_{=:B_{k+1}^\mrm U(\ell )} \\
        &\quad + \underbrace{\kappa(\ell,Z_k)(\mathcal T \Phi \theta_k^\mrm U(Z_k) - \theta_k^\mrm U(\ell)) -
        \int_\Z
        \kappa(\ell,z)
        \crbk{\mathcal T\Phi \theta_k^\mrm U(z)-\theta_k^{\mrm U}(\ell)}\mu_b(dz)}_{=:M_{k}^\mrm U(\ell)}
\end{aligned}
\end{equation}
where the second equality follows from the definitions of $\mrm K$ and $c$.

Moreover, by \eqref{eqn:conditional_target_identity}, $E[Y_{k+1}^{\mrm U}|\cF_k]=\mathcal T q_k^\mrm U(Z_k) =\mathcal T \Phi \theta_k^\mrm U(Z_k) .$ So, $E[B_{k+1}^\mrm U|\cF_k] = 0$; i.e. $\set{B_k^\mrm U:k\ge 1}$ is a martingale difference sequence in $C(\mbb L)$. On the other hand, one can interpret $M_{k}^\mrm U$ as the Markov sampling error for estimating the mean-field operator $\cH_\mrm U$; i.e. 
$M_{k}^\mrm U=  \widehat \cH_k^\mrm U \theta_k^\mrm U - \cH_{\mrm U}\theta_k^\mrm U,$ where $\widehat \cH^\mrm U$ is the Markov one-sample estimate of $\cH_\mrm U$ defined by $\widehat \cH^\mrm U_k \theta (\ell) := \kappa(\ell,Z_k)[\cT\Phi\theta(Z_k) - \theta(\ell)]$ for all $\theta\in C(\mbb L)$. 
Thus, we refer to $B_{k+1}^\mrm U$ and $M_{k}^\mrm U$ as Bellman and Markov noise, respectively.

\paragraph{Decomposing the Markov noise using Poisson's equation.} Since the Bellman noise is naturally a martingale difference, the martingale concentration suggests a parametric convergence rate for the step-size-weighted average. On the other hand, the Markov noise can be further decomposed into a martingale difference sequence and a transient bias term caused by the non-stationary initialization of the Markov chain $\set{Z_k:k\geq 0}$ via Poisson's equation.

Specifically, define the random function
\begin{align*}
m_{k}^\mrm U(\ell,z)&:= \kappa(\ell,z)\crbk{\mathcal T \Phi \theta_k^\mrm U(z) - \theta_k^\mrm U(\ell)} -
        \int_\Z
        \kappa(\ell,y)
        \crbk{\mathcal T\Phi \theta^\mrm U_k(y)-\theta_k^{\mrm U}(\ell)}\mu_b(dy). 
\end{align*}
Then, $M_k^{\mrm U}(\ell)=m_{k}^{\mrm U}(\ell ,Z_k)$ and $m_k^\mrm U(\ell,\cd)$ is $\mu_b$-centered for all $k\ge 0$ and $\ell\in\mbb L$; i.e. $$\int_\Z m_{k}^{\mrm U}(\ell,z)\mu_b(dz)=0.$$  Moreover, recall from Proposition \ref{prop:lipschitz_iterates} that $\norm{\theta_k^\mrm U}\leq \frac{1}{1-\gamma}$. Hence, for any $\ell\in\mbb L$ and $z\in\Z$,  $|m_{k}^\mrm U(\ell,z)|\leq \frac{2}{1-\gamma}. $

For each fixed $\ell\in\mbb L$, we consider bounded and centered solutions $v_k^\mrm U(\ell,\cd)$ to the Poisson equation
\begin{equation}\label{eqn:U_Poisson_eqn} v(\ell,z)-\int_\Z v(\ell,y) P_b(dy|z)=m_{k}^{\mrm U}(\ell,z),\quad \forall z\in\Z; \quad \int_\Z v(\ell,z)\mu_b(dz) = 0.     
\end{equation} Since $m_k^{\mrm U}$ is bounded, by Lemma \ref{lemma:poisson_general}, the unique bounded and centered solution to \eqref{eqn:U_Poisson_eqn} is 
\begin{equation}\label{eqn:U_Poisson_sol}
        v_{k}^{\mrm U}(\ell,z)
        =
        \sum_{t=0}^{\infty} \int_{\Z}m_{k}^{\mrm U}(\ell,y)P_b^t(dy|z),\quad \text{with}\quad \norm{v_{k}^{\mrm U}(\ell,\cd)}\leq \frac{4\tmix}{1-\gamma}.
\end{equation}

With this notation, we see that by \eqref{eqn:U_Poisson_eqn}
\begin{equation}
\begin{aligned}
        M_k^{\mrm U}(\ell)
        &=m_k^{\mrm U}(\ell,Z_k)\\
        &=v_k^{\mrm U}(\ell,Z_k) - \int_\Z v_k^{\mrm U}(\ell,y)P_b(dy|Z_k)\\
        &=v_k^{\mrm U}(\ell,Z_k)-v_k^{\mrm U}(\ell,Z_{k+1})
          + \underbrace{v_k^{\mrm U}(\ell,Z_{k+1}) - E[v_k^\mrm U(\ell,Z_{k+1})|\cF_k]}_{=:P_{k+1}^\mrm U(\ell)},
\end{aligned} 
        \label{eqn:U_poisson_decomposition_in_text}
\end{equation}
The last equality uses that $v_k^{\mrm U}(\ell,\cdot)$ is $\cF_k$-measurable, because it is determined by
$\theta_k^{\mrm U}$.  Thus $P_{k+1}^{\mrm U}$ is a martingale difference.

Thus, for $m\in\cW_\mrm U(n,T)$, the stochastic error can be decomposed as
\begin{equation}
\label{eqn:U_within_window_decomp}
    \sum_{k=n}^{m-1}\alpha_{k+1}^{\mrm U}\xi_{k+1}^{\mrm U} = \sum_{k=n}^{m-1}\alpha_{k+1}^{\mrm U} \crbk{v_{k}^\mrm U(\cd,Z_{k}) - v_{k}^\mrm U(\cd,Z_{k+1})} + \sum_{k=n}^{m-1}\alpha_{k+1}^{\mrm U} \crbk{B_{k+1}^ \mrm U + P_{k+1}^\mrm U}.
\end{equation}
The first sum in \eqref{eqn:U_within_window_decomp} need not have zero mean and can be shown to be $\widetilde O(n^{-1})$ using standard telescoping-sum techniques. We therefore state the following lemma and defer its proof to the appendix.
\begin{lemma}\label{lemma:U_bias_error_bd} Under the current setup, for any $m\in\cW_\mrm U(n,T)$,
    $$\norm{\sum_{k=n}^{m-1}\alpha_{k+1}^{\mrm U}\crbk{v_{k}^\mrm U(\cd,Z_{k}) - v_{k}^\mrm U(\cd,Z_{k+1})}} \leq  8(1+T)\frac{\tmix\alpha_\mrm U  }{(1-\gamma)(n+n_0^\mrm U)}. 
    $$
\end{lemma}

\paragraph{Convergence bound for the $C(\mbb L)$-valued martingale.} 

We next analyze the martingale term in \eqref{eqn:U_within_window_decomp}. While martingale error terms typically exhibit good statistical behavior analogous to i.i.d. sequences, our setting introduces an additional challenge: the martingale takes values in the infinite-dimensional space $C(\mbb L)$. 

Fortunately, Proposition~\ref{prop:lipschitz_iterates} shows that the iterates $\set{\theta_n^{\mrm U}:n\geq 0}$ remain Lipschitz, with Lipschitz constants scaling as $\widetilde O(1)$ as the algorithm evolves. Moreover, this smoothness and the corresponding Lipschitz estimates directly imply smoothness of the martingale differences $\set{B_n^{\mrm U},P_n^{\mrm U}:n\geq 0}$. Thus, bounding the sup-norm error reduces to bounding the error on an $\epsilon$-net of $\mbb L$, together with the interpolation error, which depends on $\epsilon$.

Specifically, we start with the following simple sup-error bound for Lipschitz functions using an
$\epsilon$-net.  Recall the definition of $\mfk N_\rho(\epsilon)$ in
Definition~\ref{def:eps_net}.  Then $|\mfk N_\rho(\epsilon)|=\cN_\rho(\epsilon)$, and every Lipschitz
function $h:\mbb L\to\R$ satisfies
\begin{equation}
        \norm{h}
        \le
        \max_{\ell\in\mfk N_\rho(\epsilon)}|h(\ell)|+[h]_{\rho}\epsilon
        .
        \label{eqn:latent_net_general}
\end{equation}
Indeed, for every $\ell$, we can find $\ell'\in\mfk N_\rho(\epsilon)$ such that
$\rho(\ell,\ell')\le\epsilon$. Also, by Lipschitzness,
$|h(\ell)|\le |h(\ell')|+[h]_{\rho}\epsilon$. This implies \eqref{eqn:latent_net_general}. 

To apply \eqref{eqn:latent_net_general}, we establish bounds on the Lipschitz constants of $B_{n}^\mrm U$ and $P_n^\mrm U$. From \eqref{eqn:U_noise_two_error}, we have that for all $\ell,\ell'\in\mbb L$ and $n\ge 0$,
$$
\begin{aligned}
\abs{B_{n+1}^\mrm U(\ell) - B_{n+1}^\mrm U(\ell')} &= \abs{\kappa(\ell,Z_{n}) - \kappa(\ell',Z_{n})}\abs{Y_{n+1}^\mrm U - \mathcal T \Phi \theta_n^\mrm U(Z_n) }\leq\frac{2 L_{\kappa,\rho}}{1-\gamma}\rho(\ell,\ell')
\end{aligned}
$$
where we used the boundedness $\abs{Y_{n+1}^\mrm U} + \abs{\mathcal T \Phi \theta_n^\mrm U(Z_n)}\leq2/(1-\gamma) $ in Proposition \ref{prop:lipschitz_iterates}. 

On the other hand, for $P_{n+1}^{\mrm U}$ defined in
\eqref{eqn:U_poisson_decomposition_in_text},
$$
\begin{aligned}
\abs{P_{n+1}^{\mrm U}(\ell)-P_{n+1}^{\mrm U}(\ell')}
&\le
2\norm{v_n^{\mrm U}(\ell,\cdot)-v_n^{\mrm U}(\ell',\cdot)}.
\end{aligned}
$$

By linearity of Poisson's equation,
$v_n^{\mrm U}(\ell,\cdot)-v_n^{\mrm U}(\ell',\cdot)$ is the unique bounded centered solution of
\eqref{eqn:poisson_eqn} with right-hand side
$m_n^{\mrm U}(\ell,\cdot)-m_n^{\mrm U}(\ell',\cdot)$.  Thus, by  Lemma \ref{lemma:poisson_general}, 
$$
\begin{aligned}
|P_{n+1}^\mrm U(\ell)- P_{n+1}^\mrm U (\ell')| &\leq 4\tmix \norm{m_n^\mrm U(\ell,\cd) - m_n^\mrm U(\ell',\cd)} \\
&\leq 8\tmix \sup_{z\in\Z}\abs{\kappa(\ell,z)\crbk{\cT\Phi\theta_n^\mrm U(z) - \theta_n^\mrm U(\ell)} - \kappa(\ell',z)\crbk{\cT\Phi\theta_n^\mrm U(z) - \theta_n^\mrm U(\ell')}}\\
&\leq 8\tmix \sup_{z\in\Z}\abs{(\kappa(\ell,z) - \kappa(\ell',z))\crbk{\cT\Phi\theta_n^\mrm U(z) - \theta_n^\mrm U(\ell)} +\kappa(\ell',z) \crbk{\theta_n^\mrm U(\ell') - \theta_n^\mrm U(\ell)}}\\
&\leq 8\tmix \crbk{\frac{2L_{\kappa,\rho}}{1-\gamma} + [\theta_n^\mrm U]_{\rho}}\rho(\ell,\ell')
\end{aligned}$$
where the last inequality used the boundedness and Lipschitz estimate in Proposition \ref{prop:lipschitz_iterates}. Therefore, we obtain the Lipschitz estimates
\begin{equation}\label{eqn:U_BP_lip_bd}
    [B_{n+1}^\mrm U + P_{n+1}^\mrm U]_{\rho}\leq \frac{(2+16\tmix)L_{\kappa,\rho}}{1-\gamma} + 8\tmix [\theta_n^\mrm U]_{\rho}. 
\end{equation}

In addition, by Proposition \ref{prop:lipschitz_iterates} and \eqref{eqn:U_Poisson_sol}, \begin{equation}\label{eqn:U_BP_norm_bd}
\norm{B_{n+1}^\mrm U + P_{n+1}^\mrm U} \leq |Y_{n+1}^\mrm U| + \norm{\cT\Phi\theta_n^\mrm U} + 2\norm {v_{n}^\mrm U} \leq \frac{10\tmix}{1-\gamma},
\end{equation}
for all $n\geq 0$.

We now establish the error bound for the martingale terms. For any fixed $\epsilon > 0$, we apply Lemma~\ref{lemma:max_mart_net_general}, after a shift of the time index, to obtain 
\begin{equation}
\begin{aligned}
    \max_{m\in\cW_{\mrm U}(n,T)}
        \max_{\ell\in\mfk N_\rho(\epsilon)}
        \abs{
        \sum_{k=n}^{m-1}\alpha_{k+1}^{\mrm U}
        \sqbk{B_{k+1}^{\mrm U}(\ell)+P_{k+1}^{\mrm U}(\ell)}}\le
        \frac{10\tmix}{1-\gamma}
        \sqrt{2 T\alpha_{n+1}^\mrm U
        \log\!\crbk{
        \frac{2\cN_\rho(\epsilon)}{\delta}}}.
\end{aligned}
        \label{eqn:U_BP_joint_net_bound}
\end{equation}
with probability at least $1-\delta$. Here, we use Definition \ref{def:eps_net}, the squared-step size sum bound in \eqref{eqn:U_sum_alpha_square_bd}, and boundedness of the noises in \eqref{eqn:U_BP_norm_bd}. 

We next control interpolation error.  By \eqref{eqn:U_BP_lip_bd}, for every
$m\in\cW_{\mrm U}(n,T)$,
\begin{equation}
\begin{aligned}
        \sqbk{\sum_{k=n}^{m-1}\alpha_{k+1}^{\mrm U}
        \crbk{B_{k+1}^{\mrm U}+P_{k+1}^{\mrm U}}
        }_\rho &\leq\sum_{k=n}^{m-1}\alpha_{k+1}^\mrm U\crbk{\frac{(2+16\tmix)L_{\kappa,\rho}}{1-\gamma} + 8\tmix [\theta_k^\mrm U]_{\rho}} \\
        &\leq
        T\crbk{
        \frac{(2+16\tmix)L_{\kappa,\rho}}{1-\gamma}
        +8\tmix\crbk{L_{0,\rho}+\frac{2\alpha_{\mrm U}L_{\kappa,\rho}}{1-\gamma}
        \log\!\crbk{
        \frac{m-1+n_0^{\mrm U}}{n_0^{\mrm U} }}}}\\
        &\leq 
       \underbrace{8T\tmix L_{0,\rho}+\frac{18  T L_{\kappa,\rho}\tmix \alpha_{\mrm U} }{1-\gamma}
        \log\!\crbk{
        \frac{e(\max\cW_\mrm U(n,T)+n_0^{\mrm U})}{n_0^{\mrm U} }}}_{=:C(n,T)}
\end{aligned}
        \label{eqn:U_BP_joint_lipschitz_sum}
\end{equation}
where we used $\tmix,\alpha_\mrm U\ge1$, and applied Proposition~\ref{prop:lipschitz_iterates} to bound the Lipschitz constant.

Thus, combining \eqref{eqn:latent_net_general}, \eqref{eqn:U_BP_joint_net_bound}, and \eqref{eqn:U_BP_joint_lipschitz_sum}, we have for every $\epsilon > 0$, 
$$
        \max_{m\in\cW_{\mrm U}(n,T)}\norm{
        \sum_{k=n}^{m-1}\alpha_{k+1}^{\mrm U}
        \crbk{B_{k+1}^{\mrm U}+P_{k+1}^{\mrm U}}
        }
        \le \frac{10\tmix}{1-\gamma}
        \sqrt{2 T\alpha_{n+1}^\mrm U
        \log\!\crbk{
        \frac{2\cN_\rho(\epsilon)}{\delta}} 
        }+  \epsilon C(n,T)$$
with probability as least $1-\delta$. 

We choose $T = 1/4$. Then, by \eqref{eqn:fixed_window_three_halves} in Lemma \ref{lemma:step size_window}, $\max\cW_\mrm U(n,T) + n_0\leq \frac{3}{2}(n+n_0)$. So, 
$$
C(n,1/4)\leq 2\tmix L_{0,\rho}+\frac{5  L_{\kappa,\rho}\tmix \alpha_{\mrm U} }{1-\gamma}
        \log\!\crbk{
        \frac{3e (n+n_0^{\mrm U})}{2n_0^{\mrm U} }} \leq \frac{\tmix \alpha_\mrm U}{(1-\gamma) (n+n_0^\mrm U)\epsilon_n^\mrm U}.
$$
Taking $\epsilon=\epsilon_n^{\mrm U}$, combining this with the remaining Poisson initial bias term bounded in Lemma \ref{lemma:U_bias_error_bd}, and using the decomposition \eqref{eqn:U_within_window_decomp}, the within-window error satisfies, with probability at least $1-\delta$,
\begin{equation}\label{eqn:U_within_window_error}
    \max_{m\in\cW_\mrm U(n,1/4)}\norm{\sum_{k=n}^{m-1}\alpha_{k+1}^{\mrm U}
        \xi_{k+1}^{\mrm U}}\leq  \frac{10\tmix}{1-\gamma}
        \sqrt{\frac{\alpha_\mrm U}{2(n+n_0^\mrm U)}
        \log\!\crbk{
        \frac{2\cN_\rho(\epsilon_n^{\mrm U})}{\delta}}
        }+  \frac{11\tmix \alpha_\mrm U}{(1-\gamma)(n+n_0^\mrm U)}.
\end{equation}

\subsection{Error Propagation Through Epoch Iteration}
\label{subsec:U_epoch_iteration}

With the ODE tracking error bound in Lemma \ref{lemma:U_ode_tracking} and the window error term bound in \eqref{eqn:U_within_window_error} we can write down an error recursion across window epochs. Iterating through the recursion (cf. Lemma \ref{lemma:deterministic_epoch_iteration}) will yield Theorem \ref{thm:unnormalized_rate}. Specifically, we consider a high-probability event that the one-window estimate  \eqref{eqn:U_within_window_error} holds for all window epochs along the trajectory of the iterates. Then, we apply Lemma~\ref{lemma:deterministic_epoch_iteration} to obtain an error bound on the last iterate.  

We fix $\delta >0$ and set
$$
        m_0:=\ceil{16\alpha_{\mrm U}},
        \quad
        m_{r+1}:=\max\cW_{\mrm U}(m_r,1/4),
        \quad\text{and}\quad
        \delta_r:=\frac{6\delta}{\pi^2(r+1)^2}
$$
for each $r\ge 0$. Let $\Omega_{r,\mrm U}$ be the event on which
\eqref{eqn:U_within_window_error} holds with $n=m_r$ and confidence
level $\delta_r$, and define
$$
        \Omega_{\mrm U}(\delta)
        :=
        \bigcap_{r\ge0}\Omega_{r,\mrm U}.
$$
Since $\sum_{r\ge0}\delta_r=\delta$, union bound gives $P(\Omega_{\mrm U}(\delta))\ge1-\delta.$

Since $m_0 \ge16\alpha_{\mrm U}$, $\alpha_{n+1}^\mrm U\leq 1/16\leq T/2$ for all $n\ge m_0$. Thus, Lemma~\ref{lemma:step size_window}
implies $1/8\le\tau_{\mrm U}(m_r,m_{r+1})\le1/4.$ In particular, $\set{m_r:r\ge 0}$ must be strictly increasing, and hence $m_r\to\infty$. Moreover,
summing the lower bound over the first $r$ completed epochs gives
$$
        \tau_{\mrm U}(m_0,m_r) = \sum_{j=0}^{r-1}\tau_{\mrm U}(m_j,m_{j+1})
        \geq r/8
        .
$$
By definition of $\tau_{\mrm U}(m_0,m_r)$ and harmonic sum bound,
$$
\begin{aligned}
        r/8 \le \tau_{\mrm U}(m_0,m_r)=
        \alpha_{\mrm U}
        \sum_{k=m_0}^{m_r-1}
        \frac{1}{k+1+n_0^{\mrm U}}\leq
        \alpha_{\mrm U}\log(m_r+n_0^{\mrm U}).
\end{aligned}
$$ Since
$\alpha_{\mrm U}>1$, this further implies
\begin{equation}
r+1\le 8\alpha_{\mrm U}
        \log\!\crbk{e(m_r+n_0^{\mrm U})}.\label{eqn:tu_r+1_bd}
\end{equation}

Define $H_n^{\mrm U}
        :=
        \cL_n^{\mrm U}(\delta)
        +\log\cN_\rho(\epsilon_n^{\mrm U}).$
where $\cL_n^{\mrm U}(\delta)$ is defined in \eqref{eqn:U_log_term}. Since $\epsilon_n^\mrm U$ is non-increasing, $\set{H_n^{\mrm U}:n\ge 0}$ is non-negative and non-decreasing.  Hence, for $0\le j\le r$,
\begin{equation}
        \log\crbk{
        \frac{2\cN_\rho(\epsilon_{m_j}^{\mrm U})}{\delta_j}}
        \le
        \log\cN_\rho(\epsilon_{m_r}^{\mrm U})
        +
        \log\crbk{
        \frac{64\pi^2\alpha_{\mrm U}^2
        \log^2(e (m_r+n_0^{\mrm U}))}{3\delta}}  \le
        H_{m_r}^{\mrm U},
        \label{eqn:U_epoch_log_domination}
\end{equation}
where we used the definition and monotonicity of $\delta_j$ and \eqref{eqn:tu_r+1_bd}. 

Define and consider for $m_r \leq m\leq m_{r+1}$
$$
        \Delta_m^{\mrm U}
        :=
        \norm{\theta_m^{\mrm U}-\theta^*}\leq  \norm{\theta^*-\vartheta(\tau_\mrm U(m_r,m);\theta_{m_r}^\mrm U)}+\norm{\theta_m^{\mrm U}-\vartheta(\tau_\mrm U(m_r,m);\theta_{m_r}^\mrm U)} .
$$
So, combining Lemmas~\ref{lemma:U_ODE_exponential_stability} and \ref{lemma:U_ode_tracking} with $t = \tau_\mrm U(m_r,m)$ and $T=1/4$, for every
$m_r\le m\le m_{r+1}$,
$$
\Delta_m^{\mrm U}
\le
e^{-c_\wedge(1-\gamma)\tau_{\mrm U}(m_r,m)}
\Delta_{m_r}^{\mrm U}
+
e^{1/2}
\left[
\max_{j\in\cW_{\mrm U}(m_r,1/4)}
\norm{\sum_{k=m_r}^{j-1}\alpha_{k+1}^{\mrm U}\xi_{k+1}^{\mrm U}}
+
\frac{\alpha_{m_r+1}^{\mrm U}}{2(1-\gamma)}
\right].
$$
Therefore, on $\Omega_{r,\mrm U}$, \eqref{eqn:U_within_window_error} and
\eqref{eqn:U_epoch_log_domination} give

$$
\Delta_m^{\mrm U}
\le
e^{-c_\wedge(1-\gamma)\tau_{\mrm U}(m_r,m)}
\Delta_{m_r}^{\mrm U}
+
e^{1/2}
\left[
\frac{10\tmix}{1-\gamma}
\sqrt{\frac{\alpha_{\mrm U}H_{m_r}^{\mrm U}}{2(m_r+n_0^\mrm U)}}
+
\frac{11\tmix\alpha_{\mrm U}}{(1-\gamma)(m_r+n_0^\mrm U)}
+
\frac{\alpha_{m_r+1}^{\mrm U}}{2(1-\gamma)}
\right].
$$

Note that the square-root term has coefficient
$$
K_1^\mrm U
:=
\frac{10e^{1/2}}{\sqrt2}
\frac{\tmix\sqrt{\alpha_{\mrm U}}}{1-\gamma}.
$$
Also, for the lower-order terms, using $\tmix\ge1$,
$$
e^{1/2}\left[
\frac{11\tmix\alpha_{\mrm U}}{(1-\gamma)(m_r+n_0^\mrm U)}
+
\frac{\alpha_{m_r+1}^{\mrm U}}{2(1-\gamma)}
\right]
\le
\frac{23e^{1/2}}{2}
\frac{\tmix\alpha_{\mrm U}}{(1-\gamma)(m_r+n_0^\mrm U)}.
$$
Therefore, with
\[
        K_2^\mrm U:=
        \frac{23e^{1/2}}{2}
        \frac{\tmix\alpha_{\mrm U}}{1-\gamma}
\]
we have that
\begin{equation}
\begin{aligned}
        \Delta_m^{\mrm U}
        &\le
        e^{-c_\wedge(1-\gamma)
        \tau_{\mrm U}(m_r,m)}
        \Delta_{m_r}^{\mrm U}
        +
        K_1^\mrm U
        \sqrt{\frac{H_{m_r}^{\mrm U}}{m_r+n_0^{\mrm U}}}
        +
        \frac{K_2^\mrm U}{m_r+n_0^{\mrm U}}.
\end{aligned}
        \label{eqn:U_epoch_recursion}
\end{equation}
Moreover, on $\Omega_\mrm U(\delta)$, \eqref{eqn:U_epoch_recursion} holds for all $r\ge 0$ and any $m_{r}\leq m\leq m_{r+1}$. 

Finally, on $\Omega_\mrm U(\delta)$, which happen with probability at least $1-\delta$,  applying Lemma~\ref{lemma:deterministic_epoch_iteration} to the recursion \eqref{eqn:U_epoch_recursion} gives
$$
\begin{aligned}
\Delta_n^{\mrm U}
        &\le e^{1/16}\Delta_{m_0}^{\mrm U}
        \crbk{\frac{m_0+n_0^\mrm U}{n+n_0^\mrm U}}^{\lambda_{\mrm U}}
        +
        \frac{14\alpha_\mrm U K_1^\mrm U}{\lambda_{\mrm U}-1/2}
        \sqrt{\frac{H_n^\mrm U}{n+n_0^\mrm U}} \quad+
        \frac{17\alpha_\mrm U K_2^\mrm U}{\lambda_{\mrm U}-1}
        \frac{1}{n+n_0^\mrm U}
\end{aligned}
$$
simultaneously for every
integer $n\ge m_0$. Note that since $\lambda_\mrm U >1$
$$
        \frac{\alpha_\mrm U}{\lambda_{\mrm U}-1/2}< \frac{\alpha_\mrm U}{\lambda_\mrm U/2} = \frac{2}{c_\wedge(1-\gamma)}.$$
Moreover, 
$$
        e^{1/16}\Delta_{m_0}^{\mrm U}\le\frac{3}{1-\gamma},
        \quad
        14\,
        \frac{10e^{1/2}}{\sqrt2}<164,\quad 
        \text{and}\quad
        17\,
        \frac{23e^{1/2}}{2}<323.
$$
Also, by the step size condition in Assumption \ref{assump:stepsizes}, $m_0 = \ceil{16\alpha_\mrm U}\leq 16 n_0^\mrm U$. 
This completes the proof of Theorem \ref{thm:unnormalized_rate}.

\section{Convergence Rate of the Normalized Algorithm}
\label{sec:empirical_normalized_rate}

The normalized algorithm \eqref{eqn:empirical_normalized_update} uses the empirical denominator to remove the multiplier $c$ from the mean-field vector field. Once the empirical denominator is sufficiently accurate, the resulting mean-field ODE contracts at rate $(1-\gamma)$, rather than $c_\wedge(1-\gamma)$ as in the unnormalized case, thereby potentially improving the overall convergence rate. However, in the worst case, normalization can amplify the sampling noise in the SA recursion by a factor of $c_\wedge^{-1}$. It may also require a longer burn-in period before stable error decay begins, because additional samples are needed to stabilize the denominator estimate $\tilde c_n^{-1}$.

Define
\begin{equation}
        \cL_n^{\mrm N}(\delta)
        :=
        \log\!\crbk{
        \frac{50\pi^2\alpha_{\mrm 
        N}^2
        \log^2\!\crbk{e(n+n_0^{\mrm N})}}{\delta}} .
        \label{eqn:N_log_term}
\end{equation}

\begin{theorem}[Finite-time convergence of the normalized latent recursion]
\label{thm:empirical_normalized_rate}
Suppose Assumptions~\ref{assump:mdp}, \ref{assump:mix},
\ref{assump:stable_kernels}, and \ref{assump:stepsizes} hold. Moreover, choose
$\frac45\alpha_{\mrm N}(1-\gamma)=:\lambda_{\mrm N}>1$ and set
\begin{equation}
        \epsilon_n^{\mrm N}
        :=
        \sqbk{(n+n_0^{\mrm N})
        \crbk{2L_{0,\rho}+4L_{\kappa,\rho}(1+c_\wedge^{-1})}
        \log\!\crbk{\frac{5(n+n_0^{\mrm N})}{n_0^{\mrm N}}}}\inv.
        \label{eqn:N_mesh_def}
\end{equation}
For $\delta\in(0,1)$, define $r_c:=       \frac{c_\wedge}{64(1\vee L_{\kappa,\rho})}$ and set the burn-in number of iterations to be
\begin{equation}
        m_0(\delta):=\ceil{
        \frac{256\tmix^2}{c_\wedge^2}
        \log\!\crbk{\frac{6\cN_\rho(r_c)}{\delta}}} 
        \vee\ceil{n_0^{\mrm N}}
        \vee\ceil{16\alpha_{\mrm N}}.
        \label{eqn:N_burnin_def}
\end{equation}
Then, with probability at least $1-\delta$
\begin{equation}
\begin{aligned}
        \norm{\theta_n^{\mrm N}-\theta^*}
        &\le
        \frac{3}{1-\gamma}
        \sqbk{\frac{m_0(\delta)+n_0^{\mrm N}}
        {n+n_0^{\mrm N}}}^{\lambda_{\mrm N}} \\
        &\quad+
        \frac{370\tmix}{c_\wedge(1-\gamma)^2}
        \sqbk{
        \sqrt{\frac{
        \alpha_{\mrm N}\crbk{
        \cL_n^{\mrm N}(\delta)+\log\cN_\rho(\epsilon_n^{\mrm N})}}
        {n+n_0^{\mrm N}}}
        +
        \frac{\alpha_{\mrm N}^{2}}
        {c_\wedge(\lambda_{\mrm N}-1)(n+n_0^{\mrm N})}}
\end{aligned}
        \label{eqn:empirical_normalized_rate}
\end{equation}
simultaneously for all $ n\ge m_0(\delta)$.
\end{theorem}

The iteration number $m_0(\delta)$ is the burn-in steps needed to achieve uniform control of the empirical denominator over all later iterations. The proof of Theorem~\ref{thm:empirical_normalized_rate} is deferred to the appendix. It follows a similar strategy but is slightly more involved due to the burn-in run that estimates $c\inv$ and the dependence between $\tilde c_n$ and the one-sample Bellman target.

The three terms in \eqref{eqn:empirical_normalized_rate} have the same interpretation as in Theorem~\ref{thm:unnormalized_rate}.  The difference is that normalization removes $c_\wedge$ from the contraction exponent $\lambda_\mrm N$, leading to the possibility of using a smaller step size factor $\alpha_\mrm N$. Thus, the leading stochastic term now scales with $c_\wedge^{-1}$, instead of  $c_\wedge^{-3/2}$ for the unnormalized case.  

As in the unnormalized case, the normalized algorithm is also agnostic to the choice of $\rho$. Thus, with appropriate adjustments to the burn-in iteration number (as it depends on $\rho$ through $r_c$), for any fixed $n$ satisfying the assumptions of Theorem~\ref{thm:empirical_normalized_rate}, the leading term in the last-iterate error satisfies
\[
\norm{\theta_n^{\mrm N}-\theta^*}\lesssim
        \frac{\tmix}{c_\wedge (1-\gamma)^2}
        \sqrt{\frac{\alpha_\mrm N}{n+n_0^{\mrm N}}}\crbk{\sqrt{\cL_n^\mrm N(\delta)} + \inf_{\rho}\sqrt{ \log\cN_\rho(\epsilon_n^\mrm N)}}. 
\]
Thus, the algorithm automatically adapts to the structure of $\mbb L$ through the use of the best choice of metric $\rho$.

A direct application of Theorem~\ref{thm:empirical_normalized_rate} gives the following corollary.

\begin{corollary}[Finite and Euclidean latent spaces for the normalized recursion]
\label{cor:normalized_covering_specializations}
Assume the setting in Theorem~\ref{thm:empirical_normalized_rate}.  In each of the following cases,
the displayed bound holds with probability at least $1-\delta$ simultaneously for all $n\ge m_0(\delta)$.
\begin{enumerate}[label=(\roman*)]
\item If $\mbb L$ is finite with cardinality $J$ and is equipped with the discrete metric, then
\[
\begin{aligned}
        \norm{\theta_n^{\mrm N}-\theta^*}
        \le
        \frac{3}{1-\gamma}
        \sqbk{\frac{m_0(\delta)+n_0^{\mrm N}}
        {n+n_0^{\mrm N}}}^{\lambda_{\mrm N}}+
        \frac{370\tmix}{c_\wedge(1-\gamma)^2}
        \sqbk{
        \sqrt{\frac{\alpha_{\mrm N}
        \crbk{\cL_n^{\mrm N}(\delta)+\log J}}
        {n+n_0^{\mrm N}}}
        +
        \frac{\alpha_{\mrm N}^{2}}
        {c_\wedge(\lambda_{\mrm N}-1)(n+n_0^{\mrm N})}}.
\end{aligned}
\]

\item Suppose there are constants $\cD\ge0$ and $d_{\mbb L}\ge1$ such that, for every $\epsilon>0$,
\[
        \cN_\rho(\epsilon)
        \le
        \crbk{1+\frac{\cD}{\epsilon}}^{d_{\mbb L}}.
\]
Then
\[
\begin{aligned}
        \norm{\theta_n^{\mrm N}-\theta^*}
        &\le
        \frac{3}{1-\gamma}
        \sqbk{\frac{m_0(\delta)+n_0^{\mrm N}}
        {n+n_0^{\mrm N}}}^{\lambda_{\mrm N}} \\
        &\quad+
        \frac{370\tmix}{c_\wedge(1-\gamma)^2}
        \sqbk{
        \sqrt{
        \frac{
        \alpha_{\mrm N}\crbk{
        \cL_n^{\mrm N}(\delta)
        +d_{\mbb L}\log\!\crbk{1+\cD/\epsilon_n^{\mrm N}}}}
        {n+n_0^{\mrm N}}}
        +
        \frac{\alpha_{\mrm N}^{2}}
        {c_\wedge(\lambda_{\mrm N}-1)(n+n_0^{\mrm N})}}.
\end{aligned}
\]
\end{enumerate}
\end{corollary}

\section{Applications and Architectures}
\label{sec:concrete_architectures}

In this section, we show that the infinite-dimensional coefficient-field formulation provides a useful abstraction. Although computation with truly infinite-dimensional objects ultimately requires approximation, this formulation makes explicit the role of Lipschitz regularity of the coefficient field in the embedding space, which allows us to retain statistical tractability even when the state space is very large. 

We give two concrete architectures for which Assumption~\ref{assump:stable_kernels} can be verified from primitive conditions.  Consequently, the unnormalized and normalized algorithms, together with Theorems~\ref{thm:unnormalized_rate} and \ref{thm:empirical_normalized_rate}, apply directly. The first example is a linear-density approximation version of the Q-measure-learning algorithm proposed by \citet{wang2026qml}.  The second example is a pretrained neural decoder whose pre-output units are indexed by a compact parameter space. 

In both examples, the coefficient field is defined on a large, potentially infinite, latent space $\mbb L$, where the infinite-dimensional abstraction serves as a lens for understanding statistical efficiency rather than as a directly implementable algorithm. In particular, both formulations allow $\mbb L$ to be chosen as a finite space, in which case the algorithms can be implemented exactly. Rather than using the discrete metric, one can equip $\mbb L$ with a more natural distance $\rho$, such as graph distance when $\mbb L$ is naturally embedded in a graph, to induce desirable smoothness of the coefficient field $\theta$.

\subsection{Q-Measure-Learning with Linear Density Approximation}
\label{subsec:linear_density_qml}

We first introduce the Q-measure-learning of \citet{wang2026qml} and relate it to the stable reconstruction-and-compression framework in this paper. 

Recall that $\mu_b$ is the stationary distribution of the behavior chain, and let $\kappa:\mbb Z\times\mbb Z\to[0,1]$ be a smoothing kernel used by Q-measure-learning; e.g. the Gaussian kernel $\kappa(y,z) = \exp[-|y-z|^2/(2\sigma^2)]$. For a finite signed measure $\nu$ on $\mbb Z$, define the normalized reconstruction map
\[
        \Psi[\nu](z) := \frac{\int_{\mbb Z}\kappa(z,u)\nu(du)}{\int_{\mbb Z}\kappa(z,u)\mu_b(du)}, \quad z\in\mbb Z.
\]

The goal of the Q-measure-learning is to approximate a signed measure $\nu^*$ so that $q^*:=\Psi[\nu^*]\in C(\Z)$ approximates the optimal $Q^*$. This is achieved by considering a smoothed version of the Bellman operator. Specifically, for $q\in C(\mbb Z)$, define the normalized smoothing kernel 
\begin{equation}
\begin{aligned}
        \mrm Kq(z) &:=\frac{\int_{\mbb Z}\kappa(z,u)q(u)\mu_b(du)}{\int_{\mbb Z}\kappa(z,u)\mu_b(du)}.
\end{aligned}
        \label{eqn:qml_population_smoother}
\end{equation}
Thus, $\mrm K$ is the stationary-normalized smoothing kernel that serves as the compression kernel in this paper. In particular, since $\mrm K$ is nonexpansive, $\mrm K\cT$ is a $\gamma$-contraction on $C(\Z)$. It follows that it has a unique fixed point $q^*$; i.e. $q^* = \mrm K\cT q^*$. The corresponding Q-measure is then defined by
\[\nu^*(du) := \cT q^*(u)\mu_b(du),\]
and it is easy to see that $q^*$ can be reconstructed from $\nu^*$ via the map $\Psi$ as $q^*=\Psi[\nu^*]$.

Hence, at the population-operator level, Q-measure-learning has the same form as the stable approximation framework with $\mbb L = \Z$, $\Phi = I$ the identity operator, and compression kernel $\mrm K$. Therefore, the contraction on the coefficient field space is $\cG=\mrm K\cT.$

We note that this identification holds only at the operator level. The original Q-measure-learning algorithm updates an empirical Q-measure supported on the behavior trajectory while simultaneously estimating the reference measure $\mu_b$, whereas the algorithms in this paper update a coefficient field directly. However, this connection suggests a linear-density approximation of Q-measure-learning, obtained by replacing the empirical Q-measure with a density function with respect to a fixed, user-specified reference measure. 

Concretely, let $(\mbb L,\rho)$ be a compact metric space, and let $\lambda$ be a probability measure on $\mbb L$. For a signed density function $\theta\in C(\mbb L)$, define
\[
        \nu_\theta(d\ell) := \theta(\ell)\lambda(d\ell).
\]
Given a density $\theta$, a Q-function is reconstructed using a user-specified feature kernel $\phi:\mbb Z\times\mbb L\to\R_+$. As in \eqref{eqn:Phi_reconstruction}, define
\begin{equation}
\begin{aligned}
        q_\theta(z) &:= \Phi\theta(z)= \int_{\mbb L}\phi(z,\ell)\theta(\ell)\lambda(d\ell) = \int_{\mbb L}\phi(z,\ell)\nu_\theta(d\ell).
\end{aligned}
        \label{eqn:density_qml_reconstruction}
\end{equation}
So, $\phi$ plays a similar role as $\Psi$ in the Q-measure-learning setup. 

To compress a candidate $q$-function back into density coordinates, we retain the stationary-normalized kernel construction of Q-measure-learning. For $\kappa:\mbb L\times\mbb Z\to[0,1]$, define

\[        \mrm Kq(\ell) := \frac{\int_{\mbb Z}\kappa(\ell,z)q(z)\mu_b(dz)}{c(\ell)},\quad c(\ell) := \int_{\mbb Z}\kappa(\ell,z)\mu_b(dz).
\]
When $\mbb L=\mbb Z$, this is the same population smoothing mechanism as $\mrm K$ in \eqref{eqn:qml_population_smoother}. The distinction is that the output is now indexed by a possibly different latent space $\mbb L$.

Motivated by Q-measure-learning and the framework developed in this paper, instead of learning the optimal Q-measure $\nu^*$, we learn an optimal density $\theta^*$ such that the reconstructed function $q_{\theta^*}$ approximates $Q^*$. The resulting Bellman operator for the coefficient field is therefore $\cG=\mrm K\cT\Phi.$

We note that, unlike in Q-measure-learning, where $\Psi$ and $\mrm K$ share the same kernel function, the linear approximation framework allows different choices of $\phi$ and $\kappa$ for $\Phi$ and $\mrm K$. The former determines how a latent density is decoded into a Q-function, while the latter determines how a Bellman target is compressed into the latent coordinates.

Applied to this density field, the unnormalized and normalized updates in Section~\ref{sec:two_algorithms} become
\begin{equation}
\begin{aligned}
        \theta_{n+1}^{\mrm U}(\ell) &= \theta_n^{\mrm U}(\ell) + \alpha_{n+1}^{\mrm U}\kappa(\ell,Z_n)\crbk{Y_{n+1}^{\mrm U}-\theta_n^{\mrm U}(\ell)},\\
        \theta_{n+1}^{\mrm N}(\ell) &= \theta_n^{\mrm N}(\ell) + \alpha_{n+1}^{\mrm N}\frac{\kappa(\ell,Z_n)}{\tilde c_n(\ell)}\crbk{Y_{n+1}^{\mrm N}-\theta_n^{\mrm N}(\ell)}.
\end{aligned}
        \label{eqn:density_qml_algorithms}
\end{equation}
The corresponding learned measures are thus $\nu_n^{\mrm M}(d\ell) := \theta_n^{\mrm M}(\ell)\lambda(d\ell)$ where $\mrm M\in\{\mrm U,\mrm N\}.$ Suppose Assumptions~\ref{assump:mdp}, \ref{assump:mix}, \ref{assump:stable_kernels}, and \ref{assump:stepsizes} hold. Then, the conclusions in Proposition \ref{prop:latent_fixed_point} hold. Moreover, Theorem~\ref{thm:unnormalized_rate} applies to the first recursion in \eqref{eqn:density_qml_algorithms}, and Theorem~\ref{thm:empirical_normalized_rate} applies to the second.

We next provide one concrete choice of the compression kernel $\mrm K$, following the Gaussian-type smoothing construction in \citet{wang2026qml}. We still leave $\phi$ for the user to specify under Assumption \ref{assump:stable_kernels}. Concretely, suppose that $\mbb L\subset\R^{d_{\mbb L}}$ is compact, equipped with the Euclidean metric $\abs{\cd}$, and let $e:\mbb Z\to\mbb L$ be a continuous embedding map. For $\sigma>0$, define the Gaussian kernel function
\begin{equation}
        \kappa_\sigma(\ell,z) := \exp\crbk{-\frac{\abs{\ell-e(z)}^2}{2\sigma^2}}.
        \label{eqn:density_qml_kappa}
\end{equation}
In addition, we assume that there exists $r_0>0$ and $p_0>0$ such that
\begin{equation}
        \inf_{\ell\in\mbb L}\mu_b\crbk{\setcond{z\in\mbb Z}{\abs{e(z)-\ell}\le r_0}} \ge p_0.
        \label{eqn:density_qml_small_ball}
\end{equation}
This is a uniform local-coverage condition: every latent coordinate receives positive stationary behavior mass in a neighborhood of its kernel center.

\begin{lemma}\label{lemma:density_qml_gaussian_compression}
Let $\phi$ be any feature kernel satisfying Assumption~\ref{assump:stable_kernels} (i) and (iv). Suppose $\mbb L\subset\R^{d_{\mbb L}}$ is compact and equipped with the Euclidean distance $\rho(\cd) = \abs{\cd}$, $e:\mbb Z\to\mbb L$ is continuous, and \eqref{eqn:density_qml_small_ball} holds. Then the pair $(\phi,\kappa_\sigma)$ satisfies Assumption~\ref{assump:stable_kernels}. In particular, one may take $c_\wedge = p_0e^{-r_0^2/(2\sigma^2)}$ and $ L_{\kappa,\rho} = 1/(\sigma\sqrt e)$. 
\end{lemma}

Before presenting the proof, we note that Lemma~\ref{lemma:density_qml_gaussian_compression} implies that the convergence bounds in Theorems~\ref{thm:unnormalized_rate} and \ref{thm:empirical_normalized_rate} apply. However, the choice of $c_\wedge$ in Lemma~\ref{lemma:density_qml_gaussian_compression}, while perhaps unimprovable in the worst case, is typically conservative. In particular, using a tighter lower bound $c_\wedge$ for $c(\ell)$ can improve the convergence rate or burn-in time.

We highlight that the issue of choosing $c_\wedge$ reveals an advantage of the normalized algorithm over the unnormalized version. To obtain the canonical $n^{-1/2}$ rate, the unnormalized algorithm requires $\alpha_\mrm U>1/[2c_\wedge(1-\gamma)]$, and a conservative choice of $c_\wedge$ can lead to slower convergence by inducing a large variance. In contrast, choosing $\alpha_\mrm N$ does not require a lower bound on $c_\wedge$; $c_\wedge$ enters the normalized algorithm only through the clipping of $\tilde c_n(\ell)$. Thus, a conservative choice of $c_\wedge$ will affect the burn-in time rather than the leading-order convergence rate of the normalized algorithm.

\begin{proof}
It remains to verify Assumption~\ref{assump:stable_kernels} (ii) and (iii). 

Note that if $\abs{e(z)-\ell}\le r_0$, then $\kappa_\sigma(\ell,z) \ge e^{-r_0^2/(2\sigma^2)}.$
Therefore,
\[
        c(\ell) = \int_{\mbb Z}\kappa_\sigma(\ell,z)\mu_b(dz)\geq \mu_b\crbk{\setcond{z\in\mbb Z}{\abs{e(z)-\ell}\le r_0}} e^{-r_0^2/(2\sigma^2)} \ge c_\wedge.
\]
Moreover, for fixed $z\in\Z$ the $\ell$-gradient norm of $\kappa_\sigma$ is
\[
        \abs{\nabla_\ell\kappa_\sigma(\ell,z)} = \frac{\abs{\ell-e(z)}}{\sigma^2}\exp\crbk{-\frac{\abs{\ell-e(z)}^2}{2\sigma^2}}.
\]
Let $x = \abs{\ell-e(z)}/\sigma$. Then, $\abs{\nabla_\ell\kappa_\sigma(\ell,z)} = \sigma\inv x e^{-x^2/2}$. Consider $f(x):=xe^{-x^2/2}$; elementary calculations show $f$ is maximized at $x = 1$. Thus, $\abs{\nabla_\ell\kappa_\sigma(\ell,z)} = f(x)/\sigma\leq f(1)/\sigma = 1/(\sigma\sqrt e)$. Hence $\kappa_\sigma$ is uniformly $1/(\sigma\sqrt e)$-Lipschitz in $\ell$. 
\end{proof}

\subsection{Pretrained Wide Network and Induced Geometry}
\label{subsec:pretrained_decoder}

We next consider a pretrained neural network setting whose layers before the output layer are frozen. We consider the coefficient field implementation proposed in the paper to tune the output weights and biases. The coefficient-field formulation clarifies that the pretrained architecture can provide a natural geometry on its pre-output layer, under which nearby neurons may have similar activation patterns and therefore receive similar coefficient updates. Treating the output coefficients as a field on this metric space makes this regularity explicit and allows the finite-time bounds to depend on the metric entropy of the representation rather than only on its raw width.

Let \(\mbb L=\{0,1,\ldots,J-1\}\) index the \(J\) neurons in the pre-output layer, where \(J\) is potentially large. We will consider a metric \(\rho\) that reflects the geometry of the pretrained embedding. Let \(a:\mbb Z\times\mbb L\to\R_+\) denote the activation intensity of neuron \(\ell\) when the frozen network is evaluated at \(z\). We reserve coordinate \(0\) for a nonzero output-layer bias by setting \(a(\cdot,0)\equiv1\).

Assume that \(a(\cdot,\ell)\in C(\mbb Z)\) for every \(\ell\in\mbb L\), and let  $A:=\max_{z\in\mbb Z,\ \ell\in\mbb L}a(z,\ell).$
We take \(\lambda\) to be the uniform distribution on \(\mbb L\) and define
\begin{equation}
        \phi(z,\ell):=\frac{a(z,\ell)}{A},
        \qquad
        \kappa(\ell,z):=\frac{a(z,\ell)}{A}.
        \label{eqn:neural_phi_kappa}
\end{equation}
The resulting reconstruction kernel is
\begin{equation}
        \Phi\theta(z)=\frac1J\sum_{\ell\in\mbb L}\phi(z,\ell)\theta(\ell)=\frac1{AJ}\sum_{\ell\in\mbb L}a(z,\ell)\theta(\ell).
        \label{eqn:finite_neural_decoder}
\end{equation}
The factor \(1/(AJ)\) is fixed and can be absorbed into the output-layer weights. Thus, \eqref{eqn:finite_neural_decoder} represents an output layer that is affine in the pre-output layer's activation levels.

The two algorithms in this paper then take the form
\begin{equation}
\begin{aligned}
        \theta_{n+1}^{\mrm U}(\ell)
        &=\theta_n^{\mrm U}(\ell)+\alpha_{n+1}^{\mrm U}\frac{a(Z_n,\ell)}{A}\crbk{Y_{n+1}^{\mrm U}-\theta_n^{\mrm U}(\ell)},\\
        \theta_{n+1}^{\mrm N}(\ell)
        &=\theta_n^{\mrm N}(\ell)+\alpha_{n+1}^{\mrm N}\frac{a(Z_n,\ell)}{A\tilde c_n(\ell)}\crbk{Y_{n+1}^{\mrm N}-\theta_n^{\mrm N}(\ell)},
\end{aligned}
        \label{eqn:neural_decoder_algorithms}
\end{equation}
where
\[
        \hat c_n(\ell)=\frac1{n+1}\sum_{t=0}^n\frac{a(Z_t,\ell)}{A},
        \qquad
        \tilde c_n(\ell)=\hat c_n(\ell)\vee c_\wedge.
\]
Both recursions are exactly implementable by maintaining the \(J\) output coefficients.

We note that these recursions admit an SGD-type interpretation as first-order methods for a frozen-target quadratic regression problem. This is similar to optimizers used in deep learning-based RL algorithms. Fix \(\theta\in\Theta\), and recall the one-sample Bellman target
\[
        Y_\theta(z):=F(z,X')+\gamma\max_{a\in\mbb A}\Phi\theta(X',a)
\]We consider the normalized case. Define
\[
\begin{aligned}
        \cR_{\mrm N}(\vartheta;\theta)&:=\frac12E_{F,Z,X'}\left[\int_{\mbb L}\frac{\kappa(\ell,Z)}{c(\ell)}\crbk{Y_\theta(Z)-\vartheta(\ell)}^2\lambda(d\ell)\right]\\
        &=\frac12E_{F,U,X',L}\left[\crbk{Y_\theta(U)-\vartheta(L)}^2\right]
\end{aligned}
\]
where \(F\eqd F_1\), \(Z\sim\mu_b\), $L\sim \lambda$, and \(X'| Z\sim P(\cdot| Z)\) under $E_{F,Z,X'}$, while $U|L\sim \mrm K(L,\cd)$ and \(X'| U\sim P(\cdot| U)\) under $E_{F,U,X',L}$. This loss is precisely the stationary $L^2$ Bellman error, where the state-action pair $U$ is transferred from averaged latent coordinate $L$ by the compression kernel. For every \(h\in L^2(\lambda)\), the derivative in the first argument satisfies
\[
\begin{aligned}
        \nabla_1\cR_{\mrm N}(\vartheta;\theta)[h]
        &=E_{F,Z,X'}\left[\int_{\mbb L}\frac{\kappa(\ell,Z)}{c(\ell)}\sqbk{\vartheta(\ell)-Y_\theta(Z)}h(\ell)\lambda(d\ell)\right]\\
        &=\int_{\mbb L}\sqbk{\vartheta(\ell)-\cG\theta(\ell)}h(\ell)\lambda(d\ell).
\end{aligned}
\]
Thus, $-\nabla_1\cR_{\mrm N}(\theta;\theta)=\cG\theta-\theta$ where the gradient is taken with respect to $\vartheta$ and evaluated at $\vartheta = \theta$. In particular, this gradient is the stationary population version of the stochastic update in \eqref{eqn:neural_decoder_algorithms}. Therefore, the normalized algorithm can be viewed as taking a gradient step on this $L^2$ Bellman-regression loss after freezing the current coefficient field inside the Bellman target. The same interpretation applies to the unnormalized algorithm.

To apply the framework in this paper, we note that if $a$ also satisfies $\abs{a(z,\ell)-a(z,\ell')}\le AL_a\rho(\ell,\ell')$ and $\min_{\ell\in\mbb L}\int_\Z a(z,\ell)\mu_b(dz)\ge Ac_\wedge>0$ for all $z\in\mbb Z,\ \ell,\ell'\in\mbb L$, then Assumption~\ref{assump:stable_kernels} holds. The first condition holds trivially when $\rho$ is the discrete metric because $\Z$ is bounded and $\mbb L$ is finite, while the latter holds when a positive activation function is used. Thus, the main theorems in this paper are applicable.  

Since the algorithms can automatically adapt to the smoothness and geometry, we seek a sharper statistical understanding by considering more geometry-aware choices of $\rho$. Specifically, a smaller statistical complexity can be obtained when the pre-output layer generates a smaller number of functionally distinct response patterns. To capture this behavior, we consider the space of normalized activation patterns
$\mbb W:=\set{\phi(\cdot,\ell):\ell\in\mbb L}\subset C(\mbb Z).$
After identifying $\mbb L$ with $\W$, we equip \(\mbb L\) with the (pseudo-)metric
\[\rho_{\star}(\ell,\ell'):=\sup_{z\in\mbb Z}\abs{\phi(z,\ell)-\phi(z,\ell')}=\norm{\phi(\cdot,\ell)-\phi(\cdot,\ell')}.\]
Under $\rho_\star$, it is easy to see that $\abs{\kappa(\ell,z)-\kappa(\ell',z)}\le\rho_{\star}(\ell,\ell')$. Thus, the main theorems apply with \(L_a=L_{\kappa,\rho_\star}=1\). Moreover, the relevant covering number is $\cN_{\rho_{\star}}(\epsilon)=\cN\crbk{\mbb W,\norm{\cdot},\epsilon}.$

This metric gives a direct interpretation of the covering number. At resolution \(\epsilon\), \(\cN_{\rho_\star}(\epsilon)\) is the number of distinct neuron patterns needed to represent the pre-output layer up to uniform error \(\epsilon\). Thus, its statistical complexity is determined not by the raw width \(J\), but by the number of functionally distinct neuron clusters. In particular, pretraining may aggregate the high-dimensional state-action input into a small collection of recurring patterns, while a wide pre-output layer contains multiple nearby variants of each pattern. Clustering nearby neural patterns yields a statistical advantage when \(\cN_{\rho_\star}(\epsilon)\ll J\) at the resolution relevant to learning.

A more flexible interpretation does not require the representative patterns to be realized by actual neurons. Define the number of ideal representative patterns
\[\cP(\epsilon;\mbb W):=\min\left\{k:\exists\,g_1,\ldots,g_k\in C(\mbb Z)\text{ s.t. }\max_{\ell\in\mbb L}\min_{1\le j\le k}\norm{\phi(\cdot,\ell)-g_j}\le\epsilon\right\}.\]
Unlike the $\epsilon$-nets that induce $\cN_{\rho_\star}(\epsilon)$, the centers \(g_j\) need not correspond to individual neurons. They can be interpreted as representative activation patterns.

To connect back to the covering number, we note that $\cP(\epsilon;\mbb W)$ is quantitatively close to $\cN_{\rho_\star}(\epsilon)$. Indeed, by definition, $\cP(\epsilon;\mbb W)\le\cN_{\rho_\star}(\epsilon).$ Conversely, from each nonempty representative pattern ball centered at $g_j$, choose one activation profile \(\phi(\cdot,\ell_{g_j})\) contained in that ball as the center of a ball of radius $2\epsilon$. For any $\phi(\cdot,\ell)$, $\ell\in\mbb L$, that belongs to the same ball,
$$\rho_\star(\ell,\ell_{g_j})
        \le\norm{\phi(\cdot,\ell)-g_j}+\norm{g_j-\phi(\cdot,\ell_{g_j})}\le2\epsilon;$$
i.e., $\phi(\cd,\ell)$ is within the ball of radius $2\epsilon$ centered at \(\phi(\cdot,\ell_{g_j})\). Thus, $\cN_{\rho_\star}(2\epsilon)\le\cP(\epsilon;\mbb W).$

\bibliographystyle{apalike}
\bibliography{references}

\newpage 

\appendixpage
\appendix

\section{Proof of Theorem~\ref{thm:empirical_normalized_rate}}
\label{sec:proof_normalized_rate}

The proof follows the same ODE-tracking argument as the proof of
Theorem~\ref{thm:unnormalized_rate}.  The additional step is to control the empirical denominator uniformly
over time.  On a good high-probability event in which the empirical denominator is close to the population denominator $c$, the normalized mean field is uniformly contractive, while the stochastic terms can be handled as before.

\subsection{Uniform Control of the Empirical Denominator}

We first establish a high-probability event on which the denominator is estimated within a safety margin uniformly for all iterations $k\geq n$ after some burn-in period. 

\begin{lemma}[Burn-in sample size for the empirical denominator]\label{lemma:N_denominator_bd} Recall the definition of $r_c:=       \frac{c_\wedge}{64(1\vee L_{\kappa,\rho})}$ in Theorem \ref{thm:empirical_normalized_rate}. If
\begin{equation}
        n+1
        \ge
        \frac{256\tmix^2}{c_\wedge^2}
        \log\!\crbk{\frac{2\cN_\rho(r_c)}{\delta}},
        \label{eqn:c_uniform_concentration_condition}
\end{equation}
then, with probability at least $1-\delta$,
\[
        \sup_{k\ge n}\norm{\hat c_k-c}
        \le
        \frac{c_\wedge}{4}.
\]
\end{lemma}
\begin{proof}
Let $\mfk N_\rho(r_c)$ be a deterministic $r_c$-net of $\mbb L$ with
cardinality $\cN_\rho(r_c)$.  For each
$\ell\in\mfk N_\rho(r_c)$, define
$g(\ell,z):=\kappa(\ell,z)-c(\ell).$
By the definition of $c(\ell)$, $\int_{\mbb Z}g(\ell,z)\mu_b(dz)=0$; i.e. $g(\ell,\cd)$ is centered. 
Moreover, since both $\kappa(\ell,z)$ and $c(\ell)$ belong to $[0,1]$,
we have $\norm{g(\ell,\cdot)}\le1$.

Let $u(\ell,\cdot)$ be the unique bounded centered solution of the
Poisson equation
\[
\begin{aligned}
        u(\ell,z)
        -
        \int_{\mbb Z}u(\ell,y)P_b(dy\mid z)
        &=
        g(\ell,z),
        \quad \forall z\in\mbb Z.
\end{aligned}
\]
Then, Lemma~\ref{lemma:poisson_general} implies
$\norm{u(\ell,\cdot)}
        \le
        2\tmix\norm{g(\ell,\cdot)}
        \le
        2\tmix.$
Evaluating the Poisson equation at $Z_t$ and using the conditional
Markov property,
\[
\begin{aligned}
        g(\ell,Z_t)
        &=
        u(\ell,Z_t)
        -
        E[u(\ell,Z_{t+1})\mid\cF_t]\\
        &=
        u(\ell,Z_t)-u(\ell,Z_{t+1})
        +
        D_{t+1}(\ell),
\end{aligned}
\]
where
\[
        D_{t+1}(\ell)
        :=
        u(\ell,Z_{t+1})
        -
        E[u(\ell,Z_{t+1})\mid\cF_t]
\]
is a martingale difference.  Summing from $t=0$ to $k-1$ gives
\begin{equation}
        \sum_{t=0}^{k-1}g(\ell,Z_t)
        =
        u(\ell,Z_0)-u(\ell,Z_k)
        +
        \sum_{t=0}^{k-1}D_{t+1}(\ell).
        \label{eqn:c_poisson_decomposition}
\end{equation}

To obtain a bound uniformly over all $k\ge n+1$, we use the same exponential-supermartingale argument in the proof of Lemma~\ref{lemma:max_mart_net_general}. Here, the difference is that there is no fixed horizon, but the argument is essentially the same. Thus, to avoid repetition, we provide a shortened proof.

First, we note that range of $D_{t+1}(\ell)$ has length at
most $2\norm{u(\ell,\cdot)} \le 4\tmix.$
Thus, as in \eqref{eqn:hoeffding_exp_supermg}, Hoeffding's lemma
implies that, for every $\lambda\in\R$,
\[
        G_k:=\exp\crbk{
        \lambda\sum_{t=0}^{k-1}D_{t+1}(\ell)
        -
        2\lambda^2\tmix^2k},
        \qquad k\ge0,
\]
is a nonnegative supermartingale.  

We choose
\[
\lambda:=\frac{3c_\wedge}{64\tmix^2}\quad \text{and hence}\quad\frac{3\lambda c_\wedge}{16}>2\lambda^2\tmix^2 = \frac{3\lambda c_\wedge}{32}.
\]
Then, 
\[
\begin{aligned}
&P\!\left(
        \sup_{k\ge n+1}
        \frac1k
        \sum_{t=0}^{k-1}D_{t+1}(\ell)> \frac{3c_\wedge}{16}
        \right)\\
        &\quad \leq P\!\left( \sup_{k\ge n+1}\set{\lambda\sum_{t=0}^{k-1}D_{t+1}(\ell)-   2\lambda^2\tmix^2k}> \set{\frac{3\lambda c_\wedge}{16}-2\lambda^2\tmix^2}(n+1)
        \right)\\
        &\quad \leq P\!\left( \sup_{k\ge n+1} G_k> \exp\set{\frac{3\lambda c_\wedge(n+1)}{32}}
        \right)\\
        &\quad \le
        \exp\!\left\{-\frac{3\lambda c_\wedge}{32}
        (n+1)
        \right\}.
\end{aligned}
\]
where the last step applies Ville's inequality \citep[Lemma 1]{howard2020timeuniform}. Applying the same argument to $-D_{t+1}(\ell)$ and a union bound yields
\[
 P\!\left(
        \sup_{k\ge n+1}\frac1k
        \abs{\sum_{t=0}^{k-1}D_{t+1}(\ell)}
        >
        \frac{3c_\wedge}{16}
        \right)
 \le
        2\exp\!\left\{
        -\frac{9c_\wedge^2(n+1)}{2048\tmix^2}
        \right\}.
\]

Note that by the choice \eqref{eqn:c_uniform_concentration_condition},
\[
        2\exp\!\left\{
        -\frac{9c_\wedge^2(n+1)}{2048\tmix^2}
        \right\} \le2\exp\!\left\{
        -\frac98
        \log\!\crbk{
        \frac{2\cN_\rho(r_c)}{\delta}}
        \right\}=
        2\crbk{
        \frac{2\cN_\rho(r_c)}{\delta}}^{-9/8}\le
        \frac{\delta}{\cN_\rho(r_c)}.
\]
Therefore, a union bound over the net $\mfk N_\rho(r_c)$ shows that, with probability at
least $1-\delta$,
\begin{equation}
        \sup_{k\ge n+1}
        \max_{\ell\in\mfk N_\rho(r_c)}
        \frac1k
        \abs{\sum_{t=0}^{k-1}D_{t+1}(\ell)}
        \le
        \frac{3c_\wedge}{16}.
        \label{eqn:c_martingale_uniform_bound}
\end{equation}

We then control the boundary term in
\eqref{eqn:c_poisson_decomposition}.  Because $\cN_\rho(r_c)\ge1$, $\delta<1$, $\tmix\ge1$, and
$c_\wedge\le1$, condition
\eqref{eqn:c_uniform_concentration_condition} also implies
\[
        n+1\ge
        \frac{256\tmix^2}{c_\wedge^2}\log2\ge
        \frac{128\tmix}{c_\wedge}.
\]
Since
$\norm{u(\ell,\cdot)}\le2\tmix$, for every $k\ge n+1$,
\[
        \frac1k
        \abs{u(\ell,Z_0)-u(\ell,Z_k)}
        \le \frac{4\tmix}{k}\leq 
        \frac{c_\wedge}{32}.
\]
Combining this estimate with
\eqref{eqn:c_poisson_decomposition} and
\eqref{eqn:c_martingale_uniform_bound}, we obtain that with probability at
least $1-\delta$,
\[
        \sup_{k\ge n+1}
        \max_{\ell\in\mfk N_\rho(r_c)}
        \abs{\hat c_{k-1}(\ell)-c(\ell)}
        \le
        \frac{3c_\wedge}{16}
        +
        \frac{c_\wedge}{32}
        =
        \frac{7c_\wedge}{32}.
\]

It remains to extend the estimate from the net to all of $\mbb L$.
Both $\hat c_{k-1}$ and $c$ are $L_{\kappa,\rho}$-Lipschitz, so
$\hat c_{k-1}-c$ is $2L_{\kappa,\rho}$-Lipschitz.  For any
$\ell\in\mbb L$, choose
$\ell'\in\mfk N_\rho(r_c)$ such that
$\rho(\ell,\ell')\le r_c$.  Then, w.h.p.,
\[
\begin{aligned}
        \abs{\hat c_{k-1}(\ell)-c(\ell)}
        &\le
        \abs{\hat c_{k-1}(\ell')-c(\ell')}
        +
        2L_{\kappa,\rho} r_c\\
        &\le
        \frac{7c_\wedge}{32}
        +
        \frac{2L_{\kappa,\rho} c_\wedge}
        {64(1\vee L_{\kappa,\rho})}\\
        &\le
        \frac{7c_\wedge}{32}
        +
        \frac{c_\wedge}{32}
        =
        \frac{c_\wedge}{4}.
\end{aligned}
\]
Taking the supremum over $\ell\in\mbb L$ and $k\ge n+1$ completes the proof. 
\end{proof}

\subsection{Mean-Field ODE and Tracking Error}
In this step, we define the mean-field ODE and bound the SA tracking error within a $T$-window of the ODE clock. 

Define
\[
        r_k(\ell):=\frac{c(\ell)}{\tilde c_k(\ell)}
        \quad\text{and}\quad
        \cH_{\mrm N,k}\theta:=r_k(\cG\theta-\theta).
\]
We consider a mean-field ODE driven by the vector field $\cH_{\mrm N,k}$. We note that, whenever $\norm{\hat c_k-c}\le c_\wedge/4$, 
\[
\tilde  c_k(\ell)\ge \max\set{c(\ell) - c_\wedge/4,c_\wedge} \geq \frac{4}{5}\crbk{c(\ell)-c_\wedge/4} + \frac{1}{5}c_\wedge = \frac{4}{5}c(\ell).
\]
Therefore, 
\begin{equation}
        \frac45\le \frac{c(\ell)}{c(\ell)+c_\wedge/4}\leq r_k(\ell) \le\frac54
        \label{eqn:N_ratio_bound}
\end{equation}
for all $\ell\in\mbb L$. 

Next, as in the unnormalized case, for $m\ge n$, define the ODE clock and $T$-window induced by the step sizes $\set{\alpha_k^\mrm N:k\ge 0}$ as 
\begin{equation}
        \tau_{\mrm N}(n,m)
        :=
        \sum_{k=n}^{m-1}\alpha_{k+1}^{\mrm N},
        \quad \text{and}\quad
        \cW_{\mrm N}(n,T)
        :=
        \setcond{m\ge n}{\tau_{\mrm N}(n,m)\le T}.
        \label{eqn:N_clock_and_window}
\end{equation}
Note that under this definition $\set{[\tau_{\mrm N}(n,k),\tau_{\mrm N}(n,k+1)):k\ge n}$ is a partition of $\R_+$. Then, for all $t\geq 0$ we consider the solution to the following piecewise-defined ODE: 
\begin{equation}
        \dot\vartheta_{\mrm N}(t)
        =
        \cH_{\mrm N,k}\vartheta_{\mrm N}(t), \quad t\in[\tau_{\mrm N}(n,k),\tau_{\mrm N}(n,k+1)), \quad k\ge n;
        \quad
        \vartheta_{\mrm N}(0;\theta)=\theta.
        \label{eqn:N_piecewise_ODE}
\end{equation}

\begin{lemma}[Exponential stability]
\label{lemma:N_ODE_stability}
On the event
$\sup_{k\ge n}\norm{\hat c_k-c}\le c_\wedge/4$, the ODE
\eqref{eqn:N_piecewise_ODE} has a unique solution in $\Theta$ for all $t\ge 0$. In addition, the solution satisfies
\[\norm{\vartheta_{\mrm N}(t;\theta)-\theta^*}
        \le
        e^{-4(1-\gamma)t/5}\norm{\theta-\theta^*}.\]
\end{lemma}

\begin{proof}
The proof is the piecewise-autonomous counterpart of the proof of
Lemma~\ref{lemma:U_ODE_exponential_stability}. We present a short
proof to highlight some differences compared to
Lemma~\ref{lemma:U_ODE_exponential_stability}.

Fix the window start $n$, and write
$t_k:=\tau_{\mrm N}(n,k)$ for all $k\ge n$. In particular, note that $t_n=0$. 

On the interval
$[t_k,t_{k+1})$, the vector field of the ODE \eqref{eqn:N_piecewise_ODE} is $\cH_{\mrm N,k}$. As in the proof of Lemma \ref{lemma:U_ODE_exponential_stability}, we consider the clipped vector field  $\overline\cH_{\mrm N,k}\theta:=r_k\crbk{\cG\Pi\theta-\theta}.$ 
Since $r_k$ is bounded, $\overline\cH_{\mrm N,k}$ is globally
Lipschitz. Hence, by Lemma~\ref{lemma:ODE_exist_uniques_sol}, the
solution exists uniquely on each interval. These interval solutions
concatenate uniquely to give a solution
$\bar\vartheta_{\mrm N}(t;\theta)$, with
$\bar\vartheta_{\mrm N}(0;\theta)=\theta$ and vector field $\overline\cH_{\mrm N,k}$.

Suppose inductively that
$\bar\vartheta_{\mrm N}(t_k;\theta)\in\Theta$. For
$t\in[t_k,t_{k+1})$, the solution satisfies
\[
        \bar\vartheta_{\mrm N}(t;\theta)
        =
        \bar\vartheta_{\mrm N}(t_k;\theta)
        +
        \int_{t_k}^t
        \overline\cH_{\mrm N,k}
        \bar\vartheta_{\mrm N}(s;\theta)\,ds.
\]
Applying the integrating-factor argument pointwise in
$\ell\in\mbb L$, with factor
$e^{r_k(\ell)(t-t_k)}$, gives
\[
        \bar\vartheta_{\mrm N}(t;\theta)
        =
        e^{-r_k(t-t_k)}
        \bar\vartheta_{\mrm N}(t_k;\theta)\quad+
        \int_{t_k}^t
        r_ke^{-r_k(t-s)}
        \cG\Pi
        \bar\vartheta_{\mrm N}(s;\theta)ds.
\]
Since $\cG$ maps $\Theta$ into itself, taking sup norm on both sides gives
\[
        \norm{\bar\vartheta_{\mrm N}(t;\theta)}
        \le 
        \frac{1}{1-\gamma}\norm{e^{-r_k(t-t_k)}
        +
        \int_{t_k}^t
        r_ke^{-r_k(t-s)}\,ds} = \frac{1}{1-\gamma}.
\]
Thus, $\bar\vartheta_{\mrm N}(t;\theta)\in\Theta$ throughout the
$k$-th interval. Since
$\bar\vartheta_{\mrm N}(t_n;\theta)=\theta\in\Theta$, induction
over $k$ proves that the solution remains in $\Theta$ for all
$t\ge0$. Therefore, the clipping in
$\overline\cH_{\mrm N,k}$ is ineffective along the trajectory,
and the clipped solution agrees with the solution driven by
$\cH_{\mrm N,k}$ for all $k\ge n$.

To show the exponential stability bound, let
$\epsilon(t)
        :=
        \vartheta_{\mrm N}(t;\theta)-\theta^*,
        $ and $\bar r:=5/4.$
Fix $t\in[t_k,t_{k+1})$. Since
$\cG\theta^*=\theta^*$, we have
\[
\begin{aligned}
        \dot\epsilon(t)
        &=
        r_k\crbk{
        \cG\vartheta_{\mrm N}(t;\theta)
        -
        \cG\theta^*
        -
        \epsilon(t)}\\
        &=
        -\bar r\epsilon(t)
        +
        (\bar r-r_k)\epsilon(t)
        +
        r_k\crbk{
        \cG\vartheta_{\mrm N}(t;\theta)-\cG\theta^*}.
\end{aligned}
\]
For every $\ell\in\mbb L$, the ratio bound
\eqref{eqn:N_ratio_bound} gives
$0\le\bar r-r_k(\ell)$ and $r_k(\ell)\ge4/5$. Since
$\cG$ is a $\gamma$-contraction on $\Theta$,
\[
\begin{aligned}
&\abs{
(\bar r-r_k(\ell))\epsilon(t)(\ell)
+
r_k(\ell)
\crbk{
\cG[\vartheta_{\mrm N}(t;\theta)](\ell)
-
\cG\theta^*(\ell)}}\\
&\quad\le
(\bar r-r_k(\ell))\norm{\epsilon(t)}
+
r_k(\ell)\gamma\norm{\epsilon(t)}\\
&\quad=
\sqbk{\bar r-(1-\gamma)r_k(\ell)}
\norm{\epsilon(t)}\\
&\quad\le
\sqbk{\bar r-\frac45(1-\gamma)}
\norm{\epsilon(t)}.
\end{aligned}
\]

As in the proof of Lemma \ref{lemma:U_ODE_exponential_stability}, applying the integrating factor
$e^{\bar r(t-t_k)}$ on the $k$-th interval gives
\[
e^{\bar r(t-t_k)}\epsilon(t)=
\epsilon(t_k)+
\int_{t_k}^t
e^{\bar r(s-t_k)}
\left[
(\bar r-r_k)\epsilon(s)
+
r_k\crbk{
\cG\vartheta_{\mrm N}(s;\theta)-\cG\theta^*}
\right]ds.
\]
Consequently,
\[
e^{\bar r(t-t_k)}\norm{\epsilon(t)}
\le
\norm{\epsilon(t_k)}+
\sqbk{\bar r-\frac45(1-\gamma)}
\int_{t_k}^t
e^{\bar r(s-t_k)}
\norm{\epsilon(s)}ds,
\]
and Gronwall's inequality gives
\[
        \norm{\epsilon(t)}
        \le
        e^{-4(1-\gamma)(t-t_k)/5}
        \norm{\epsilon(t_k)},
        \qquad t\in[t_k,t_{k+1}).
\]
Applying this estimate successively over each of the time intervals starting from $t_n=0$, we see that
\[
        \norm{\vartheta_{\mrm N}(t;\theta)-\theta^*}
        \le
        e^{-4(1-\gamma)t/5}
        \norm{\theta-\theta^*},
        \qquad t\ge0.
\]
This proves Lemma \ref{lemma:N_ODE_stability}.
\end{proof}

Define the error term for the normalized algorithm by
\begin{equation}
\begin{aligned}
        \xi_{k+1}^{\mrm N}(\ell):=
        \frac{\kappa(\ell,Z_k)}{\tilde c_k(\ell)}
        \crbk{Y_{k+1}^{\mrm N}-\theta_k^{\mrm N}(\ell)}-
        r_k(\ell)
        \crbk{\cG\theta_k^{\mrm N}(\ell)-\theta_k^{\mrm N}(\ell)},
\end{aligned}
        \label{eqn:N_xi_def}
\end{equation}
for all $\ell\in\mbb L$. Then, the SA recursion for $\set{\theta_{k}^\mrm N:k\ge 1}$ can be written as
$$\theta_{k+1}^{\mrm N}=\theta_k^{\mrm N}
+\alpha_{k+1}^{\mrm N}\cH_{\mrm N,k}\theta_k^{\mrm N}
+\alpha_{k+1}^{\mrm N}\xi_{k+1}^{\mrm N}.$$
Thus, following the same analysis as in Lemma \ref{lemma:U_ode_tracking}, we obtain the following lemma.
\begin{lemma}[ODE tracking error]
\label{lemma:N_ode_tracking}
On the event
$\sup_{k\ge n}\norm{\hat c_k-c}\le c_\wedge/4$, for every $T>0$,
\[\max_{m\in\cW_{\mrm N}(n,T)}
\norm{\theta_m^{\mrm N}
-\vartheta_{\mrm N}(\tau_{\mrm N}(n,m);\theta_n^{\mrm N})}\le
        e^{5T/2}
        \sqbk{
        \max_{m\in\cW_{\mrm N}(n,T)}
        \norm{\sum_{k=n}^{m-1}\alpha_{k+1}^{\mrm N}\xi_{k+1}^{\mrm N}}
        +
        \frac{25T\alpha_{n+1}^{\mrm N}}{8(1-\gamma)}}.
\]
\end{lemma}

\begin{proof}
We note that the statement essentially only replaces the factors $e^{2T}$ and $2$ in Lemma \ref{lemma:U_ode_tracking} by $e^{2T\cd (5/4)}$ and $2\cd(5/4)^2$, where $5/4$ comes from \eqref{eqn:N_ratio_bound}. 

Specifically, on the event $\set{\sup_{k\ge n}\norm{\hat c_k-c}\le c_\wedge/4}$, $\cH_{\mrm N,k}$ is $5/2$-Lipschitz on $\Theta$ and
$\norm{\cH_{\mrm N,k}\theta}\le5/[2(1-\gamma)]$.  Repeating the decomposition in the proof of
Lemma~\ref{lemma:U_ode_tracking}, the Euler residual on the $k$-th clock interval is at most
\[
        \frac{25}{8(1-\gamma)}(\alpha_{k+1}^{\mrm N})^2.
\]
Applying the discrete Gronwall's inequality (cf. Lemma \ref{lemma:discrete_gronwall}) and
$\sum_{k=n}^{m-1}(\alpha_{k+1}^{\mrm N})^2
\le T\alpha_{n+1}^{\mrm N}$ proves Lemma \ref{lemma:N_ode_tracking}.
\end{proof}

\subsection{Within-Window Error Bound}
\label{subsec:N_stochastic_error}

We now decompose the error term in \eqref{eqn:N_xi_def} and bound the cumulative error within a window.

Write $q_k^{\mrm N}=\Phi\theta_k^{\mrm N}$ and decompose
\begin{equation}
        \xi_{k+1}^{\mrm N}
        =
        B_{k+1}^{\mrm N}+M_k^{\mrm N},
        \label{eqn:N_noise_decomposition}
\end{equation}
where
\[
\begin{aligned}
        B_{k+1}^{\mrm N}(\ell)&:=
        \frac{\kappa(\ell,Z_k)}{\tilde c_k(\ell)}
        \crbk{Y_{k+1}^{\mrm N}-(\cT q_k^{\mrm N})(Z_k)},\\
        M_k^{\mrm N}(\ell)
        &:=
        \frac{\kappa(\ell,Z_k)}{\tilde c_k(\ell)}
        \crbk{(\cT q_k^{\mrm N})(Z_k)-\theta_k^{\mrm N}(\ell)}-
        \int_\Z
        \frac{\kappa(\ell,z)}{\tilde c_k(\ell)}
        \crbk{(\cT q_k^{\mrm N})(z)-\theta_k^{\mrm N}(\ell)}
        \mu_b(dz).
\end{aligned}
\]
By \eqref{eqn:conditional_target_identity}, $B_{k+1}^{\mrm N}$ is a martingale difference.  Let
$m_k^{\mrm N}(\ell,z)$ denote the centered integrand defining $M_k^{\mrm N}$, and let
$v_k^{\mrm N}(\ell,\cdot)$ be the centered solution of
\begin{equation}
        v(\ell,z)-\int_\Z v(\ell,y)P_b(dy\mid z)
        =m_k^{\mrm N}(\ell,z).
        \label{eqn:N_Poisson_equation}
\end{equation}
The interval argument used in the unnormalized proof gives
$\norm{m_k^{\mrm N}(\ell,\cdot)}\le2/(c_\wedge(1-\gamma))$; hence
\begin{equation}
        \norm{v_k^{\mrm N}(\ell,\cdot)}
        \le
        \frac{4\tmix }{c_\wedge(1-\gamma)}.
        \label{eqn:N_v_uniform_bound}
\end{equation}
Since $m_k^{\mrm N}$ is $\cF_k$-measurable, so is $v_k^{\mrm N}$.  Evaluating
\eqref{eqn:N_Poisson_equation} at $Z_k$ gives
\[
        M_k^{\mrm N}(\ell)=
        v_k^{\mrm N}(\ell,Z_k)-v_k^{\mrm N}(\ell,Z_{k+1})
        +P_{k+1}^{\mrm N}(\ell),
\]
where
$P_{k+1}^{\mrm N}(\ell)=v_k^{\mrm N}(\ell,Z_{k+1})
-E[v_k^{\mrm N}(\ell,Z_{k+1})\mid\cF_k]$ is a martingale difference. Thus, we have the same decomposition as in the unnormalized case: 
\begin{equation}
\sum_{k=n}^{m-1}\alpha_{k+1}^{\mrm N}\xi_{k+1}^{\mrm N} = \sum_{k=n}^{m-1}\alpha_{k+1}^{\mrm N}
        (v_k^{\mrm N}(\cdot,Z_k)-v_k^{\mrm N}(\cdot,Z_{k+1})) + \sum_{k=n}^{m-1}\alpha_{k+1}^{\mrm N}
(B_{k+1}^{\mrm N}+P_{k+1}^{\mrm N})
\label{eqn:N_within_window_decomp}
\end{equation}

We then proceed to bound the two cumulative error terms. First, note that the following bounds hold almost surely:
\begin{equation}\label{eqn:N_primitive_coefficient_bounds}
    0\le
    \frac{\kappa(\ell,z)}{\tilde c_k(\ell)}\le \frac{1}{c_\wedge} \quad \text{and} \quad  
    \left[\frac{\kappa(\cdot,z)}{\tilde c_k(\cdot)}\right]_{\rho}\le L_{\kappa,\rho}\crbk{\frac{1}{c_\wedge}+\frac{1}{c_\wedge^2}}
\end{equation}

Indeed, since $\tilde c_k(\ell)\ge c_\wedge$, the first bound follows from $\kappa\in[0,1]$. Since $\hat c_k$ is an empirical average of $\kappa(\cdot,Z_t)$, Assumption~\ref{assump:stable_kernels} gives
\[
        [\hat c_k]_{\rho}
        \le
        \frac1{k+1}\sum_{t=0}^k[\kappa(\cdot,Z_t)]_{\rho}
        \le L_{\kappa,\rho}.
\]
Moreover, the map $x\ra x\vee c_\wedge$ is one-Lipschitz, so $[\tilde c_k]_{\rho}\le[\hat c_k]_{\rho}\le L_{\kappa,\rho}$. Consequently, for $\ell,\ell'\in\mbb L$, we have
\[
\begin{aligned}
\abs{
\frac{\kappa(\ell,z)}{\tilde c_k(\ell)}
-
\frac{\kappa(\ell',z)}{\tilde c_k(\ell')}} \le
\frac{\abs{\kappa(\ell,z)-\kappa(\ell',z)}}
{\tilde c_k(\ell)}
+
\kappa(\ell',z)
\frac{\abs{\tilde c_k(\ell)-\tilde c_k(\ell')}}
{\tilde c_k(\ell)\tilde c_k(\ell')}\le
L_{\kappa,\rho}\crbk{\frac{1}{c_\wedge}+\frac{1}{c_\wedge^2}}
\rho(\ell,\ell')
\end{aligned}
\]
This proves the second Lipschitz-seminorm
bound in \eqref{eqn:N_primitive_coefficient_bounds}.

The same calculations as in \eqref{eqn:U_BP_lip_bd}--\eqref{eqn:U_BP_norm_bd}, using
\eqref{eqn:N_primitive_coefficient_bounds}, yield
\begin{equation}
\begin{aligned}
        \norm{B_{k+1}^{\mrm N}+P_{k+1}^{\mrm N}}
        &\le
        \frac{10\tmix}{c_\wedge(1-\gamma)},\\
        [B_{k+1}^{\mrm N}+P_{k+1}^{\mrm N}]_{\rho}
        &\le
        \frac{(2+16\tmix)L_{\kappa,\rho}}{1-\gamma}
        \crbk{\frac{1}{c_\wedge}+\frac{1}{c_\wedge^2}}
        +\frac{8\tmix}{c_\wedge}[\theta_k^{\mrm N}]_{\rho}.
\end{aligned}
        \label{eqn:N_BP_lip_bd}
\end{equation}

\begin{lemma}
\label{lemma:N_bias_error_bd}
Suppose $n\ge\ceil{n_0^{\mrm N}}$ and $0<T\le1$. Then, for every
$m\in\cW_{\mrm N}(n,T)$,
\[
\norm{\sum_{k=n}^{m-1}\alpha_{k+1}^{\mrm N}
        \sqbk{v_k^{\mrm N}(\cdot,Z_k)-v_k^{\mrm N}(\cdot,Z_{k+1})}}\le
        \crbk{8+\frac{72}{5}T}
        \frac{\tmix\alpha_{\mrm N}}{c_\wedge^2(1-\gamma)(n+n_0^{\mrm N})}.\]
\end{lemma}

The proof of this lemma is deferred to Section~\ref{section:proof:lemma:N_bias_error_bd}.

We now set $T_{\mrm N}=1/5$.  Fix
$n\ge\ceil{n_0^{\mrm N}}\vee\ceil{16\alpha_{\mrm N}}$ and $\delta\in(0,1)$.  Applying
Lemma~\ref{lemma:max_mart_net_general} to the martingale term in
\eqref{eqn:N_within_window_decomp} on any $\epsilon$-net of $\mbb L$, gives
\[
\max_{m\in\cW_{\mrm N}(n,T_{\mrm N})}
\max_{\ell\in\mfk N_\rho(\epsilon)}
\abs{\sum_{k=n}^{m-1}\alpha_{k+1}^{\mrm N}
\crbk{B_{k+1}^{\mrm N}(\ell)+P_{k+1}^{\mrm N}(\ell)}}\le
        \frac{10\tmix }{c_\wedge(1-\gamma)}
        \sqrt{\frac{2\alpha_{\mrm N}}{5(n+n_0^{\mrm N})}
        \log\!\crbk{\frac{2\cN_\rho(\epsilon)}{\delta}}}
\]
with probability at least $1-\delta$.  

Moreover, \eqref{eqn:N_BP_lip_bd}, Proposition~\ref{prop:lipschitz_iterates}, and
Lemma~\ref{lemma:step size_window} imply
\[
\begin{aligned}
\left[\sum_{k=n}^{m-1}\alpha_{k+1}^{\mrm N}
\crbk{B_{k+1}^{\mrm N}+P_{k+1}^{\mrm N}}\right]_{\rho}&\le
        \frac{8\tmix}{5c_\wedge}L_{0,\rho}
        +\frac{18}{5}\frac{\tmix\alpha_{\mrm N}L_{\kappa,\rho}}{c_\wedge(1-\gamma)}
        \crbk{\frac{1}{c_\wedge}+\frac{1}{c_\wedge^2}}
        \log\!\crbk{\frac{3e(n+n_0^{\mrm N})}{2n_0^{\mrm N}}}\\
&\le
        \frac{\tmix\alpha_{\mrm N}}{c_\wedge^2(1-\gamma)}
        \crbk{2L_{0,\rho}+4L_{\kappa,\rho}\crbk{1+c_\wedge\inv }}
        \log\!\crbk{\frac{5(n+n_0^{\mrm N})}{n_0^{\mrm N}}}.
\end{aligned}
\]
Here, we used $2+16\tmix\le18\tmix$, $\alpha_{\mrm N}/c_\wedge>1$, and
$m+n_0^{\mrm N}\le3(n+n_0^{\mrm N})/2$ on the fixed window.  

Recall from \eqref{eqn:latent_net_general} that for Lipschitz $h:\mbb L\ra\R$
\[\norm{h}
        \le
        \max_{\ell\in\mfk N_\rho(\epsilon)}|h(\ell)|+[h]_{\rho}\epsilon.
        \]
Taking $h = \sum_{k=n}^{m-1}\alpha_{k+1}^{\mrm N}
\crbk{B_{k+1}^{\mrm N}+P_{k+1}^{\mrm N}}$ and $\epsilon=\epsilon_n^{\mrm N}$ in \eqref{eqn:N_mesh_def}, we have
\[
\left[\sum_{k=n}^{m-1}\alpha_{k+1}^{\mrm N}
\crbk{B_{k+1}^{\mrm N}+P_{k+1}^{\mrm N}}\right]_{\rho} \epsilon_n^{\mrm N}\le \frac{\tmix\alpha_{\mrm N}}{c_\wedge^{2}(1-\gamma)(n+n_0^{\mrm N})}.
\]

 Combining this with \eqref{eqn:N_within_window_decomp} and 
Lemma~\ref{lemma:N_bias_error_bd} and using $T_{\mrm N}=1/5$, we have
\begin{equation}
\begin{aligned}
        \max_{m\in\cW_{\mrm N}(n,1/5)}
        \norm{\sum_{k=n}^{m-1}\alpha_{k+1}^{\mrm N}\xi_{k+1}^{\mrm N}}
        &\le
        \frac{10\tmix}{c_\wedge(1-\gamma)}
        \sqrt{\frac{2\alpha_{\mrm N}}{5(n+n_0^{\mrm N})}\log\!\crbk{\frac{2\cN_\rho(\epsilon_n^{\mrm N})}{\delta}}}
        \\
        &\quad +
        \frac{297}{25}
        \frac{\tmix\alpha_{\mrm N}}
        {c_\wedge^2(1-\gamma)(n+n_0^{\mrm N})},
\end{aligned}
        \label{eqn:N_window_bound}
\end{equation} with probability at least $1-\delta$.

\subsection{Error Propagation Through Epoch Iteration}
\label{subsec:N_epoch_iteration}

We analyze the last-iterate error by solving the error recursion with \eqref{eqn:N_window_bound} and Lemma \ref{lemma:deterministic_epoch_iteration} on the good event in Lemma \ref{lemma:N_denominator_bd} that the empirical denominator is uniformly close to the population value.

Specifically, apply Lemma \ref{lemma:N_denominator_bd} with confidence level $\delta/3$ at time
$m_0(\delta)$. Recall the definition of $m_0(\delta)$ in \eqref{eqn:N_burnin_def}. Then, $P(\Omega_c)\geq 1-\delta/3$ where
\begin{equation}
        \Omega_c
        :=
        \set{\sup_{k\ge m_0(\delta)}\norm{\hat c_k-c}
        \le c_\wedge/4}.
        \label{eqn:N_denominator_good_event}
\end{equation}
Set
\[
        m_0:=m_0(\delta),
        \qquad
        m_{r+1}:=\max\cW_{\mrm N}(m_r,1/5),
        \qquad
        \delta_r:=\frac{4\delta}{\pi^2(r+1)^2}.
\]
Since $m_0\ge\ceil{16\alpha_{\mrm N}}$, Lemma~\ref{lemma:step size_window} gives
\begin{equation}
        \frac1{10}
        \le
        \tau_{\mrm N}(m_r,m_{r+1})
        \le
        \frac15.
        \label{eqn:N_epoch_geometry}
\end{equation}
Let $\Omega_{r,\mrm N}$ be the event on which \eqref{eqn:N_window_bound} holds with $n=m_r$ and
confidence level $\delta_r$.  Since $\sum_{r\ge0}\delta_r=2\delta/3$, by union bound,
\[
        P\!\left(\Omega_c\cap\bigcap_{r\ge0}\Omega_{r,\mrm N}\right)
        \ge1-\delta.
\]
Moreover, as in \eqref{eqn:tu_r+1_bd} in the proof of Theorem~\ref{thm:unnormalized_rate}, 
\eqref{eqn:N_epoch_geometry} implies that
\[
        r+1
        \le
        10\alpha_{\mrm N}\log\!\crbk{e(m_r+n_0^{\mrm N})}.
\]

Define
$H_n^{\mrm N}=\cL_n^{\mrm N}(\delta)+\log\cN_\rho(\epsilon_n^{\mrm N})$.
This sequence is nonnegative and nondecreasing.  Hence, for $0\le j\le r$,
\begin{equation}
\log\!\crbk{\frac{2\cN_\rho(\epsilon_{m_j}^{\mrm N})}{\delta_j}}
\le \log\cN_\rho(\epsilon_{m_r}^{\mrm N})+\log\!\crbk{\frac{50\pi^2\alpha_{\mrm N}^2
\log^2(e(m_r+n_0^{\mrm N}))}{\delta}}\le H_{m_r}^{\mrm N}.
        \label{eqn:N_epoch_log_domination}
\end{equation}

Put $\Delta_n^{\mrm N}=\norm{\theta_n^{\mrm N}-\theta^*}$. On the event $\Omega_c\cap\bigcap_{r\ge0}\Omega_{r,\mrm N}$, combining
Lemmas~\ref{lemma:N_ODE_stability} and \ref{lemma:N_ode_tracking} with
\eqref{eqn:N_window_bound} and \eqref{eqn:N_epoch_log_domination}, for every
$m\in\cW_{\mrm N}(m_r,1/5)$,
\begin{equation}
\begin{aligned}
        \Delta_m^{\mrm N}
        &\le
        e^{-4(1-\gamma)\tau_{\mrm N}(m_r,m)/5}
        \Delta_{m_r}^{\mrm N}
        +K_1^{\mrm N}
        \sqrt{\frac{H_{m_r}^{\mrm N}}{m_r+n_0^{\mrm N}}}
        +\frac{K_2^{\mrm N}}{m_r+n_0^{\mrm N}},
\end{aligned}
        \label{eqn:N_epoch_recursion}
\end{equation}
where
\begin{equation}
\begin{aligned}
        K_1^{\mrm N}
        &:={}
        10e^{1/2}\sqrt{\frac25}
        \frac{\tmix\sqrt{\alpha_{\mrm N}}}
        {(1-\gamma)c_\wedge},\\
        K_2^{\mrm N}
        &:={}
        e^{1/2}\crbk{\frac{297}{25}+\frac58}
        \frac{\tmix\alpha_{\mrm N}}
        {(1-\gamma)c_\wedge^2}.
\end{aligned}
        \label{eqn:N_epoch_constants}
\end{equation}
Thus, applying Lemma~\ref{lemma:deterministic_epoch_iteration} with $T=1/5$, $g=4(1-\gamma)/5$, and $\lambda=\lambda_{\mrm N}$, we obtain that
\[
\begin{aligned}
        \Delta_n^{\mrm N}
        \le
        e^{1/16}\Delta_{m_0}^{\mrm N}
        \sqbk{\frac{m_0+n_0^{\mrm N}}{n+n_0^{\mrm N}}}^{\lambda_{\mrm N}}+
        \frac{14\alpha_{\mrm N}K_1^{\mrm N}}
        {\lambda_{\mrm N}-1/2}
        \sqrt{\frac{H_n^{\mrm N}}{n+n_0^{\mrm N}}}
        +
        \frac{17\alpha_{\mrm N}K_2^{\mrm N}}
        {\lambda_{\mrm N}-1}
        \frac1{n+n_0^{\mrm N}}.
\end{aligned}
\]

Finally, since $\lambda_{\mrm N}>1$,
\[
        \frac{\alpha_{\mrm N}}{\lambda_{\mrm N}-1/2}
        <
        \frac{\alpha_{\mrm N}}{\lambda_{\mrm N}/2}
        =
        \frac{5}{2(1-\gamma)}.
\]
Moreover,
\[
        e^{1/16}\Delta_{m_0}^{\mrm N}\le\frac3{1-\gamma},
        \quad
        14\cdot\frac52\cdot10e^{1/2}\sqrt{\frac25}<370,\quad\text{and}\quad
        17e^{1/2}\crbk{\frac{297}{25}+\frac58}<351.
\]
Applying these constants and substituting in \eqref{eqn:N_epoch_constants}
gives \eqref{eqn:empirical_normalized_rate} simultaneously for all
$n\ge m_0(\delta)$ on the high-probability event $\Omega_c\cap\bigcap_{r\ge0}\Omega_{r,\mrm N}$. This completes the proof.

\section{Auxiliary Results}

\subsection{Existence and Uniqueness of Solutions}
\begin{lemma}[ \citet{brezis2011functional}, Theorem~7.3]

\label{lemma:ODE_exist_uniques_sol}

Let $(E,\norm{\cd})$ be a Banach space and let $\cH:E\to E$ be globally Lipschitz; that is, there is
$L<\infty$ such that
$$
        \norm{\cH\theta-\cH\theta'}\le L\norm{\theta-\theta'},
        \qquad \theta,\theta'\in E.
$$
Then, for every $\theta_0\in E$, there exists a unique continuously differentiable function
$\vartheta:[0,\infty)\to E$ satisfying
$$
        \vartheta(t)=\theta_0+
        \int_0^t\cH\vartheta(s)ds,
        \qquad t\ge0 .
$$
\end{lemma}

\subsection{Discrete Gr\"onwall's Inequality}
\begin{lemma}
\label{lemma:discrete_gronwall}
Suppose that, for $n\le m\le j$,
$$
        r_m
        \le
        \sum_{k=n}^{m-1}\beta_{k+1}r_k+s_m,
$$
where $r_k,\beta_k\ge0$ and $s_m$ are real-valued.  Then
$$
        r_j
        \le
        \max_{n\le m\le j}|s_m|
        \exp\!\crbk{\sum_{k=n}^{j-1}\beta_{k+1}}.
$$
\end{lemma}

\begin{proof}[Proof of Lemma \ref{lemma:discrete_gronwall}]
Let $s^*_m:= \max_{n\leq k\le m}|s_k|$. We have that 
    $$r_m\leq \sum_{k=n}^{m-1} \beta_{k+1}r_k + s_m^*.$$
We inductively show that $$r_m\leq s_m^*\prod_{k=n}^{m-1}(1+\beta_{k+1})$$ for all $n\leq m\leq j$. The base case $m = n$ holds trivially. For the induction step, we see that 
$$r_{m+1}\leq \sum_{k=n}^{m} \beta_{k+1}s_k^*\prod_{i=n}^{k-1} (1+\beta_{i+1}) + s_{m+1}^*\leq  s_{m+1}^*\crbk{1+\sum_{k=n}^m \beta_{k+1}\prod_{i=n}^{k-1}(1+\beta_{i+1})} $$
where the first inequality uses the induction hypothesis and the second inequality follows because $\set{\beta_k,s_k^*:k\ge n}$ are non-negative and $\set{s_k^*:k\ge n}$ are non-decreasing.  

Note that 
\begin{align*}
    \sum_{k=n}^m \beta_{k+1}\prod_{i=n}^{k-1}(1+\beta_{i+1}) & =  \sum_{k=n}^m (1+\beta_{k+1})\prod_{i=n}^{k-1}(1+\beta_{i+1}) - \sum_{k=n}^m \prod_{i=n}^{k-1}(1+\beta_{i+1}) \\
    &= \sum_{k=n}^m \prod_{i=n}^{k}(1+\beta_{i+1}) - \sum_{k=n-1}^{m-1} \prod_{i=n}^{k}(1+\beta_{i+1})\\
    &= \prod_{i=n}^{m}(1+\beta_{i+1}) - 1
\end{align*}
where we used the convention that $\prod_{k=n}^{n-1} \dots = 1$. This implies the induction step. 

Finally, using the exponential inequality $e^{x}\ge 1+x$, we conclude that 
$$r_j \leq s_j^* \prod_{k=n}^{j-1}(1+\beta_{k+1} )\leq s_j^*\exp\crbk{\sum_{k=n}^{j-1}\beta_{k+1}}.$$
\end{proof}

\subsection{Poisson's Equation}
\begin{lemma}
\label{lemma:poisson_general}
Suppose Assumption \ref{assump:mix} holds with mixing time in Definition \ref{def:mixing_time}. Let $m:\mbb Z\to\R$ be bounded and centered; i.e, $\norm{m}\le b$ and $\int_\Z md\mu_b =0.$ Then
\begin{equation}
        v(z)
        :=
        \sum_{t=0}^{\infty}\int_\Z m(y)P_b^t(dy\mid z)
        \label{eqn:poisson_series_general}
\end{equation}
is the unique bounded centered solution of
\begin{equation}
        v(z)-\int_\Z v(y)P_b(dy\mid z)=m(z),
        \quad \forall z\in\mbb Z,\quad \text{and}\quad \int_\Z v(z)\mu_b(dz)=0
        \label{eqn:poisson_eqn}
\end{equation}
where the sum converges absolutely and uniformly. Moreover, $\norm{v}\le2\tmix b.$
\end{lemma}

\begin{proof}
The construction of the solution, uniqueness, and Lyapunov-type bounds for Markov-chain Poisson equations are standard; see \citet{glynnmeyn1996} and \citet[Chapter~17]{meyn2012markov}. Here, we specialize to the uniform ergodic case and prove the sup norm estimate.

Recall the total-variation convention used in Definition~\ref{def:mixing_time}. Let
$$
        d(t)
        :=
        \sup_{z\in\mbb Z}
        \normTV{P_b^t(z,\cdot)-\mu_b},\quad\text{and}\quad
        \delta(P_b^t)
        :=
        \frac12\sup_{z,z'\in\Z}
        \normTV{P_b^t(z,\cdot)-P_b^t(z',\cdot)}.
$$

Note that for every $z,z'\in\mbb Z$,
$$
\begin{aligned}
\normTV{P_b^t(z,\cdot)-P_b^t(z',\cdot)}
\le
\normTV{P_b^t(z,\cdot)-\mu_b}
+\normTV{P_b^t(z',\cdot)-\mu_b} \le 2d(t).
\end{aligned}
$$
Thus, $\delta(P_b^t)\le d(t)$. Moreover, since $\mu_b$ is invariant for $P_b$, we have for $s,t\ge 0$,
$$
\begin{aligned}
        d(s+t) &= \sup_{z\in\Z}\normTV{P_b^{s+t}(z,\cdot)-\mu_b}\\
        &=\sup_{z\in\Z}\normTV{(P_b^s(z,\cdot)-\mu_b)P_b^t}\\
        &\leq
        \sup_{z\in\Z}\normTV{P_b^s(z,\cdot)-\mu_b}\delta(P_b^t) \\
        &\le d(s)d(t).
\end{aligned}
$$
Since $d(\tmix)\leq 1/2$, we have that
$d(k\tmix)\le 2^{-k}$ for every integer $k\ge1$. 

Because $m$ is centered,
$$
        \abs{\int_\Z m(y)P_b^t(dy\mid z)}
        =
        \abs{\int_\Z m(y)\crbk{P_b^t(dy\mid z)-\mu_b(dy)}}
        \le
        b\min\{1,d(t)\}.
$$
Moreover, note that since $d(t)$ is non-increasing, 
\begin{equation}
\begin{aligned}
        \sum_{t=0}^{\infty}1\wedge d(t)
        = \sum_{k=0}^\infty\sum_{j=k\tmix}^{(k+1)\tmix-1}1\wedge d(t)&\le
        \tmix
        +\tmix\sum_{k=1}^{\infty}2^{-k}
        =
        2\tmix.
\end{aligned}
        \label{eqn:sum_mixing_time_est}
\end{equation}
This proves that the sum in \eqref{eqn:poisson_series_general} converges absolutely and uniformly, and also yields the sup norm estimate $\norm{v}\leq 2\tmix b$. 
\end{proof}

\subsection{Martingale Maximal Inequality on a Net}
\begin{lemma}
\label{lemma:max_mart_net_general}
Let $\mfk N$ be a deterministic finite subset of $\mbb L$.  For every $\ell\in\mfk N$, let
$D_{k+1}(\ell)$ be $\cF_{k+1}$-measurable and $E[D_{k+1}(\ell)\mid\cF_k]=0.$
Suppose in addition that, for some $b\ge0$,
$\abs{D_{k+1}(\ell)}\le b$ almost surely. 

Let $w_0,\ldots,w_{N-1}$ be deterministic weights.  Then, with probability at least $1-\delta$,
\begin{equation}
        \max_{0\le m\le N}
        \max_{\ell\in\mfk N}
        \abs{\sum_{k=0}^{m-1}w_kD_{k+1}(\ell)}
        \le
        b\sqrt{
        2\log\!\crbk{\frac{2|\mfk N|}{\delta}}
        \sum_{k=0}^{N-1}w_k^2}.
        \label{eqn:max_mart_net_general}
\end{equation}
\end{lemma}

\begin{proof}
Fix $\ell\in\mfk N$ and define
$$
        S_m(\ell):=\sum_{k=0}^{m-1}w_kD_{k+1}(\ell),
        \qquad
        V_m:=b^2\sum_{k=0}^{m-1}w_k^2,
        \qquad 0\le m\le N.
$$
Conditionally on $\cF_k$, the random variable $w_kD_{k+1}(\ell)$ has mean zero and lies in
$[-b\abs{w_k},b\abs{w_k}]$.  Thus, Hoeffding's lemma
\citep[Lemma~1]{hoeffding1963probability} gives, for every $\lambda\in\R$,
\begin{equation}
\begin{aligned}
        E\!\left[
        \exp\!\crbk{\lambda w_kD_{k+1}(\ell)}
        \,\middle|\,\cF_k\right]
        &\le
        \exp\!\crbk{\frac{\lambda^2b^2w_k^2}{2}}.
\end{aligned}
        \label{eqn:conditional_hoeffding_mgf}
\end{equation}
Consequently, for each fixed $\lambda\in\R$,
\begin{equation}
        L_m(\ell,\lambda)
        :=
        \exp\!\left\{
        \lambda S_m(\ell)-\frac{\lambda^2V_m}{2}
        \right\},
        \qquad 0\le m\le N,
\label{eqn:hoeffding_exp_supermg}
\end{equation}
is a nonnegative supermartingale with $L_{0}(\ell,\lambda)=1$.  Indeed,
$$
\begin{aligned}
        E[L_{m+1}(\ell,\lambda)\mid\cF_m]=L_m(\ell,\lambda)
        E\!\left[
        \exp\!\left\{
        \lambda w_m D_{m+1}(\ell)-\frac{\lambda^2b^2w_m^2}{2}
        \right\}
        \middle|\cF_m\right] \le L_m(\ell,\lambda),
\end{aligned}
$$
where the last step uses \eqref{eqn:conditional_hoeffding_mgf}.

Set $V:=V_N$.  For $t>0$ and $\lambda>0$, if $S_m(\ell)\ge t$ for some $m\le N$, then, since
$V_m\le V$,
$$
        L_m(\ell,\lambda)
        \ge
        \exp\!\crbk{\lambda t-\lambda^2V/2}.
$$
Ville's inequality \citep[Lemma~1]{howard2020timeuniform} therefore yields
\begin{equation}
\begin{aligned}
        \mbb P\!\left(
        \max_{0\le m\le N}S_m(\ell)\ge t
        \right)
        &\le
        \exp\!\crbk{-\lambda t+\lambda^2V/2}.
\end{aligned}
        \label{eqn:hoeffding_ville_one_sided}
\end{equation}

If $V=0$, then $S_m(\ell)=0$ a.s. for every $m$.  Otherwise, choosing $\lambda=t/V$ in
\eqref{eqn:hoeffding_ville_one_sided} gives
$$
        \mbb P\!\left(
        \max_{0\le m\le N}S_m(\ell)\ge t
        \right)
        \le
        \exp\!\crbk{-\frac{t^2}{2V}}.
$$
Applying the same argument to $-D_{k+1}(\ell)$ and taking a union bound over the two signs and all
$\ell\in\mfk N$ gives
$$
        \mbb P\!\left(
        \max_{0\le m\le N}\max_{\ell\in\mfk N}
        \abs{S_m(\ell)}\ge t
        \right)
        \le
        2\abs{\mfk N}\exp\!\crbk{-\frac{t^2}{2V}}.
$$
Taking
$$
        t =b\sqrt{
        2\log\!\crbk{\frac{2\abs{\mfk N}}{\delta}}
        \sum_{k=0}^{N-1}w_k^2}
$$
proves \eqref{eqn:max_mart_net_general}.
\end{proof}

\subsection{Step Size Properties}

\begin{lemma}\label{lemma:step size_window}
Fix real numbers $\alpha\ge 1$ and $T>0$, and integers $n\ge0$ and $n_0>0$.  Set $\alpha_k=\alpha/(k+n_0)$,
$\tau(n,m)=\sum_{k=n}^{m-1}\alpha_{k+1}$, and
$\cW(n,T)=\setcond{m\ge n}{\tau(n,m)\le T}$.  Then, for every $m\in\cW(n,T)$,
\begin{equation}
        m+n_0+1
        \le
        e^{T/\alpha}(n+n_0+1).
        \label{eqn:fixed_window_size}
\end{equation}

If $m_\vee=\max\cW(n,T)$ and $\alpha_{n+1}\le T/2$, then the end time satisfies
\begin{equation}
        \frac{T}{2}
        \le
        \tau(n,m_\vee)
        \le
        T.
        \label{eqn:fixed_window_endpoint_length}
\end{equation}

If either $T=1/4$ and $n+n_0\ge8$, or
$T=1/5$ and $n+n_0\ge10$, then
\begin{equation}
        m+n_0
        \le
        \frac32(n+n_0),
        \quad \forall m\in\cW(n,T).
        \label{eqn:fixed_window_three_halves}
\end{equation}
\end{lemma}

\begin{proof}
Since $x\ra 1/x$ is decreasing for $x > 0$, we have
$$
        \sum_{k=n}^{m-1}\frac1{k+n_0+1}
        \ge \int_{n+n_0+1}^{m+n_0+1}\frac{1}{x}dx =
        \log\!\crbk{\frac{m+n_0+1}{n+n_0+1}}.
$$
Thus, for $m\in \cW(n,T)$, $$ T\ge \tau(n,m)\ge \sum_{k=n}^{m-1} \alpha_{k+1} = \sum_{k=n}^{m-1} \frac{\alpha}{k+1+n_0}\ge \alpha\log\!\crbk{\frac{m+n_0+1}{n+n_0+1}}$$ which gives \eqref{eqn:fixed_window_size}. 

For the end time estimate in \eqref{eqn:fixed_window_endpoint_length}, we note that the maximality of $m_\vee$ implies
$\tau(n,m_\vee)+\alpha_{m_\vee+1}>T$, and
$\alpha_{m_\vee+1}\le\alpha_{n+1}\le T/2$.

For the last claim, \eqref{eqn:fixed_window_size} gives
$$
        \frac{m+n_0}{n+n_0}
        \le
        e^{T/\alpha}
        \crbk{1+\frac1{n+n_0}}.
$$
In the first case the right-hand side is at most $e^{1/4}(1+1/8)<3/2$; in the second it is at most $e^{1/5}(1+1/10)<3/2$. \end{proof}

\section{Analyses for the Bias Terms from Poisson's Equation}

\subsection{Unnormalized Algorithm: Proof of Lemma \ref{lemma:U_bias_error_bd}}

\begin{proof}

We start by establishing some auxiliary bounds. First, we recall from \eqref{eqn:U_Poisson_eqn} that $\norm{v_k(\ell,\cd)}\leq \frac{4\tmix}{1-\gamma}.$ 

Next, observe that$$m_{k}^\mrm U(\ell,z) - m_{k-1}^\mrm U(\ell,z) =\kappa(\ell,z)([\mathcal T \Phi \theta_k^\mrm U(z) -\mathcal T \Phi \theta_{k-1}^\mrm U(z)] - [\theta_k^\mrm U(\ell) -\theta_{k-1}^\mrm U(\ell)]) + C(\ell) $$
for some centering term $C(\ell)$ that is independent of $z$. Since $\cT\Phi$ is a contraction, 
\begin{align*}
[\mathcal T \Phi \theta_k^\mrm U(z) -\mathcal T \Phi \theta_{k-1}^\mrm U(z)] - [\theta_k^\mrm U(\ell) -\theta_{k-1}^\mrm U(\ell)] &\in \sqbk{[\theta_{k-1}^\mrm U(\ell)-\theta_k^\mrm U(\ell) ]\pm \gamma \norm {\theta_k^\mrm U -\theta_{k-1}^\mrm U}}\\
&\subset \sqbk{0\pm (1+\gamma )\norm{\theta_k^\mrm U -\theta_{k-1}^\mrm U}}.    
\end{align*}
On the other hand, $m_k^\mrm U - m_{k-1}^\mrm U$ is also $\mu_b$-centered and $0\leq\kappa\leq 1$, we have
$$
\norm {m_{k}^\mrm U(\ell,\cd) - m_{k-1}^\mrm U(\ell,\cd)} \leq \spnorm{m_{k}^\mrm U(\ell,\cd) - m_{k-1}^\mrm U(\ell,\cd)} \leq 2\norm{\theta_k^\mrm U -\theta_{k-1}^\mrm U}
$$
where the span seminorm is $\spnorm{f}=\sup_{z\in\Z} f(z) - \inf_{z\in\Z}f(z)$.

For fixed $\ell\in\mbb L$, the solution to Poisson's equation with $m := m_k^\mrm U(\ell,\cd) - m_{k-1}^\mrm U(\ell,\cd)$ is $v_k^\mrm U(\ell,\cd) - v_{k-1}^\mrm U(\ell,\cd)$. Thus, by Lemma \ref{lemma:poisson_general},
$$\begin{aligned}
\norm{v_{k}^\mrm U(\ell,\cd) - v_{k-1}^\mrm U(\ell,\cd)} &\leq 2\tmix\norm{m_k^\mrm U(\ell,\cd) - m_{k-1}^\mrm U(\ell,\cd)}
\leq 4\tmix \norm{\theta_k^\mrm U-\theta_{k-1}^\mrm U}.
\end{aligned}
$$
On the other hand, by the SA recursion defining $\theta_k^\mrm U$ in \eqref{eqn:unnormalized_update},
$$\
\norm {\theta_{k}^\mrm U - \theta_{k-1}^\mrm U }= \alpha_k^\mrm U
\norm{\kappa(\cd,Z_{k-1})(Y_k^\mrm U-\theta_{k-1}^\mrm U)}\leq \frac{2\alpha_k^\mrm U}{1-\gamma}$$
Therefore, 
\begin{equation}\label{eqn:tu_v_diff_bd}
    \norm{v_{k}^\mrm U(\ell,\cd) - v_{k-1}^\mrm U(\ell,\cd)} \leq \frac{8\tmix \alpha_k^\mrm U}{1-\gamma}
\end{equation}

In addition, since $\set{\alpha_k^\mrm U:k\ge n+1}$ is non-increasing
\begin{equation}
\label{eqn:tu_step size_diff_bd}
    \sum_{k=n+1}^{m-1}|\alpha_{k+1}^\mrm U - \alpha_k^\mrm U| = \sum_{k = n+1}^{m-1} \alpha_k^\mrm U - \sum_{k = n+1}^{m-1}\alpha_{k+1}^\mrm U = \alpha^\mrm U_{n+1} - \alpha_m^\mrm U
\end{equation}

With these estimates, we look at the error term in Lemma \ref{lemma:U_bias_error_bd}. Changing the indices of the sum, we see that for $m\in \cW_\mrm U(n,T)$
$$
\begin{aligned}
&\sum_{k=n}^{m-1}\alpha_{k+1}^{\mrm U}
\crbk{v_k^{\mrm U}(\ell,Z_k)-v_k^{\mrm U}(\ell,Z_{k+1})}        \\
&=
\alpha_{n+1}^{\mrm U}v_n^{\mrm U}(\ell,Z_n)
-
\alpha_m^{\mrm U} v_{m-1}^{\mrm U}(\ell,Z_m)
+
\sum_{k=n+1}^{m-1}
\crbk{
\alpha_{k+1}^{\mrm U}v_k^{\mrm U}(\ell,Z_k)
-
\alpha_k^{\mrm U}v_{k-1}^{\mrm U}(\ell,Z_k)}.\\
&= \alpha_{n+1}^{\mrm U}v_n^{\mrm U}(\ell,Z_n)
-
\alpha_m^{\mrm U} v_{m-1}^{\mrm U}(\ell,Z_m)
+
\sum_{k=n+1}^{m-1}
(\alpha_{k+1}^{\mrm U} - \alpha_k^\mrm U)v_k^{\mrm U}(\ell,Z_k)\\
&\quad +
\sum_{k=n+1}^{m-1}
\alpha_k^{\mrm U}\crbk{v_k^{\mrm U}(\ell,Z_k) - v_{k-1}^{\mrm U}(\ell,Z_k)}\\
&\stackrel{(i)}{\leq}\crbk{\alpha_{n+1}^\mrm U + \alpha_m^\mrm U +  \sum_{k=n+1}^{m-1}|\alpha_{k+1}^\mrm U - \alpha_k^\mrm U|}\frac{4\tmix }{1-\gamma} +\frac{8\tmix }{1-\gamma} \sum_{k=n+1}^{m-1}(\alpha_k^\mrm U)^2\\
&\stackrel{(ii)}{=} 2\alpha_{n+1}^\mrm U\frac{4\tmix }{1-\gamma} + \frac{8\tmix }{1-\gamma}\sum_{k=n+1}^{m-1}(\alpha_k^\mrm U)^2\\
&\stackrel{(iii)}{\leq} 8(1 + T)\frac{\tmix  \alpha_{n+1}^\mrm U }{1-\gamma} \\
&\leq 8(1+T)\frac{\tmix\alpha_\mrm U  }{(1-\gamma)(n+n_0)}
\end{aligned}
$$
where $(i)$ uses the bounds in \eqref{eqn:U_Poisson_eqn} and \eqref{eqn:tu_v_diff_bd}, $(ii)$ applies \eqref{eqn:tu_step size_diff_bd}, and $(iii)$ uses the definition of $\cW_\mrm U(n,T)$ so that $\sum_{k=n+1}^{m-1}\alpha_{k+1}^\mrm U = \tau_\mrm U(n,m) \leq T$ and hence $\sum_{k=n+1}^{m-1} \crbk{\alpha_{k}^\mrm U}^2 \leq \alpha_{n+1}^\mrm U\sum_{k=n}^{m-1} \alpha_{k+1}^\mrm U \leq T\alpha_{n+1}^\mrm U$. 
\end{proof}

\subsection{Normalized Algorithm: Proof of Lemma \ref{lemma:N_bias_error_bd}}
\label{section:proof:lemma:N_bias_error_bd}

\begin{proof}
The claim is immediate when $m=n$, so we prove for the case $m\ge n+1$.

We use the same decomposition as in the previous proof 
\begin{equation}
\begin{aligned}
&\sum_{k=n}^{m-1}\alpha_{k+1}^{\mrm N}
\crbk{v_k^{\mrm N}(\ell,Z_k)-v_k^{\mrm N}(\ell,Z_{k+1})}     \\
&= \alpha_{n+1}^{\mrm N}v_n^{\mrm N}(\ell,Z_n)
-
\alpha_m^{\mrm N} v_{m-1}^{\mrm N}(\ell,Z_m)
+
\sum_{k=n+1}^{m-1}
(\alpha_{k+1}^{\mrm N} - \alpha_k^\mrm N)v_k^{\mrm N}(\ell,Z_k)\\
&\quad +
\sum_{k=n+1}^{m-1}
\alpha_k^{\mrm N}\crbk{v_k^{\mrm N}(\ell,Z_k) - v_{k-1}^{\mrm N}(\ell,Z_k)}\\
\end{aligned}\label{eqn:tu_v-v_error_decomp}
\end{equation}
We focus on the last sum term. Let
\[
        f_k(\ell,z)
        :=
        \frac{\kappa(\ell,z)}{\tilde c_k(\ell)}
        \crbk{(\cT q_k^{\mrm N})(z)-\theta_k^{\mrm N}(\ell)}.
\]
By definition, $m_k^{\mrm N}(\ell,\cdot)$ is the centered version of
$f_k(\ell,\cdot)$; that is,
\[
        m_k^{\mrm N}(\ell,z)
        =
        f_k(\ell,z)
        -
        \int_{\mbb Z}f_k(\ell,y)\mu_b(dy).
\]
Moreover, subtracting the Poisson equations for $v_k^{\mrm N}$ and
$v_{k-1}^{\mrm N}$ shows that
$v_k^{\mrm N}(\ell,\cdot)-v_{k-1}^{\mrm N}(\ell,\cdot)$ is the
centered Poisson solution corresponding to
$m_k^{\mrm N}(\ell,\cdot)-m_{k-1}^{\mrm N}(\ell,\cdot)$.

The difference can be written as
\[
\begin{aligned}
        f_k(\ell,z)-f_{k-1}(\ell,z)
        &=
        \crbk{\frac{1}{\tilde c_k(\ell)}-\frac{1}{\tilde c_{k-1}(\ell)}}
        \kappa(\ell,z)
        \crbk{(\cT q_k^{\mrm N})(z)-\theta_k^{\mrm N}(\ell)}\\
        &\quad+
        \frac{\kappa(\ell,z)}{\tilde c_{k-1}(\ell)}
        \Big\{(\cT q_k^{\mrm N})(z)-(\cT q_{k-1}^{\mrm N})(z)
        -\theta_k^{\mrm N}(\ell)+\theta_{k-1}^{\mrm N}(\ell)\Big\}.
\end{aligned}
\]
The values of the first term lie in an interval of length at most
\[
        \frac{2}{1-\gamma}
        \norm{\frac{1}{\tilde c_k}-\frac{1}{\tilde c_{k-1}}},
\]
while the same interval argument as in the unnormalized proof shows
that the values of the second term lie in an interval of length at most
$2c_\wedge^{-1} \norm{\theta_k^{\mrm N}-\theta_{k-1}^{\mrm N}}.$
Moreover,
\[
        \norm{\theta_k^{\mrm N}-\theta_{k-1}^{\mrm N}}
        \le
        \frac{2\alpha_k^{\mrm N}}{c_\wedge(1-\gamma)}.
\]
Since centering does not change the span, and a centered function has sup norm bounded by its span, Lemma~\ref{lemma:poisson_general} gives
\begin{equation}\label{eqn:tu_vkN-vk-1N}
        \norm{v_k^{\mrm N}(\ell,\cdot)-v_{k-1}^{\mrm N}(\ell,\cdot)}
        \le
        \frac{4\tmix}{1-\gamma}
        \norm{\frac{1}{\tilde c_k}-\frac{1}{\tilde c_{k-1}}}+
        \frac{8\tmix\alpha_k^{\mrm N}}
        {c_\wedge^2(1-\gamma)}.
\end{equation}

To bound the right-hand side, we note that the recursion version of the $\hat c_k$ update in \eqref{eqn:empirical_c_recursion} gives
\[
        \hat c_k-\hat c_{k-1}
        =
        \frac{\kappa(\cdot,Z_k)-\hat c_{k-1}}{k+1}.
\]
Since $x\ra x\vee c_\wedge$ is one-Lipschitz and
$\tilde c_k,\tilde c_{k-1}\ge c_\wedge$, we have
\[
\begin{aligned}
        \norm{\frac{1}{\tilde c_k}-\frac{1}{\tilde c_{k-1}}}\le
        \frac{1}{c_\wedge^2}
        \norm{\tilde c_k-\tilde c_{k-1}}\le
        \frac{1}{c_\wedge^2(k+1)}.
\end{aligned}
\]
Since $n\ge\ceil{n_0^{\mrm N}}$, for $k\ge n+1$ we have
$(k+1)^{-1}\le2(k+n_0^{\mrm N})^{-1}$. So, 
\begin{equation}
        \sum_{k=n+1}^{m-1}\alpha_k^{\mrm N}
        \norm{\frac{1}{\tilde c_k}-\frac{1}{\tilde c_{k-1}}}
        \le
        \frac{2}{c_\wedge^2\alpha_{\mrm N}}
        \sum_{k=n+1}^{m-1}(\alpha_k^{\mrm N})^2\le
        \frac{2T}{c_\wedge^2(n+n_0^{\mrm N})}.
        \label{eqn:N_inverse_variation_sum}
\end{equation}

Multiplying \eqref{eqn:tu_vkN-vk-1N} by
$\alpha_k^{\mrm N}$ and taking the sum,  \eqref{eqn:N_inverse_variation_sum} implies that
\begin{equation}
\sum_{k=n+1}^{m-1}\alpha_k^{\mrm N}
        \norm{v_k^{\mrm N}(\ell,\cdot)-v_{k-1}^{\mrm N}(\ell,\cdot)}
        \le
        \frac{72T}{5}
        \frac{\tmix\alpha_{\mrm N}}
        {c_\wedge^2(1-\gamma)(n+n_0^{\mrm N})},
        \label{eqn:N_v_variation_sum}
\end{equation}
where we used $\alpha_{\mrm N}>5/4$, which follows from
$\lambda_{\mrm N}>1$.

Finally, by
\eqref{eqn:N_v_uniform_bound} and the step size difference bound \eqref{eqn:tu_step size_diff_bd}, the first three terms in \eqref{eqn:tu_v-v_error_decomp} are bounded by
\[
        2\alpha_{n+1}^{\mrm N}
        \frac{4\tmix}{c_\wedge(1-\gamma)}
        \le
        \frac{8\tmix\alpha_{\mrm N}}
        {c_\wedge^2(1-\gamma)(n+n_0^{\mrm N})}.
\]
Adding this to \eqref{eqn:N_v_variation_sum} completes the proof.
\end{proof}

\section{Epoch Error Recursion Analysis}
\begin{lemma}[Error recursion bounds]
\label{lemma:deterministic_epoch_iteration}
For $\alpha>1$ and integers $n_0,m_0>0$ such that
$m_0+n_0\ge16\alpha$, define
$$\alpha_k=\frac{\alpha}{k+n_0},\quad     \tau(j,m)=\sum_{k=j}^{m-1}\alpha_{k+1},\quad\text{and}\quad m_{r+1}:=\max\cW(m_r,T),\; r\ge 0$$ 
where $\cW(\cdot,T)$ is the window defined in
Lemma~\ref{lemma:step size_window}. We will consider $T=1/4$ or $T = 1/5$.

Let $g > 1/\alpha$ be a fixed constant and set $\lambda:=\alpha g$. In addition, when $T=1/4$, assume $g\le1$; when $T=1/5$, assume $g\le4/5$.

Suppose that for every $r\ge0$ and every $m\in\cW(m_r,T)$ the nonnegative sequence
$\set{\Delta_n:n\ge m_0}$ satisfies, 
\begin{equation}
\begin{aligned}
        \Delta_m
        &\le
        e^{-g\tau(m_r,m)}\Delta_{m_r}
        +
        K_1\sqrt{\frac{H_{m_r}}{m_r+n_0}}
        +
        \frac{K_2}{m_r+n_0}
\end{aligned}
        \label{eqn:deterministic_epoch_recursion}
\end{equation}
for some non-negative and non-decreasing sequence $\set{H_n:n\ge m_0}$ and $K_1,K_2\ge0$. 
Then, for every $n\ge m_0$,
\begin{equation}
\begin{aligned}
        \Delta_n
        \le e^{1/16}\Delta_{m_0}
        \crbk{\frac{m_0+n_0}{n+n_0}}^\lambda +
        \frac{14\alpha  K_1}{\lambda-1/2}
        \sqrt{\frac{H_n}{n+n_0}}+
       \frac{17\alpha  K_2}{\lambda-1}
        \frac{1}{n+n_0}.
\end{aligned}
        \label{eqn:deterministic_epoch_conclusion}
\end{equation}
\end{lemma}

\begin{proof}
    First, $\cW(m_r,T)$ is nonempty because it contains $m_r$, and it is finite because
$\tau(m_r,m)\to\infty$ as $m\to\infty$.  Moreover, for both choices of $T = 1/4$ and $1/5$, 
$$
        \alpha_{m_r+1}
        \le
        \frac{\alpha}{m_0+n_0}
        \le\frac1{16}
        \le\frac{T}{2}.
$$
Thus, the end time bound \eqref{eqn:fixed_window_endpoint_length} in Lemma \ref{lemma:step size_window} shows that each epoch has ODE clock length between $T/2$ and $T$.  In particular, $m_{r+1}>m_r$, so the integer sequence $(m_r)$ tends to infinity.

For integers $a\le b$, 
\begin{equation}
    \tau(a,b) = \sum_{k=a}^{b-1} \alpha_{k+1} \ge \alpha\int_{a+n_0}^{b+n_0}\frac{1}{x}dx - \frac\alpha{a+n_0}
=\alpha\log\!\crbk{\frac{b+n_0}{a+n_0}} - \frac\alpha{a+n_0}.\label{eqn:harmonic_sum_bd}
\end{equation}
Consequently, take $a = m_0$ and $b = n\ge m_0$, 
\[e^{-g\tau(m_0,n)}
        \le
        e^{\alpha g/(m_0+n_0)}
        \crbk{\frac{m_0+n_0}{n+n_0}}^{\alpha g}
        \le
        e^{1/16}
        \crbk{\frac{m_0+n_0}{n+n_0}}^\lambda,
\]
where we recall that $\lambda=\alpha g\le\alpha$ and
$m_0+n_0\ge16\alpha$.

Fix $n\ge m_0$, consider the unique $r\ge0$ such that $m_r\le n<m_{r+1}$. Applying \eqref{eqn:deterministic_epoch_recursion} successively at the completed epoch endpoints $m_1,\ldots,m_r$ gives
\[
\begin{aligned}
\Delta_{m_r}
&\le e^{1/16}
        \crbk{\frac{m_0+n_0}{m_r+n_0}}^\lambda\Delta_{m_0}
+\sum_{j=0}^{r-1}e^{-g\tau(m_{j+1},m_r)}
\left\{K_1\sqrt{\frac{H_{m_j}}{m_j+n_0}}+\frac{K_2}{m_j+n_0}\right\}.
\end{aligned}
\]
Here, the error introduced during epoch $j$ is multiplied by all subsequent contraction factors, whose product equals $e^{-g\tau(m_{j+1},m_r)}$ by the additivity of $\tau$. Applying \eqref{eqn:deterministic_epoch_recursion} once more on the final partial epoch from $m_r$ to $n$ yields
\begin{equation}
\begin{aligned}
\Delta_n &\le e^{1/16}
        \crbk{\frac{m_0+n_0}{n+n_0}}^\lambda \Delta_{m_0} +K_1\sqrt{H_n}\crbk{\frac{1}{\sqrt{m_r+n_0} } + \sum_{j=0}^{r-1}\frac{e^{-g\tau(m_{j+1},n)}}{\sqrt{m_j+n_0}}
}\\
&\quad+K_2\crbk{\frac{1}{m_r+n_0} + \sum_{j=0}^{r-1}\frac{e^{-g\tau(m_{j+1},n)}}{m_j+n_0}}.
\end{aligned}\label{eqn:deterministic_epoch_unrolled}
\end{equation}
Here we used the monotonicity of $H_n$, so that $H_{m_j}\le H_n$ for every $0\le j\le r$.  We then control these two sums by comparing them with $j = r$.

We first note that 
\[
\frac{1}{\crbk{m_j+n_0}^p}=
        \frac{1}{\crbk{m_r+n_0}^p}
        \exp\!\left\{p
        \log\!\crbk{\frac{m_r+n_0}{m_j+n_0}}\right\}. 
\]
The log term inside the exponential can be bounded using \eqref{eqn:harmonic_sum_bd}. Specifically, 
for $j<r$, put $u_j:=\tau(m_{j+1},m_r)$. Since $\tau(m_j,m_r)\le T+u_j$, applying \eqref{eqn:harmonic_sum_bd} with $a=m_j$ and $b=m_r$ gives
\[
        \alpha\log\!\crbk{\frac{m_r+n_0}{m_j+n_0}}
        \le
        \tau(m_j,m_r)+\frac{\alpha}{m_j+n_0}
        \le
        T+u_j+\frac{\alpha}{m_j+n_0}.
\]
Since $m_j+n_0\ge m_0+n_0\ge16\alpha$, it follows that
\[
        \log\!\crbk{\frac{m_r+n_0}{m_j+n_0}}
        \le
        \frac{T+u_j}{\alpha}+\frac{1}{16\alpha}.
\]
Therefore,
\[
\begin{aligned}
        \frac{1}{\sqrt{m_j+n_0}}
        &\le
        \exp\!\crbk{\frac{T}{2\alpha}+\frac{u_j}{2\alpha}
        +\frac{1}{32\alpha}}
        \frac{1}{\sqrt{m_r+n_0}},\\
        \frac{1}{m_j+n_0}
        &\le
        \exp\!\crbk{\frac{T}{\alpha}+\frac{u_j}{\alpha}
        +\frac{1}{16\alpha}}
        \frac{1}{m_r+n_0}.
\end{aligned}
\]
Multiplying these two bounds by $e^{-gu_j}$ and using
$g = \lambda/\alpha$ yields
\[\begin{aligned}
        \frac{e^{-gu_j}}{\sqrt{m_j+n_0}}
        &\le
        \exp\!\crbk{\frac{T}{2\alpha}+\frac{1}{32\alpha}}
        e^{-(\lambda-1/2)u_j/\alpha}
        \frac{1}{\sqrt{m_r+n_0}},\\
        \frac{e^{-gu_j}}{m_j+n_0}
        &\le
        \exp\!\crbk{\frac{T}{\alpha}+\frac{1}{16\alpha}}
        e^{-(\lambda-1)u_j/\alpha}
        \frac{1}{m_r+n_0}.
\end{aligned}\]

On the other hand, the end time lower bound from \eqref{eqn:fixed_window_endpoint_length} in Lemma~\ref{lemma:step size_window} implies that 
$$u_j = \sum_{k=j+1}^{r-1} \tau(m_k,m_{k+1}) \geq \frac{(r-j-1)T}{2}. $$
Thus, letting $y_{p}:=(\lambda-p)/{\alpha}$,  we have $\sum_{j=0}^{r-1} e^{-y_p u_j}\leq \frac{1}{1-e^{-y_pT/2}}. $
Therefore, the sums in \eqref{eqn:deterministic_epoch_unrolled} can be bounded as
\begin{equation}
\crbk{1+\sum_{j=0}^{r-1}e^{-g\tau(m_{j+1},n)}}\frac{1}{(m_r+n_0)^{p}}\le
\underbrace{\sqbk{1+
\exp\!\crbk{\frac{pT}{\alpha}+\frac{p}{16\alpha}}
\frac{1}{1-e^{-y_pT/2}}
}}_{C_p}
\frac{1}{(m_r+n_0)^{p}}.
\label{eqn:epoch_convolution}
\end{equation}
for $p \in\set{1/2,1}$. Here we used $e^{-g\tau(m_{j+1},n)}\le e^{-gu_j}$, since $n\ge m_r$.

Since $n\in\cW(m_r,T)$, Lemma~\ref{lemma:step size_window} and $m_r+n_0\ge16\alpha$ imply 
\[\frac{n+n_0}{m_r+n_0} \le e^{T/\alpha}\crbk{1+\frac{1}{16\alpha}}=:C.\]
Thus, combining this bound with \eqref{eqn:epoch_convolution} and \eqref{eqn:deterministic_epoch_unrolled}, we have
\begin{equation}
\begin{aligned}
\Delta_n &\le e^{1/16}
        \crbk{\frac{m_0+n_0}{n+n_0}}^\lambda \Delta_{m_0} +\frac{C_{1/2}K_1\sqrt{H_n}}{\sqrt{m_r+n_0} } +\frac{C_1K_2}{m_r+n_0} .\\
        &\leq e^{1/16}
        \crbk{\frac{m_0+n_0}{n+n_0}}^\lambda \Delta_{m_0} +\frac{C^{1/2}C_{1/2}K_1\sqrt{H_n}}{\sqrt{n+n_0} } +\frac{C \cd C_1K_2}{n+n_0}. 
\end{aligned}\label{eqn:tu_only_const_needed}
\end{equation}

It remains to bound $C^pC_p$ for $p\in\set{1/2,1}$.  Set $x:=1/\alpha$.  Since $\lambda=\alpha g$ with $g\leq 1$, the assumptions imply $0<x<\lambda x\le1$. By the definitions of $C$ and $C_p$,
\[
        C^pC_p
        =
        \sqbk{e^{Tx}\crbk{1+\frac{x}{16}}}^{p}
        \left\{
        1+
        \frac{e^{p(T+1/16)x}}{1-e^{-(\lambda-p)xT/2}}
        \right\}.
\]
Note that $(\lambda-p)x\le\lambda x\le1$, so $1\le1/[(\lambda-p)x]$.  Moreover, $e^z\ge1+ z$, $1/(1-e^{-z})\leq e^z/{z}.$ Therefore, applying this inequality with $z=(\lambda-p)xT/2$ yields
\[
        C^pC_p\leq
        \frac{
        \sqbk{e^{Tx}\crbk{1+\frac{x}{16}}}^{p}}
        {(\lambda-p)x}
        \left\{
        1+\frac{2}{T}
        \exp\!\crbk{
        p\crbk{\frac{T}{2}+\frac{1}{16}}x
        +\frac{\lambda xT}{2}}
        \right\}.
\]

When $T=1/4$, using $x\le1$ and $\lambda x\le1$ gives
\[
\begin{aligned}
        C^{1/2}C_{1/2}
        &\le
        \frac{
        \sqrt{\frac{17}{16}e^{1/4}}
        \crbk{1+8e^{7/32}}}
        {(\lambda-1/2)x}
        <
        \frac{14}{(\lambda-1/2)x},\\
        CC_1
        &\le
        \frac{
        \frac{17}{16}e^{1/4}
        \crbk{1+8e^{5/16}}}
        {(\lambda-1)x}
        <
        \frac{17}{(\lambda-1)x}.
\end{aligned}
\]
When $T=1/5$, using $x\le4/5$ and
$\lambda x\le4/5$ gives
\[
\begin{aligned}
        C^{1/2}C_{1/2}
        &\le
        \frac{
        \sqrt{\frac{21}{20}e^{4/25}}
        \crbk{1+10e^{29/200}}}
        {(\lambda-1/2)x}
        <
        \frac{14}{(\lambda-1/2)x},\\
        CC_1
        &\le
        \frac{
        \frac{21}{20}e^{4/25}
        \crbk{1+10e^{21/100}}}
        {(\lambda-1)x}
        <
        \frac{17}{(\lambda-1)x}.
\end{aligned}
\]
Substituting these bounds into
\eqref{eqn:tu_only_const_needed} proves
\eqref{eqn:deterministic_epoch_conclusion}.
\end{proof}

\end{document}